\documentclass{article}
\usepackage[T1]{fontenc}
\usepackage{styles/iclr2027_conference,times}

\usepackage{amsmath,amsfonts,bm}

\def\eqref#1{equation~\ref{#1}}

\def\1{\bm{1}}

\DeclareMathAlphabet{\mathsfit}{\encodingdefault}{\sfdefault}{m}{sl}
\SetMathAlphabet{\mathsfit}{bold}{\encodingdefault}{\sfdefault}{bx}{n}

\newcommand{\E}{\mathbb{E}}

\newcommand{\KL}{D_{\mathrm{KL}}}

\usepackage{amsthm}
\usepackage{algorithm,algorithmic}
\usepackage{hyperref}
\hypersetup{hidelinks}
\usepackage{url}
\usepackage{xcolor}
\definecolor{darkgreen}{RGB}{0,100,0}
\usepackage{graphicx}

\AddToHook{env/figure/begin}{\setlength{\abovecaptionskip}{2pt}}

\newtheorem{proposition}{Proposition}
\hypersetup{pdftitle={FERPO: Forward Entropy-Regularized Policy Optimization},
  pdfauthor={Sebastian Sanokowski, Alireza Sarmadi, Majid Khadiv}}

\hypersetup{hidelinks}
\let\FERPOOriginalColor\color
\renewcommand{\color}[2][]{\FERPOOriginalColor{black}}
\let\FERPOOriginalTextColor\textcolor
\renewcommand{\textcolor}[3][]{\FERPOOriginalTextColor{black}{#3}}

\newcommand{\FERPORevision}[1]{#1}

\title{FERPO: Forward Entropy-Regularized\\Policy Optimization}
\author{
Sebastian Sanokowski \qquad Alireza Sarmadi \qquad Majid Khadiv\\
\normalfont Applied and Theoretical Aspects of Robot Intelligence (ATARI) Lab\\
\normalfont Munich Institute of Robotics and Machine Intelligence (MIRMI)\\
\normalfont Technical University of Munich
}

\iclrfinalcopy

\begin{document}
\addtocontents{toc}{\protect\setcounter{tocdepth}{-1}}
\maketitle
\lhead{Under review}

\begin{abstract}
Several state-of-the-art methods for online reinforcement learning in continuous control improve policies using action gradients of a learned critic. However, critics are typically trained to predict returns, and accurate value predictions do not necessarily yield accurate action derivatives, potentially leading to unreliable policy updates. We propose Forward Entropy-Regularized Policy Optimization (FERPO), an on-policy maximum entropy reinforcement learning algorithm that performs policy improvement using critic values without differentiating the critic with respect to actions. FERPO derives an optimal target action distribution from a policy-improvement objective regularized by entropy and Kullback--Leibler (KL) divergence. We then fit the actor to this target by minimizing a forward-KL objective, estimated using self-normalized importance sampling (SNIS) with actions drawn from the rollout policy. By limiting the target distribution's deviation from the rollout policy, the KL regularization helps keep these importance weights well behaved. In contrast to reverse-KL objectives, which can favor a subset of the target distribution's modes, the forward-KL objective encourages coverage of multiple high-value modes and thereby promotes exploration. Experiments and ablations on MuJoCo Playground and ManiSkill show competitive performance and sample-efficiency gains. Computational benchmarks also demonstrate faster actor updates than Relative Entropy Pathwise Policy Optimization (REPPO).
Our implementation is publicly available.\footnote{Code: {\urlstyle{same}\url{https://github.com/Atarilab/FERPO}}.}
\end{abstract}

\section{Introduction}
\label{sec:introduction}

Reinforcement learning (RL) learns policies through interaction to
maximize cumulative reward. Its successes range from game playing
\citep{Silver_2017} to autonomous drone racing \citep{Kaufmann_2023}
and robot locomotion \citep{ha2025learning}. In robotics, massively
parallel simulation has accelerated on-policy learning by collecting
trajectories from thousands of environments simultaneously. This has
enabled rapid policy training and successful transfer from simulation
to physical robots \citep{pmlr-v164-rudin22a,zakka2025mujocoplayground}.
These advances make efficient and reliable policy improvement central
to scaling online robot learning.

Relative Entropy Pathwise Policy Optimization (REPPO;
\citealp{voelcker2026relative}) is a recent state-of-the-art on-policy
method for this setting. It learns a soft action-value critic from
on-policy trajectories and improves the actor by differentiating the
critic with respect to actions, while limiting policy changes relative
to the rollout policy. This makes pathwise actor updates compatible
with massively parallel on-policy learning. However, standard critic
training supervises values, whereas these actor updates require accurate
action derivatives. Accurate value predictions do not necessarily imply accurate
derivatives: critic-gradient error can bias policy improvement
\citep{NEURIPS2020_03255088}, and learned action gradients can show poor
directional agreement with rollout-based estimates
\citep{pmlr-v139-sarafian21a}. In on-policy learning, biased gradients
can also interact with changing data distributions and destabilize
training \citep{pmlr-v211-gurumurthy23a}. These limitations motivate
policy improvement that uses critic values without requiring their
action derivatives.

We derive the optimal target distribution for a
standard RL policy-improvement objective with entropy and
Kullback--Leibler (KL) regularization. Making this target explicit lets
us replace reverse-KL projection by forward-KL fitting, which encourages
coverage of target modes \citep{minka2005divergence}.
The KL trust region keeps the target close to the rollout policy,
helping control importance-weight concentration, as in recent
diffusion and normalizing flow based sampling methods
\citep{ICLR2026_8f38c9d6,NEURIPS2025_f1ccd13e}.
We introduce Forward Entropy-Regularized Policy Optimization
(FERPO), which applies this construction to on-policy learning.
Its actor update uses critic values without differentiating the critic
with respect to actions. Our contributions are:

\begin{itemize}
  \item An on-policy actor update that uses critic values to construct
  a regularized target action distribution and fits the actor to this
  target using forward KL, without differentiating the critic with
  respect to actions (Section~\ref{sec:target_and_projection}).
  \item A practical estimator using self-normalized
  importance sampling (SNIS), with Lagrange multipliers that regulate
  policy entropy and importance-weight concentration through normalized
  effective sample size (ESS)
  (Sections~\ref{sec:target_and_projection} and~\ref{sec:log_ess}).
  \item {An evaluation on ManiSkill and MuJoCo Playground
  demonstrating competitive performance and sample-efficiency gains
  (Section~\ref{sec:learning_benchmarks}), supported by ablations
  (Section~\ref{sec:ablations}) and computational benchmarks showing
  faster actor updates than REPPO
  (Section~\ref{sec:computational_cost}).}
  \item {Cached actor updates that reduce computation,
  and a state-value extension based on one-step simulation, with
  promising results on AcrobotSwingupSparse
  (Section~\ref{sec:critic_ablations}).}
\end{itemize}

\section{Problem Description}
\label{sec:problem_description}

\textbf{Maximum-entropy reinforcement learning.}
We consider a discounted Markov decision process with
continuous state space $\mathcal S$, bounded continuous action space
$\mathcal A$, initial-state density $\rho_0$, transition density
$P(s'\mid s,a)$, reward $r(s,a)$, and discount factor
$\gamma\in[0,1)$.
For a policy $\pi$, let $\rho_t^\pi$ denote the state
distribution at time $t$, generated by $s_0\sim\rho_0$,
$a_t\sim\pi(\cdot\mid s_t)$, and
$s_{t+1}\sim P(\cdot\mid s_t,a_t)$. Define its normalized discounted
state-visitation distribution as
$d^\pi(s):=(1-\gamma)\sum_{t=0}^{\infty}\gamma^t\rho_t^\pi(s)$,
with $\int d^\pi(s)\,\mathrm ds=1$.
Using this distribution, the maximum-entropy
objective can be written as
\begingroup
\begin{equation}
  \begin{aligned}
    \max_\theta\quad J_\tau(\pi_\theta)
    &:=\frac{1}{1-\gamma}
      \E_{\substack{s\sim d^{\pi_\theta}\\a\sim\pi_\theta(\cdot\mid s)}}
      \!\left[r(s,a)-\tau\log\pi_\theta(a\mid s)\right]\\
    &=\E_{\rho_0,P,\pi_\theta}\!\left[
      \sum_{t=0}^{\infty}\gamma^t
      \bigl(r(s_t,a_t)-\tau\log\pi_\theta(a_t\mid s_t)\bigr)
      \right],
  \end{aligned}
  \label{eq:maxent_objective}
\end{equation}
\endgroup
where the first expectation draws a state
$s\sim d^{\pi_\theta}$ and then an action
$a\sim\pi_\theta(\cdot\mid s)$. The second expectation is over
trajectories generated by $s_0\sim\rho_0$,
$a_t\sim\pi_\theta(\cdot\mid s_t)$, and
$s_{t+1}\sim P(\cdot\mid s_t,a_t)$.
The factor $1/(1-\gamma)$ accounts for the normalization of
$d^{\pi_\theta}$, making the two forms equal.
{Here $\pi_\theta$ is a neural-network policy
(see Appendix~\ref{app:ferpo_algorithm} for its parameterization),
and $\tau>0$ controls its entropy.} Since
$\mathcal H(\pi(\cdot\mid s))=-\E_{a\sim\pi}[\log\pi(a\mid s)]$,
this objective balances expected reward and policy entropy
\citep{ziebart2010modeling,haarnoja2017reinforcement,haarnoja2018soft}.
Its soft value functions satisfy
\begin{equation*}
  Q_\tau^\pi(s,a)
  =r(s,a)+\gamma\E_{s'\sim P(\cdot\mid s,a)}[V_\tau^\pi(s')],
  \quad
  V_\tau^\pi(s)
  =\E_{a\sim\pi(\cdot\mid s)}
    [Q_\tau^\pi(s,a)-\tau\log\pi(a\mid s)].
\end{equation*}
Thus, $Q_\tau^\pi$ includes future entropy rewards, while the current
action's entropy contribution appears explicitly in $V_\tau^\pi$.

\textbf{KL surrogate.}
Let $\pi_{\mathrm{roll}}:=\pi_{\theta_{\mathrm{roll}}}$ be the frozen
policy that generated the current rollout, with soft critic
$Q_{\mathrm{roll}}:=Q_\tau^{\pi_{\mathrm{roll}}}$.
Applying the policy-gradient theorem \citep{sutton1999policy} to
Eq.~\ref{eq:maxent_objective} yields the local reverse-KL surrogate
\begin{equation}
  \mathcal L_{\mathrm{ME}}(\theta;\theta_{\mathrm{roll}})
  :=\frac{\tau}{1-\gamma}\E_{s\sim d_{\mathrm{roll}}}
    \!\left[\KL\!\left(
      \pi_\theta(\cdot\mid s)\middle\|p_{Q,\tau}(\cdot\mid s)
    \right)\right],
  \label{eq:maxent_surrogate}
\end{equation}
where $d_{\mathrm{roll}}:=d^{\pi_{\theta_{\mathrm{roll}}}}$ is the rollout
policy's normalized discounted state-visitation distribution and
$p_{Q,\tau}(a\mid s)\propto\exp(Q_{\mathrm{roll}}(s,a)/\tau)$ is its
Boltzmann policy-improvement target. The policy $\pi_\theta$ acts as a
variational approximation to this target. With $Q_{\mathrm{roll}}$ and
$d_{\mathrm{roll}}$ held fixed, the surrogate satisfies
$\nabla_\theta\mathcal L_{\mathrm{ME}}(\theta;\theta_{\mathrm{roll}})
=-\nabla_\theta J_\tau(\pi_\theta)$
at $\theta=\theta_{\mathrm{roll}}$.
Thus, the surrogate's descent direction matches the return's ascent
direction at the rollout parameters; see
Appendix~\ref{app:kl_surrogate} for the derivation.

\begingroup

\textbf{Trust regions and total variation.}
Repeated updates on one rollout can change state visitation and reduce
surrogate accuracy. Write $\pi_{\mathrm{curr}}=\pi_\theta$ and define
$D_{\mathrm{TV}}(\pi_{\mathrm{curr}},\pi_{\mathrm{roll}})
:=\frac12\int_{\mathcal A}
|\pi_{\mathrm{curr}}(a)-\pi_{\mathrm{roll}}(a)|\,\mathrm da$,
suppressing state conditioning. Its worst-state value is
$D_{\mathrm{TV}}^{\max}:=\sup_s D_{\mathrm{TV}}(
\pi_{\mathrm{curr}}(\cdot\mid s),\pi_{\mathrm{roll}}(\cdot\mid s))$.
Let
$\Delta J_\tau:=J_\tau(\pi_{\mathrm{curr}})-J_\tau(\pi_{\mathrm{roll}})$
be the actual return change and
$\Delta J_{\mathrm{surrogate}}
:=\mathcal L_{\mathrm{ME}}(\theta_{\mathrm{roll}};\theta_{\mathrm{roll}})
 -\mathcal L_{\mathrm{ME}}(\theta;\theta_{\mathrm{roll}})$
the predicted gain. For an exact critic and soft advantage bounded
in magnitude by $\epsilon_A$, the performance-difference argument
\citep{schulman2015trust,cayci2024convergence} gives
\begin{equation}
  \Delta J_\tau
  \geq\Delta J_{\mathrm{surrogate}}
    -\frac{4\gamma\epsilon_A}{(1-\gamma)^2}
       (D_{\mathrm{TV}}^{\max})^2
    -\frac{\tau}{1-\gamma}D_{\mathrm{KL}}^{\max},
  \label{eq:trust_region_bound}
\end{equation}
where $D_{\mathrm{KL}}^{\max}
:=\sup_s\KL(\pi_{\mathrm{curr}}(\cdot\mid s)
            \|\pi_{\mathrm{roll}}(\cdot\mid s))$.
The TV term bounds the advantage error; the KL term is the entropy
correction (Appendix~\ref{app:trust_region_bound}). Pinsker's inequality,
$2(D_{\mathrm{TV}}^{\max})^2\leq D_{\mathrm{KL}}^{\max}$
\citep[Theorem~31]{vanerven2014renyi}, upper-bounds the total penalty by
$\bigl(\frac{2\gamma\epsilon_A}{(1-\gamma)^2}
       +\frac{\tau}{1-\gamma}\bigr)D_{\mathrm{KL}}^{\max}$.
This motivates minimizing $\mathcal L_{\mathrm{ME}}$ subject to
$D_{\mathrm{KL}}^{\max}\leq\delta_{\mathrm{KL}}$.
{For tanh-squashed Gaussians, as frequently used in
MaxEnt-RL, uniform forward-KL constraints or TV bounds strictly below one also control the
reverse-KL entropy correction (Appendix~\ref{app:gaussian_trust_regions}).}
\endgroup

\textbf{Practical objective.}
Replacing the supremum by an expectation over rollout states gives the
average-KL constraint
$\bar D_\theta:=\E_{s\sim d_{\mathrm{roll}}}
[\KL(\pi_\theta(\cdot\mid s)\|\pi_{\mathrm{roll}}(\cdot\mid s))]
\leq\delta_{\mathrm{KL}}$.
We also require a minimum average entropy
$\bar{\mathcal H}_\theta:=\E_{s\sim d_{\mathrm{roll}}}
[\mathcal H(\pi_\theta(\cdot\mid s))]\geq\mathcal H_{\mathrm{target}}$,
interpreting $\tau$ as its Lagrange multiplier, as in automatic entropy
tuning \citep{haarnoja2018applications}.
With $Q_{\mathrm{roll}}$ and $d_{\mathrm{roll}}$ held fixed, maximizing
expected critic value under these constraints gives the Lagrangian
\begin{equation}
  \mathcal L(\pi_\theta,\tau,\lambda)
  :=-\E_{\substack{s\sim d_{\mathrm{roll}}\\a\sim\pi_\theta(\cdot\mid s)}}
    [Q_{\mathrm{roll}}(s,a)]
    +\tau(\mathcal H_{\mathrm{target}}-\bar{\mathcal H}_\theta)
    +\lambda(\bar D_\theta-\delta_{\mathrm{KL}}),
  \label{eq:constrained_lagrangian}
\end{equation}
where $\tau,\lambda\geq0$. The policy minimizes $\mathcal L$, while
the multipliers maximize it: $\tau$ increases when entropy is below
target, and $\lambda$ increases when the KL bound is exceeded.
The rollout-average constraint and estimated critic approximate the
trust-region problem but do not guarantee its worst-state bound.
This motivates a shared design principle for practical
on-policy algorithms: control deviation from the rollout policy when
making multiple updates on the same data, so that the local surrogate
remains informative.

\textbf{Connection to existing methods.}
TRPO controls policy changes through an explicit KL
constraint \citep{schulman2015trust}, while Proximal Policy Optimization
(PPO) clips policy likelihood ratios in the actor surrogate to
discourage large changes \citep{schulman_proximal_2017}.
Relative entropy pathwise policy optimization (REPPO)
learns a soft critic $Q_\phi$ from on-policy rollouts. At each state,
let $\widehat k_\theta(s)$ estimate the forward KL
$\KL(\pi_{\mathrm{roll}}(\cdot\mid s)\|\pi_\theta(\cdot\mid s))$,
using actions drawn from $\pi_{\mathrm{roll}}$.
For $\widehat k_\theta(s)<\delta_{\mathrm{KL}}$, REPPO minimizes the unconstrained
loss $\ell_{\mathrm{ME}}(\theta;s)=\E_{a\sim\pi_\theta(\cdot\mid s)}
[\tau\log\pi_\theta(a\mid s)-Q_\phi(s,a)]$ using pathwise gradients.
Otherwise, it replaces this loss by $\lambda\widehat k_\theta(s)$,
with adaptive $\lambda>0$, pulling the policy back toward
$\pi_{\mathrm{roll}}$ \citep{voelcker2026relative}.

\section{Method: FERPO}
\label{sec:method}

Pathwise actor updates require accurate critic action derivatives,
although standard critic learning supervises values. Value accuracy
alone does not ensure derivative accuracy
\citep{NEURIPS2020_03255088,pmlr-v139-sarafian21a}.
This motivates an actor update that uses critic values without requiring
$\nabla_a Q_\phi(s,a)$.
We introduce \emph{Forward Entropy-Regularized Policy Optimization}
(\textbf{FERPO}). FERPO first minimizes the Lagrangian in
Eq.~\ref{eq:constrained_lagrangian} over unrestricted action
distributions $q(\cdot\mid s)$ with fixed multipliers, then fits
the parametric policy to the resulting target using the forward KL
divergence combined with self-normalized importance
sampling (SNIS). The
temperature and trust-region multiplier are adapted alongside the actor.
The actor differentiates the policy log density with the target weights
held fixed; it still depends on critic-value accuracy and adequate
sampling of the target.

\subsection{Target construction and actor updates}
\label{sec:target_and_projection}

We define the target as
$q^*=\arg\min_q\mathcal L(q,\tau,\lambda)$, with $\mathcal L$ from
Eq.~\ref{eq:constrained_lagrangian} and the minimization over all
conditional action densities $q(\cdot\mid s)$, unrestricted by the
parametric policy family. With the critic, rollout policy, $\tau$, and
$\lambda$ fixed, this minimization separates across states.
Introducing a Lagrange multiplier for the normalization constraint
$\int_{\mathcal A}q(a\mid s)\,\mathrm da=1$ and setting the functional
derivative with respect to $q$ to zero yields the unique optimal density
\begin{equation}
  q^*(a\mid s)
  =\frac{1}{Z_*(s)}
    \exp\!\left(\frac{Q_{\mathrm{roll}}(s,a)}{\tau+\lambda}\right)
    \pi_{\mathrm{roll}}(a\mid s)^{\lambda/(\tau+\lambda)}.
  \label{eq:optimal_target}
\end{equation}
Here $Z_*(s)$ normalizes the density. Thus, $q^*$ is the unrestricted
minimizer of $\mathcal L(\cdot,\tau,\lambda)$, which we subsequently
approximate within the policy family. See Appendix~\ref{app:optimal_target}
for a derivation. The target combines high-value
actions with proximity to the rollout policy: $\lambda=0$ recovers
$p_{Q,\tau}$, while increasing $\lambda$ makes the update more
conservative.
Under the assumptions of Appendix~\ref{app:exact_fit_proof}, repeated
exact unrestricted updates with a fixed critic and temperature and a
refreshed reference converge to $p_{Q,\tau}$. The fixed KL multiplier
$\lambda$ controls the convergence rate without changing the limit.

For fixed multipliers, Eq.~\ref{eq:constrained_lagrangian} equals
$(\tau+\lambda)\E_{s\sim d_{\mathrm{roll}}}[\KL(\pi_\theta\|q^*)]$
plus a constant independent of the policy.
FERPO replaces this reverse-KL projection by a forward-KL fit, keeping the
critic and rollout parameters fixed during each actor update:
\begin{equation*}
  \mathcal L_{\mathrm{F}}(\theta)
  :=\E_{s\sim d_{\mathrm{roll}}}\!\left[
    \KL\!\left(q^*(\cdot\mid s)\middle\|\pi_\theta(\cdot\mid s)\right)
  \right]
  =-\E_{\substack{s\sim d_{\mathrm{roll}}\\a\sim q^*(\cdot\mid s)}}
    [\log\pi_\theta(a\mid s)]+\mathrm{const}.
\end{equation*}
This is a weighted maximum-likelihood fit to the optimal regularized
target. All state expectations below use the frozen distribution
$d_{\mathrm{roll}}$ and are estimated using rollout batches.
To estimate the action expectation, draw $K$ candidate actions
$a_i\sim\pi_{\mathrm{roll}}(\cdot\mid s)$ and minimize the SNIS loss
\begin{equation}
  \boxed{
    \widehat{\mathcal L}_{\mathrm F}(\theta)
    :=-\E_{s\sim d_{\mathrm{roll}}}\!\left[
      \sum_{i=1}^{K}w_i(s)\log\pi_\theta(a_i\mid s)\right],
    \quad\text{where}\quad
    w_i(s):=\frac{\exp(\ell_i(s))}{\sum_{j=1}^{K}\exp(\ell_j(s))}
  },
  \label{eq:sampled_forward_kl}
\end{equation}
with $\ell_i(s):=[Q_\phi(s,a_i)-\tau\log\pi_{\mathrm{roll}}(a_i\mid s)]/(\tau+\lambda)$.
The weights are held fixed during the actor update.

\textbf{Multiplier losses and complete actor update.}
As suggested by Eq.~\ref{eq:constrained_lagrangian}, we adapt the
temperature $\tau$ and trust-region coefficient $\lambda$ as Lagrange
multipliers alongside the policy.
We parameterize $\tau=\exp(\eta)$ and $\lambda=\exp(\xi)$ to keep both
multipliers positive.
The temperature update follows automatic entropy tuning in SAC
\citep{haarnoja2018applications}. Using fresh actions
$\widetilde a_s\sim\pi_\theta(\cdot\mid s)$, we estimate the policy entropy as
$\widehat{\mathcal H}_\theta:=-\E_{s\sim d_{\mathrm{roll}}}
[\log\pi_\theta(\widetilde a_s\mid s)]$ and minimize
\begin{equation}
    \mathcal L_\tau(\eta)
    :=\tau\,\operatorname{stopgrad}
       [\widehat{\mathcal H}_\theta-\mathcal H_{\mathrm{target}}],
    \label{eq:temperature_controller}
\end{equation}
where $\operatorname{stopgrad}$ holds its argument fixed during
differentiation. Gradient descent increases $\tau$ when entropy is below
$\mathcal H_{\mathrm{target}}$ and decreases it when entropy exceeds the target.

For $\lambda$, however, we replace direct KL control with an update
based on the log normalized effective sample size (ESS) of the importance
weights. This regulates weight concentration;
Section~\ref{sec:log_ess} gives the detailed motivation and its connection
to trust-region control. Using the weights from
Eq.~\ref{eq:sampled_forward_kl}, define
$\widehat e(s):=(K\sum_{i=1}^{K}w_i(s)^2)^{-1}$, which ranges from
$1/K$ when one action carries all weight to $1$ for uniform weights.
For $K\geq2$ and a target $e_{\mathrm{target}}\in(1/K,1)$, we minimize
\begin{equation}
    \mathcal L_\lambda(\xi)
    :=\lambda\,\operatorname{stopgrad}\!\left[
      \E_{s\sim d_{\mathrm{roll}}}[\log\widehat e(s)]
        -\log e_{\mathrm{target}}\right].
    \label{eq:log_ess_controller}
\end{equation}
This increases $\lambda$ when average log ESS falls below
$\log e_{\mathrm{target}}$, flattening the weights, and decreases it when
average log ESS exceeds the target.
Together with the forward fit, these multiplier losses give the complete
actor objective
\begin{equation*}
  \mathcal L_{\mathrm{policy}}(\theta,\eta,\xi)
  :=\widehat{\mathcal L}_{\mathrm F}(\theta)
    +\mathcal L_\tau(\eta)+\mathcal L_\lambda(\xi).
\end{equation*}
The weights are detached from both the policy and multipliers, and the
entropy and log-ESS gaps are detached in their respective losses.
The critic and rollout policy stay fixed within each actor update; updated
$\tau$ and $\lambda$ are used to construct the next target weights.
Algorithm~\ref{alg:ferpo} in
Appendix~\ref{app:ferpo_algorithm} summarizes the full training loop.

\subsection{Log ESS and trust-region control}
\label{sec:log_ess}

\textbf{Controlling importance-weight concentration.}
Using the optimal regularized distribution $q^*$ as the forward-KL
target tempers the importance weights needed to fit the Boltzmann
target $p_{Q,\tau}$ directly. For $\tau>0$,
$q^*(a\mid s)/\pi_{\mathrm{roll}}(a\mid s)
\propto [p_{Q,\tau}(a\mid s)/\pi_{\mathrm{roll}}(a\mid s)]^{\tau/(\tau+\lambda)}$.
For $\lambda>0$, the exponent lies below one, reducing log-weight
differences and flattening the normalized weights on the same candidate
actions. With zero entropy regularization ($\tau=0$, $\lambda>0$), the
explicit rollout-density factor cancels entirely, leaving
$w_i(s)\propto\exp(Q_\phi(s,a_i)/\lambda)$.
For positive $\tau$, the residual factor is
$\pi_{\mathrm{roll}}(a_i\mid s)^{-\tau/(\tau+\lambda)}$.
For a fixed critic and rollout policy, increasing $\tau$ at fixed
$\lambda$ strengthens this inverse-density dependence and can increase
weight concentration, although it also softens the value weighting,
so the overall effect need not be monotone.
The update of the Lagrange multiplier $\lambda$ in
Eq.~\ref{eq:log_ess_controller} uses this flattening effect to limit
weight concentration.

At the population level, a lower bound on normalized ESS also gives an
upper bound on the target's KL divergence from the rollout policy
\FERPORevision{\citep{metelli2018policy,watson2023inferring,drolet2026daps}}.
{The following proposition also transfers this control
to the fitted actor in TV when its population fitting loss does not
increase. This motivates using ESS to adapt $\lambda$ and connects
it to the policy-proximity motivation in
Section~\ref{sec:problem_description}.}

\begin{proposition}[{ESS control of target and actor}]
\label{prop:log_ess_chi_squared}
\label{prop:ess_actor_tv}
{Fix a state and a target $q^*\ll\pi_{\mathrm{roll}}$
with population normalized ESS
$e_\infty:=1/\E_{\pi_{\mathrm{roll}}}
[(q^*/\pi_{\mathrm{roll}})^2]>0$.
If fitting to this fixed target satisfies
$\KL(q^*\|\pi_{\mathrm{curr}})
\leq\KL(q^*\|\pi_{\mathrm{roll}})$,
the triangle inequality and Pinsker's inequality give
\begin{equation}
  \frac12D_{\mathrm{TV}}(\pi_{\mathrm{curr}},\pi_{\mathrm{roll}})^2
  \leq \KL(q^*\|\pi_{\mathrm{roll}})
  \leq D_2(q^*\|\pi_{\mathrm{roll}})=-\log e_\infty,
  \label{eq:ess_actor_tv_bound}
\end{equation}
where $D_2$ is the order-two R\'enyi divergence
\citep{vanerven2014renyi}.}
\end{proposition}

{The proof is given in Appendix~\ref{app:log_ess_proof}.}
\textbf{What the ESS target controls.}
FERPO uses a shared $\lambda$ to regulate log ESS averaged over rollout
states. {Under the population fitting condition, higher
population ESS tightens the bound on the total variation between
$\pi_{\mathrm{curr}}$ and $\pi_{\mathrm{roll}}$.
{For Gaussian policies, a uniform TV bound below one also
bounds $\KL(\pi_{\mathrm{curr}}\|\pi_{\mathrm{roll}})$.}
In practice, ESS is estimated from sampled actions and averaged across
states, so the controller does not guarantee this KL bound at every state.}

\subsection{Critic learning}
\label{sec:critic_learning}

To improve learning of the soft action-value function $Q_\phi$, we follow
the critic-training approach of REPPO \citep{voelcker2026relative},
combining multi-step temporal-difference (TD) targets
with trace parameter $\lambda_{\mathrm{TD}}$, the histogram
loss with Gaussian targets (HL-Gauss) \citep{farebrother2024stop},
and auxiliary prediction losses.
The targets are computed from on-policy rollouts and include the entropy
contributions to the soft return. Here $\lambda_{\mathrm{TD}}$ is the
trace-decay parameter, separate from the actor's KL coefficient $\lambda$.
HL-Gauss encodes each scalar soft-return target as a
Gaussian-smoothed histogram over value bins and trains the critic's
categorical output using cross-entropy. The scalar $Q_\phi(s,a)$ is
recovered as the mean of the predicted histogram. Following REPPO,
auxiliary self-prediction trains the critic representation to predict
the next transition's latent representation, providing an additional
learning signal alongside the value targets
\citep{schwarzer2021dataefficient,voelcker2024when,fujimoto2025towards}.

\section{Related Work}
\label{sec:related_work}

\textbf{Constrained measure transport (CMT).}
Related methods use similar constrained target construction
and distribution fitting for sampling with normalizing flows
\citep{ICLR2026_8f38c9d6} and diffusion models
\citep{NEURIPS2025_f1ccd13e}. Both adapt the trust-region multiplier
$\lambda$ using KL constraints. The normalizing-flow method additionally
regulates temperature through a bound on entropy decrease.

\textbf{Maximum a posteriori policy optimization.}
MPO is an off-policy method that alternates KL-constrained target
construction with forward-KL policy fitting, using replay data and
Retrace for critic learning \citep{abdolmaleki2018maximum}.
Using our notation, its variational objective uses the augmented reward
$r(s,a)-\lambda\log(q(a\mid s)/\pi_{\mathrm{roll}}(a\mid s))$,
with unrestricted improvement distribution $q$ and frozen reference
policy $\pi_{\mathrm{roll}}$. In MPO, this reference can differ from
the behavior policies represented in replay.
On-policy MPO would set $q=\pi_{\mathrm{roll}}$ where the log ratio vanishes,
so its critic estimates unregularized environment returns. MPO uses a KL
constraint to limit the improvement of $q$ away from $\pi_{\mathrm{roll}}$ while
this critic is held fixed \citep{abdolmaleki2018maximum}.
FERPO also uses forward-KL fitting, but learns from on-policy rollouts
with the entropy-augmented return $r(s,a)-\tau\log\pi_\theta(a\mid s)$.
Its critic therefore retains
future entropy contributions even when evaluating the rollout policy,
while a separate $\lambda$ controls policy deviation during improvement.

\FERPORevision{\textbf{Policy fitting with exponential weights.}
Relative entropy policy search (REPS; \citealp{peters2010relative})
is an early example of KL-constrained policy search, deriving
exponential reweighting of the joint state-action distribution.
Advantage-weighted regression (AWR; \citealp{peng2019advantage}) and
advantage-weighted actor--critic (AWAC; \citealp{nair2020awac}) fit
policies to replay actions using exponential advantage weights.}
AWR can learn online from scratch, while AWAC targets offline
pretraining followed by online fine-tuning. Their underlying target is
$q^*(a\mid s)\propto\pi_{\mathrm{beh}}(a\mid s)\exp(A_{\mathrm{env}}(s,a)/\lambda)$,
where $\pi_{\mathrm{beh}}$ is the conditional action
distribution represented in replay, potentially combining multiple
behavior policies, and $A_{\mathrm{env}}$ estimates unregularized
advantage. \FERPORevision{Their weighting coefficient $\lambda$ arises as a KL multiplier
but is held fixed in both methods.} The state-dependent advantage
baseline cancels on normalization, so FERPO at $\tau=0$ has the same
target for a common reference and critic.
AWR and AWAC omit the state-dependent normalizer,
thereby reweighting states in the policy-fitting objective
(see Appendix~\ref{app:state_normalization} for details). FERPO instead learns from
fresh on-policy rollouts, samples candidate actions from the frozen
rollout policy, and uses SNIS to normalize their weights separately
within each state. For $\tau>0$, FERPO
also favors less likely actions at equal critic value, providing an
explicit exploration incentive.
\FERPORevision{Policy Optimization via Importance Sampling
(POIS; \citealp{metelli2018policy}) relates ESS to R\'enyi-2
divergence and uses this connection to penalize uncertainty
in importance-sampled return estimates.}
\FERPORevision{\citet{watson2023inferring} use ESS to adapt exponential weighting
in sampling-based control and relate ESS to R\'enyi-2 divergence.}
\FERPORevision{Discrete Autoencoding via Policy Search
(DAPS; \citealp{drolet2026daps}) combines an entropy- and
KL-regularized target of the same form as Eq.~\ref{eq:optimal_target}
with self-normalized importance-weighted maximum likelihood and
ESS-based multiplier adaptation for discrete variational encoders.
DAPS matches mean ESS, whereas FERPO regulates mean log normalized ESS.
DAPS also relates population ESS to R\'enyi-2 divergence and bounds
the target's KL from its sampling reference, corresponding to
$\KL(q^*\|\pi_{\mathrm{roll}})$ in our paper.
FERPO further relates ESS to an upper bound on the TV between
the current and rollout policies
(Proposition~\ref{prop:ess_actor_tv}), linking ESS-based trust-region
control to TRPO-style policy-improvement bounds
(Section~\ref{sec:problem_description}).}

\begingroup
\section{Experiments}
\label{sec:experiments}

\subsection{Update directions and sampling}
\label{sec:motivating_experiment}

We compare actor updates on the bimodal cost $c(a)=(a^2-1)^2-0.1a$,
whose right well is better. We use an exact critic $Q=-c$ and Gaussian
policies $\pi_\mu=\mathcal N(\mu,\sigma^2)$, with fixed $\sigma=0.4$,
$\tau=0.1$, $\lambda=0.05$, and $T=\tau+\lambda$.

\textbf{Forward fitting changes the update direction.}
Let $\nu$ denote the candidate mean and $m_*:=\E_{q_\mu^*}[a]$.
We differentiate with respect to $\nu$, holding the reference $\pi_\mu$
and target $q_\mu^*$ fixed. At $\nu=\mu$, REPPO's KL regularization is inactive:
\begin{equation}
  \begin{aligned}
    \left.\nabla_\nu L_{\mathrm{REPPO}}(\nu)\right|_{\nu=\mu}
      &=T\left.\nabla_\nu\KL(\pi_\nu\|q_\mu^*)\right|_{\nu=\mu}
        =\E_{\pi_\mu}[c'(a)],\\
    \left.\nabla_\nu L_{\mathrm{FERPO}}(\nu;\mu)\right|_{\nu=\mu}
      &=\left.\nabla_\nu\KL(q_\mu^*\|\pi_\nu)\right|_{\nu=\mu}
        =\frac{\mu-m_*}{\sigma^2}
        =\frac{1}{\lambda}\E_{q_\mu^*}[c'(a)].
  \end{aligned}
  \label{eq:bimodal_update_directions}
\end{equation}
Both methods average the same cost derivative under different
distributions (derivation: Appendix~\ref{app:projection_geometry}).
FERPO's value-based target can weight distant low-cost regions:
at B ($\mu=-0.6$), $\E_{\pi_\mu}[c']=0.284$ and
$\E_{q_\mu^*}[c']=-0.0419$. Gradient descent moves REPPO left and
FERPO right toward $m_*\simeq-0.466$
(Figure~\ref{fig:bimodal_actor_losses}).

\begin{figure}[t]
  \centering
  \setlength{\abovecaptionskip}{2pt}
  \includegraphics[width=0.82\textwidth]{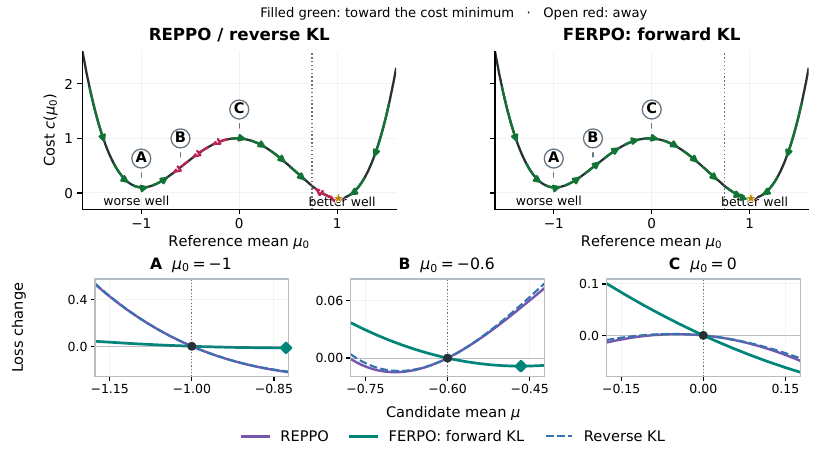}
  \caption{\textbf{Actor directions and local losses.}
  Top: green/red arrows point toward/away from the cost minimum
  (star); lengths are illustrative, not convergence guarantees.
  Bottom: frozen-reference losses at A--C, zeroed at the reference;
  KL losses are scaled by $T$.
  Details: Appendix~\ref{app:bimodal_protocol}.}
  \label{fig:bimodal_actor_losses}
\end{figure}

\textbf{Recovering the direction with SNIS.}
At B, SNIS with $K=1{,}024$ rollout-policy samples recovers the
rightward direction in only $26.4\%$ of batches: the better well
is rarely sampled. Keeping the target fixed and sampling from
$r_\kappa=\mathcal N(\mu,\kappa^2\sigma^2)$ with proposal-corrected
importance weights raises recovery to $98.9\%$ for $\kappa=2$
and $99.9\%$ for $\kappa=3$.
Appendix~\ref{app:bimodal_sampling} gives the estimator and sampling
results (Figure~\ref{fig:bimodal_snis}).

\subsection{Computational cost}
\label{sec:computational_cost}

We time updates on fixed synthetic rollouts using REPPO's published
configuration \citep{voelcker2026relative}, including its 16 KL samples
per state. Uncached FERPO samples $K$ actions and evaluates $Q$ at every
actor epoch; cached FERPO reuses actions, rollout log probabilities,
and $Q$ values after critic training, recomputing weights as the
multipliers change. Cached actor updates are faster than REPPO for
$K=4$--$64$; uncached updates become slower at $K\geq32$
(Figure~\ref{fig:actor_cost_main}; protocol in Appendix~\ref{app:actor_cost}).
At $K=16$, uncached FERPO reduces combined actor-and-critic update time
by 6.9\% relative to REPPO. Caching provides a further 16.4\% reduction,
bringing the total saving to 22.2\%. A drawback is higher GPU memory
usage: on CheetahRun, peak GPU memory usage is 22.4\% higher for
uncached FERPO and 25.6\% higher for cached FERPO than for REPPO,
because PyTorch retains additional memory after updates.

\begin{figure}[t]
  \centering
  \setlength{\abovecaptionskip}{2pt}
  \includegraphics[width=\textwidth]{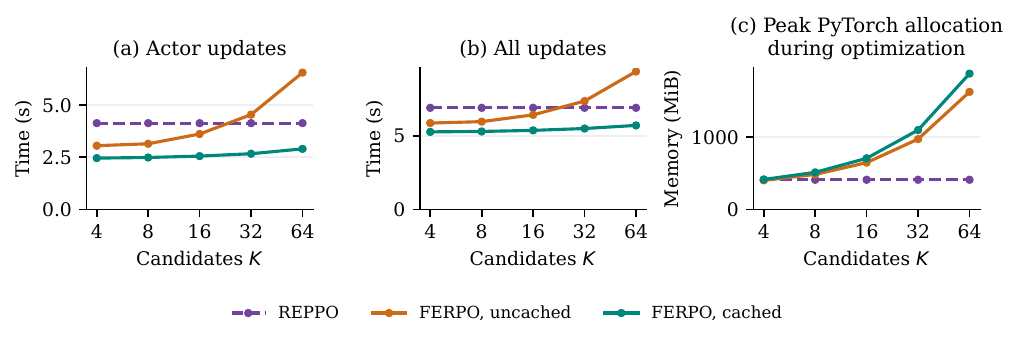}
  \caption{\textbf{Update cost on an A100.}
  Medians and ranges over three repetitions with 131,072 rollout states.
  Actor time includes cache construction; total updates include critic
  training. PyTorch allocation excludes simulation and unused reservations.}
  \label{fig:actor_cost_main}
\end{figure}

\subsection{Ablations}
\label{sec:ablations}
\label{sec:critic_ablations}

\textbf{Number of SNIS samples.}
To test whether a larger sampling budget improves learning, we vary
the number of rollout-policy action samples per state,
$K\in\{2,4,8,16\}$, on G1 flat terrain with adaptive temperature.
Mean return improves as $K$ increases, with smaller gains at larger
sample counts suggesting diminishing returns. At 50.07M environment
steps, increasing $K$ from 8 to 16 raises mean return from
$34.32\pm0.18$ to $34.81\pm0.24$ (mean $\pm$ standard error of the
mean, SEM; Figure~\ref{fig:learning_k}, left), while doubling the
number of candidate actions evaluated.

\textbf{Temperature.}
To assess the contribution of entropy regularization, we compare adaptive
temperature with $\tau=0$ at $K=16$ using the same G1 evaluation budget.
Without entropy regularization, mean return falls to $30.32\pm0.27$
(mean $\pm$ SEM; Figure~\ref{fig:learning_k}, left), 4.49 points below
the adaptive control. FERPO's maximum-entropy formulation therefore
achieves higher mean return than the variant without entropy regularization
on this task (see Appendix~\ref{app:learning_ablation_details} for details).

\begin{samepage}
{\textbf{ESS versus KL control.}
On HopperHop and AcrobotSwingupSparse, both controllers support learning
at suitable settings, with task-dependent best returns. We retain ESS
control because it directly regulates SNIS weight concentration and
shows lower temporal variation at comparable achieved ESS in these
experiments (Appendix~\ref{app:ess_kl_ablation}).}\par
\end{samepage}

\textbf{Cached $Q$ values.}
To test whether caching speeds up training while retaining performance,
we complete critic training before the actor phase, freeze the critic,
and cache candidate actions, $Q$ values, and rollout log probabilities.
Actor updates reuse these quantities while recomputing SNIS weights.
On G1 with $K=8$, cached and uncached updates reach $33.58\pm0.40$ and
$34.03\pm0.21$ at 50.07M steps (mean $\pm$ SEM over five seeds;
Figure~\ref{fig:cached_q_ablation}, middle).
Caching reduces mean time to
complete this budget from 68.4 to 55.7 minutes, an 18.6\% saving.
It therefore trades a modest observed return gap for faster training. Details are provided in Appendix~\ref{app:cached_q_timing}.

\textbf{Cached $V$ extension.}
A learned $Q$ function must approximate both the immediate reward and
future returns, and may therefore fail to preserve the reward's shape
across actions. Learning this shape can be particularly difficult in
sparse reward settings, where few sampled transitions yield nonzero
rewards. To mitigate this source of error, we extend the cached $Q$
variant above by replacing its critic with a soft state-value function.
For each rollout state, we simulate $K$ candidate actions
and, after fitting and freezing $V$, cache $r(s,a)+\gamma V(s')$
{as a one-sample Monte Carlo estimate of $Q_V(s,a)$}
for reuse throughout actor updates. This incorporates the observed
immediate reward directly, at the cost of $K$ additional environment
calls per rollout state.
In massively parallel online RL, these independent one-step simulations
can be batched on the GPU \citep{zakka2025mujocoplayground}, and caching
amortizes their cost across all actor epochs.
With Warp graph execution and batched candidate simulation
during rollout, cached $V$ is only 8.4\% slower per training iteration
than uncached $Q$ on AcrobotSwingupSparse with $K=16$: 8.51 versus
7.85 seconds on an A100, measured after warmup. The $K$ additional transitions nevertheless increase
the total simulation budget by a factor of $K+1=17$.
{Figure~\ref{fig:value_acrobot_ablation} (right) compares returns over
12 seeds. See
Appendix~\ref{app:state_value_ablation} for more details.}

\begin{figure}[t]
  \centering
  \includegraphics[width=\textwidth]{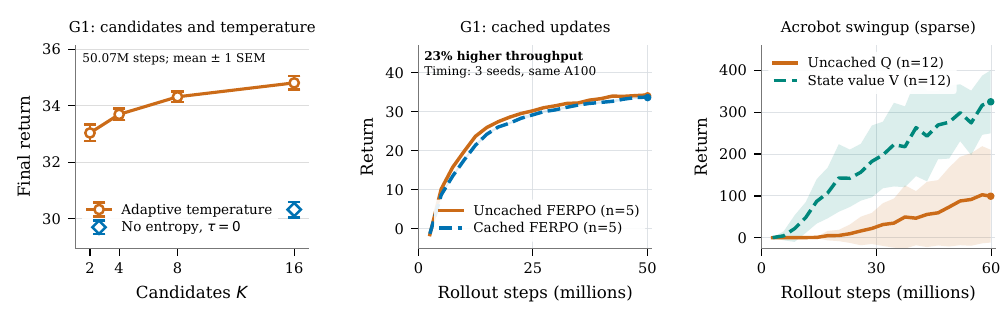}
  \caption{\textbf{Ablations.} Left: G1 candidate count and
  temperature at 50.07M steps (eight seeds for $K=2,16$ and $\tau=0$;
  ten for $K=4,8$). Middle: cached versus uncached G1 updates
  ($K=8$). Right: state-value versus $Q$ critics on AcrobotSwingupSparse
  (sparse rewards, $K=16$).
  Means with one SEM (left, middle) or one SD (right); panels use separate
  controls. The state-value variant uses $17\times$ more simulator transitions
  per rollout step.}
  \label{fig:learning_k}
  \label{fig:cached_q_ablation}
  \label{fig:value_acrobot_ablation}
\end{figure}

\subsection{Learning benchmarks}
\label{sec:learning_benchmarks}

We evaluate whether FERPO's policy updates translate into competitive
final performance across three benchmark suites: 22 DeepMind Control
Suite (DMC) tasks \citep{tassa2018deepmindcontrolsuite} and four G1/T1
locomotion tasks in MuJoCo Playground \citep{zakka2025mujocoplayground},
and eight ManiSkill manipulation tasks \citep{tao2025maniskill3}.
All baseline results come from the experiments accompanying the REPPO
paper \citep{voelcker2026relative}. For each benchmark suite, we train
FERPO with exactly the same actor and critic architectures as REPPO.
Alongside REPPO, we compare with PPO \citep{schulman_proximal_2017}
and PPO (Brax) separately on DMC, PPO on ManiSkill, and FastTD3
\citep{seo2025fasttd3}, a scalable variant of TD3
\citep{pmlr-v80-fujimoto18a}, on DMC.
For each of these additional baseline families, we select one configuration
per suite using performance at 50M steps and keep that configuration
fixed across the suite's tasks.
Appendix~\ref{app:learning_results} provides the selection protocol and
all per-environment learning curves.
{The G1/T1 FERPO results use $K=16$ and normalized ESS
target $0.65$, with eight completed seeds per task.}

Figure~\ref{fig:learning_aggregate} aggregates results separately within
each suite using the interquartile mean (IQM). Every environment receives
equal total weight, divided equally among its available seeds; returns
and success rates retain their original units. We focus the comparison
on the final common evaluation budget: 50M interactions for DMC and
G1/T1, and 49,971,200 interactions for ManiSkill.
Individual IQM uncertainties use pointwise 95\% bootstrap confidence
intervals from 10,000 seed resamples within
environments, matching Figure~\ref{fig:learning_aggregate}.
Below, estimate $\pm$ half the bootstrap interval width gives a symmetric
approximation; all significance statements use the original interval bounds.

\begin{figure}[t]
  \centering
  \includegraphics[width=\textwidth]{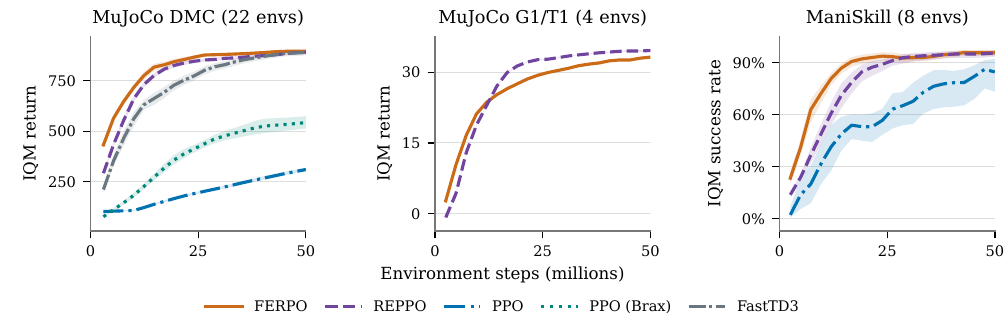}
  \caption{\textbf{Benchmark learning curves.} IQM with equal environment
  weights and equal seed weights within environments. Bands: pointwise
  95\% confidence intervals from 10,000 within-environment bootstrap resamples.
  {G1/T1 FERPO uses $K=16$, ESS target $0.65$, and eight
  seeds per task.}
  Details: Appendix~\ref{app:learning_results}.}
  \label{fig:learning_aggregate}
\end{figure}

\textbf{Final performance.}
On DMC, FERPO reaches an IQM return of $897.7\pm6.7$.
REPPO reaches $891.0\pm6.4$, FastTD3 $890.6\pm12.8$,
PPO (Brax) $542.8\pm30.2$, and PPO $309.8\pm9.4$.
FERPO therefore has the highest point estimate, with a small numerical
lead over REPPO and FastTD3 and a much larger lead over both PPO
families.
{On G1/T1, REPPO reaches $34.63\pm0.17$, compared with
$33.22\pm0.28$ for FERPO. FERPO's endpoint estimate is slightly
higher on T1 flat and similar on T1 rough, while REPPO leads on both
G1 tasks (Figure~\ref{fig:learning_locomotion}).} On ManiSkill, FERPO's success-rate
IQM is $(95.67\pm1.83)\%$; REPPO reaches $(95.52\pm1.85)\%$,
and PPO reaches $(84.61\pm9.49)\%$. The two leading ManiSkill point
estimates differ by only 0.15 percentage points.

Beyond final performance, FERPO reaches high returns on DMC and high
success rates on ManiSkill with fewer environment interactions than
REPPO (Figure~\ref{fig:learning_aggregate}). This sample-efficiency
advantage complements its lower optimization cost: at $K=16$,
combined actor-and-critic updates take 6.9\% less time in the matched
A100 benchmark (Section~\ref{sec:computational_cost}). Together, these
results suggest improved wall-clock efficiency.

\endgroup

\section{Conclusion, limitations, and future work}
\label{sec:conclusion}

FERPO performs on-policy maximum-entropy policy improvement by fitting
an entropy- and KL-regularized target through forward KL. Critic values
weight actions sampled from the rollout policy using self-normalized
importance sampling, avoiding critic action derivatives. An ESS
controller regulates the concentration of these weights.

On DMC and ManiSkill, FERPO reaches high performance with fewer
environment interactions than REPPO and achieves competitive final
results. At $K=16$, combined actor-and-critic update time is 6.9\%
lower than REPPO for uncached FERPO and we also evaluate cached FERPO with a 22.2\% actor update speedup. Ablations on G1 flat terrain highlight
the benefit of entropy regularization and gains from
larger SNIS candidate sets. Limitations include critic-value errors,
finite-sample importance-sampling errors, restricted Gaussian policies,
and higher GPU memory use. {With eight seeds per task,
FERPO remains below REPPO on the G1/T1 suite aggregate at the same
environment-interaction budget. The gap is concentrated on the G1
tasks; T1 endpoints are similar or slightly higher for FERPO under
the recorded evaluation protocols.}

{State-value FERPO combines one-step simulation with a
learned state-value critic to score candidate actions. {Its higher returns
on AcrobotSwingupSparse and the modest iteration-time overhead measured
for its optimized implementation make it a particularly promising
direction for future work.}}
{Two further directions are FERPO in discrete action spaces and combining FERPO with diffusion policies.}
For small discrete action spaces, target normalization and the forward-KL
objective could be evaluated by exact summation over actions and larger ones through sampling.
Diffusion-augmented MDPs (DA-MDPs)
\citep{sanokowski2026diffusionaugmentedmarkovdecisionprocesses}
offer a promising and relatively straightforward route to combining
FERPO with expressive diffusion policies: FERPO's forward-KL updates
could be applied to individual denoising transitions in the augmented MDP.

\section*{Acknowledgements}
We acknowledge the EuroHPC Joint Undertaking for providing access to
the Vega supercomputer at the Institute of Information Science (IZUM),
Slovenia, and the Deucalion supercomputer at the Minho Advanced Computing
Centre (MACC), Portugal. These computing resources supported the
experiments reported in this work.

\clearpage
\bibliographystyle{styles/iclr2027_conference}
\bibliography{bibliography/references}

\begin{thebibliography}{41}
\providecommand{\natexlab}[1]{#1}
\providecommand{\url}[1]{\texttt{#1}}
\expandafter\ifx\csname urlstyle\endcsname\relax
  \providecommand{\doi}[1]{doi: #1}\else
  \providecommand{\doi}{doi: \begingroup \urlstyle{rm}\Url}\fi

\bibitem[Abdolmaleki et~al.(2018)Abdolmaleki, Springenberg, Tassa, Munos,
  Heess, and Riedmiller]{abdolmaleki2018maximum}
Abbas Abdolmaleki, Jost~Tobias Springenberg, Yuval Tassa, Remi Munos, Nicolas
  Heess, and Martin Riedmiller.
\newblock Maximum a posteriori policy optimisation.
\newblock In \emph{International Conference on Learning Representations}, 2018.
\newblock URL \url{https://arxiv.org/abs/1806.06920}.

\bibitem[Blessing et~al.(2025)Blessing, Berner, Richter, Domingo~i Enrich, Du,
  Vahdat, and Neumann]{NEURIPS2025_f1ccd13e}
Denis Blessing, Julius Berner, Lorenz Richter, Carles Domingo~i Enrich, Yuanqi
  Du, Arash Vahdat, and Gerhard Neumann.
\newblock Trust region constrained measure transport in path space for
  stochastic optimal control and inference.
\newblock In D.~Belgrave, C.~Zhang, H.~Lin, R.~Pascanu, P.~Koniusz,
  M.~Ghassemi, and N.~Chen (eds.), \emph{Advances in Neural Information
  Processing Systems}, volume 38, Main Conference, pp.\  165462--165510. Curran
  Associates, Inc., 2025.
\newblock \doi{10.52202/085713-5516}.
\newblock URL
  \url{https://proceedings.neurips.cc/paper_files/paper/2025/file/f1ccd13e51f610ff4a7a08db9a8b458b-Paper-Conference.pdf}.

\bibitem[Cayci et~al.(2024)Cayci, He, and Srikant]{cayci2024convergence}
Semih Cayci, Niao He, and R.~Srikant.
\newblock Convergence of entropy-regularized natural policy gradient with
  linear function approximation.
\newblock \emph{SIAM Journal on Optimization}, 34\penalty0 (3):\penalty0
  2729--2755, 2024.
\newblock \doi{10.1137/22M1540156}.
\newblock URL \url{https://doi.org/10.1137/22M1540156}.

\bibitem[Devroye et~al.(2023)Devroye, Mehrabian, and
  Reddad]{devroye2023gaussiantv}
Luc Devroye, Abbas Mehrabian, and Tommy Reddad.
\newblock The total variation distance between high-dimensional {Gaussians}
  with the same mean, 2023.
\newblock URL \url{https://arxiv.org/abs/1810.08693v7}.
\newblock Version 7.

\bibitem[D'Oro \& Ja\'{s}kowski(2020)D'Oro and
  Ja\'{s}kowski]{NEURIPS2020_03255088}
Pierluca D'Oro and Wojciech Ja\'{s}kowski.
\newblock How to learn a useful critic? model-based action-gradient-estimator
  policy optimization.
\newblock In H.~Larochelle, M.~Ranzato, R.~Hadsell, M.F. Balcan, and H.~Lin
  (eds.), \emph{Advances in Neural Information Processing Systems}, volume~33,
  pp.\  313--324. Curran Associates, Inc., 2020.
\newblock URL
  \url{https://proceedings.neurips.cc/paper_files/paper/2020/file/03255088ed63354a54e0e5ed957e9008-Paper.pdf}.

\bibitem[Drolet et~al.(2026)Drolet, Al-Hafez, Bhatt, Peters, and
  Arenz]{drolet2026daps}
Michael Drolet, Firas Al-Hafez, Aditya Bhatt, Jan Peters, and Oleg Arenz.
\newblock Discrete variational autoencoding via policy search.
\newblock In \emph{International Conference on Learning Representations}, 2026.
\newblock URL \url{https://openreview.net/forum?id=wJhhCmbFzY}.

\bibitem[Farebrother et~al.(2024)Farebrother, Orbay, Vuong, Ali~Ta{\"i}ga,
  Chebotar, Xiao, Irpan, Levine, Castro, Faust, Kumar, and
  Agarwal]{farebrother2024stop}
Jesse Farebrother, Jordi Orbay, Quan Vuong, Adrien Ali~Ta{\"i}ga, Yevgen
  Chebotar, Ted Xiao, Alex Irpan, Sergey Levine, Pablo~Samuel Castro,
  Aleksandra Faust, Aviral Kumar, and Rishabh Agarwal.
\newblock Stop regressing: Training value functions via classification for
  scalable deep {RL}.
\newblock In \emph{Proceedings of the 41st International Conference on Machine
  Learning}, volume 235 of \emph{Proceedings of Machine Learning Research},
  pp.\  13049--13071. PMLR, 2024.
\newblock URL \url{https://proceedings.mlr.press/v235/farebrother24a.html}.

\bibitem[Fujimoto et~al.(2018)Fujimoto, van Hoof, and
  Meger]{pmlr-v80-fujimoto18a}
Scott Fujimoto, Herke van Hoof, and David Meger.
\newblock Addressing function approximation error in actor-critic methods.
\newblock In Jennifer Dy and Andreas Krause (eds.), \emph{Proceedings of the
  35th International Conference on Machine Learning}, volume~80 of
  \emph{Proceedings of Machine Learning Research}, pp.\  1587--1596. PMLR,
  10--15 Jul 2018.
\newblock URL \url{https://proceedings.mlr.press/v80/fujimoto18a.html}.

\bibitem[Fujimoto et~al.(2025)Fujimoto, D'Oro, Zhang, Tian, and
  Rabbat]{fujimoto2025towards}
Scott Fujimoto, Pierluca D'Oro, Amy Zhang, Yuandong Tian, and Michael Rabbat.
\newblock Towards general-purpose model-free reinforcement learning.
\newblock In Y.~Yue, A.~Garg, N.~Peng, F.~Sha, and R.~Yu (eds.),
  \emph{International Conference on Learning Representations}, volume 2025,
  pp.\  83509--83552, 2025.
\newblock URL
  \url{https://proceedings.iclr.cc/paper_files/paper/2025/file/cf6501108fced72ee5c47e2151c4e153-Paper-Conference.pdf}.

\bibitem[Gurumurthy et~al.(2023)Gurumurthy, Manchester, and
  Kolter]{pmlr-v211-gurumurthy23a}
Swaminathan Gurumurthy, Zachary Manchester, and J~Zico Kolter.
\newblock Practical critic gradient based actor critic for on-policy
  reinforcement learning.
\newblock In Nikolai Matni, Manfred Morari, and George~J. Pappas (eds.),
  \emph{Proceedings of The 5th Annual Learning for Dynamics and Control
  Conference}, volume 211 of \emph{Proceedings of Machine Learning Research},
  pp.\  625--638. PMLR, 15--16 Jun 2023.
\newblock URL \url{https://proceedings.mlr.press/v211/gurumurthy23a.html}.

\bibitem[Ha et~al.(2025)Ha, Lee, van~de Panne, Xie, Yu, and
  Khadiv]{ha2025learning}
Sehoon Ha, Joonho Lee, Michiel van~de Panne, Zhaoming Xie, Wenhao Yu, and Majid
  Khadiv.
\newblock Learning-based legged locomotion: State of the art and future
  perspectives.
\newblock \emph{The International Journal of Robotics Research}, 44\penalty0
  (8):\penalty0 1396--1427, 2025.

\bibitem[Haarnoja et~al.(2017)Haarnoja, Tang, Abbeel, and
  Levine]{haarnoja2017reinforcement}
Tuomas Haarnoja, Haoran Tang, Pieter Abbeel, and Sergey Levine.
\newblock Reinforcement learning with deep energy-based policies.
\newblock In \emph{Proceedings of the 34th International Conference on Machine
  Learning}, volume~70 of \emph{Proceedings of Machine Learning Research}, pp.\
   1352--1361. PMLR, 2017.
\newblock URL \url{https://proceedings.mlr.press/v70/haarnoja17a.html}.

\bibitem[Haarnoja et~al.(2018{\natexlab{a}})Haarnoja, Zhou, Abbeel, and
  Levine]{haarnoja2018soft}
Tuomas Haarnoja, Aurick Zhou, Pieter Abbeel, and Sergey Levine.
\newblock Soft actor-critic: Off-policy maximum entropy deep reinforcement
  learning with a stochastic actor.
\newblock In \emph{Proceedings of the 35th International Conference on Machine
  Learning}, volume~80 of \emph{Proceedings of Machine Learning Research}, pp.\
   1861--1870. PMLR, 2018{\natexlab{a}}.
\newblock URL \url{https://proceedings.mlr.press/v80/haarnoja18b.html}.

\bibitem[Haarnoja et~al.(2018{\natexlab{b}})Haarnoja, Zhou, Hartikainen,
  Tucker, Ha, Tan, Kumar, Zhu, Gupta, Abbeel, and
  Levine]{haarnoja2018applications}
Tuomas Haarnoja, Aurick Zhou, Kristian Hartikainen, George Tucker, Sehoon Ha,
  Jie Tan, Vikash Kumar, Henry Zhu, Abhishek Gupta, Pieter Abbeel, and Sergey
  Levine.
\newblock Soft actor-critic algorithms and applications.
\newblock \emph{arXiv preprint arXiv:1812.05905}, 2018{\natexlab{b}}.
\newblock URL \url{https://arxiv.org/abs/1812.05905}.

\bibitem[Kaufmann et~al.(2023)Kaufmann, Bauersfeld, Loquercio, M{\"u}ller,
  Koltun, and Scaramuzza]{Kaufmann_2023}
Elia Kaufmann, Leonard Bauersfeld, Antonio Loquercio, Matthias M{\"u}ller,
  Vladlen Koltun, and Davide Scaramuzza.
\newblock Champion-level drone racing using deep reinforcement learning.
\newblock \emph{Nature}, 620\penalty0 (7976):\penalty0 982--987, August 2023.
\newblock ISSN 1476-4687.
\newblock \doi{10.1038/s41586-023-06419-4}.
\newblock URL \url{https://doi.org/10.1038/s41586-023-06419-4}.

\bibitem[Khan \& Rue(2023)Khan and Rue]{khan2023bayesian}
Mohammad~Emtiyaz Khan and H{\aa}vard Rue.
\newblock The {Bayesian} learning rule.
\newblock \emph{Journal of Machine Learning Research}, 24\penalty0
  (281):\penalty0 1--46, 2023.
\newblock URL \url{https://jmlr.org/papers/v24/22-0291.html}.

\bibitem[Metelli et~al.(2018)Metelli, Papini, Faccio, and
  Restelli]{metelli2018policy}
Alberto~Maria Metelli, Matteo Papini, Francesco Faccio, and Marcello Restelli.
\newblock Policy optimization via importance sampling.
\newblock In \emph{Advances in Neural Information Processing Systems},
  volume~31, 2018.
\newblock URL
  \url{https://papers.nips.cc/paper_files/paper/2018/hash/6aed000af86a084f9cb0264161e29dd3-Abstract.html}.

\bibitem[Minka(2005)]{minka2005divergence}
Tom Minka.
\newblock Divergence measures and message passing.
\newblock Technical Report MSR-TR-2005-173, Microsoft Research, 2005.
\newblock URL
  \url{https://www.microsoft.com/en-us/research/publication/divergence-measures-and-message-passing/}.

\bibitem[Nair et~al.(2020)Nair, Gupta, Dalal, and Levine]{nair2020awac}
Ashvin Nair, Abhishek Gupta, Murtaza Dalal, and Sergey Levine.
\newblock {AWAC}: Accelerating online reinforcement learning with offline
  datasets.
\newblock \emph{arXiv preprint arXiv:2006.09359}, 2020.
\newblock URL \url{https://arxiv.org/abs/2006.09359}.

\bibitem[Owen(2013)]{owen2013monte}
Art~B. Owen.
\newblock \emph{Monte Carlo Theory, Methods and Examples}.
\newblock Online book, 2013.
\newblock URL \url{https://artowen.su.domains/mc/}.

\bibitem[Peng et~al.(2019)Peng, Kumar, Zhang, and Levine]{peng2019advantage}
Xue~Bin Peng, Aviral Kumar, Grace Zhang, and Sergey Levine.
\newblock Advantage-weighted regression: Simple and scalable off-policy
  reinforcement learning.
\newblock \emph{arXiv preprint arXiv:1910.00177}, 2019.
\newblock URL \url{https://arxiv.org/abs/1910.00177}.

\bibitem[Peters et~al.(2010)Peters, M{\"u}lling, and Altun]{peters2010relative}
Jan Peters, Katharina M{\"u}lling, and Yasemin Altun.
\newblock Relative entropy policy search.
\newblock In \emph{Proceedings of the AAAI Conference on Artificial
  Intelligence}, volume~24, pp.\  1607--1612, 2010.
\newblock \doi{10.1609/aaai.v24i1.7727}.
\newblock URL \url{https://ojs.aaai.org/index.php/AAAI/article/view/7727}.

\bibitem[Rudin et~al.(2022)Rudin, Hoeller, Reist, and
  Hutter]{pmlr-v164-rudin22a}
Nikita Rudin, David Hoeller, Philipp Reist, and Marco Hutter.
\newblock Learning to walk in minutes using massively parallel deep
  reinforcement learning.
\newblock In Aleksandra Faust, David Hsu, and Gerhard Neumann (eds.),
  \emph{Proceedings of the 5th Conference on Robot Learning}, volume 164 of
  \emph{Proceedings of Machine Learning Research}, pp.\  91--100. PMLR, 08--11
  Nov 2022.
\newblock URL \url{https://proceedings.mlr.press/v164/rudin22a.html}.

\bibitem[Sanokowski \& Patil(2026)Sanokowski and
  Patil]{sanokowski2026diffusionaugmentedmarkovdecisionprocesses}
Sebastian Sanokowski and Kaustubh Patil.
\newblock Diffusion-augmented {Markov} decision processes for maximum entropy
  reinforcement learning, 2026.
\newblock URL \url{https://arxiv.org/abs/2512.02019}.

\bibitem[Sarafian et~al.(2021)Sarafian, Keynan, and
  Kraus]{pmlr-v139-sarafian21a}
Elad Sarafian, Shai Keynan, and Sarit Kraus.
\newblock Recomposing the reinforcement learning building blocks with
  hypernetworks.
\newblock In Marina Meila and Tong Zhang (eds.), \emph{Proceedings of the 38th
  International Conference on Machine Learning}, volume 139 of
  \emph{Proceedings of Machine Learning Research}, pp.\  9301--9312. PMLR,
  18--24 Jul 2021.
\newblock URL \url{https://proceedings.mlr.press/v139/sarafian21a.html}.

\bibitem[Schulman et~al.(2015)Schulman, Levine, Abbeel, Jordan, and
  Moritz]{schulman2015trust}
John Schulman, Sergey Levine, Pieter Abbeel, Michael Jordan, and Philipp
  Moritz.
\newblock Trust region policy optimization.
\newblock In \emph{Proceedings of the 32nd International Conference on Machine
  Learning}, volume~37 of \emph{Proceedings of Machine Learning Research}, pp.\
   1889--1897. PMLR, 2015.
\newblock URL \url{https://proceedings.mlr.press/v37/schulman15.html}.

\bibitem[Schulman et~al.(2017)Schulman, Wolski, Dhariwal, Radford, and
  Klimov]{schulman_proximal_2017}
John Schulman, Filip Wolski, Prafulla Dhariwal, Alec Radford, and Oleg Klimov.
\newblock Proximal policy optimization algorithms.
\newblock \emph{arXiv preprint arXiv:1707.06347}, 2017.
\newblock URL \url{https://arxiv.org/abs/1707.06347}.

\bibitem[Schwarzer et~al.(2021)Schwarzer, Anand, Goel, Hjelm, Courville, and
  Bachman]{schwarzer2021dataefficient}
Max Schwarzer, Ankesh Anand, Rishab Goel, R.~Devon Hjelm, Aaron Courville, and
  Philip Bachman.
\newblock Data-efficient reinforcement learning with self-predictive
  representations.
\newblock In \emph{International Conference on Learning Representations}, 2021.
\newblock URL \url{https://openreview.net/forum?id=uCQfPZwRaUu}.

\bibitem[Seo et~al.(2025)Seo, Sferrazza, Geng, Nauman, Yin, and
  Abbeel]{seo2025fasttd3}
Younggyo Seo, Carmelo Sferrazza, Haoran Geng, Michal Nauman, Zhao-Heng Yin, and
  Pieter Abbeel.
\newblock {FastTD3}: Simple, fast, and capable reinforcement learning for
  humanoid control.
\newblock \emph{arXiv preprint arXiv:2505.22642}, 2025.
\newblock URL \url{https://arxiv.org/abs/2505.22642}.

\bibitem[Silver et~al.(2017)Silver, Schrittwieser, Simonyan, Antonoglou, Huang,
  Guez, Hubert, Baker, Lai, Bolton, Chen, Lillicrap, Hui, Sifre, van~den
  Driessche, Graepel, and Hassabis]{Silver_2017}
David Silver, Julian Schrittwieser, Karen Simonyan, Ioannis Antonoglou, Aja
  Huang, Arthur Guez, Thomas Hubert, Lucas Baker, Matthew Lai, Adrian Bolton,
  Yutian Chen, Timothy Lillicrap, Fan Hui, Laurent Sifre, George van~den
  Driessche, Thore Graepel, and Demis Hassabis.
\newblock Mastering the game of {Go} without human knowledge.
\newblock \emph{Nature}, 550\penalty0 (7676):\penalty0 354--359, October 2017.
\newblock ISSN 1476-4687.
\newblock \doi{10.1038/nature24270}.
\newblock URL \url{https://doi.org/10.1038/nature24270}.

\bibitem[Sutton et~al.(1999)Sutton, McAllester, Singh, and
  Mansour]{sutton1999policy}
Richard~S. Sutton, David~A. McAllester, Satinder~P. Singh, and Yishay Mansour.
\newblock Policy gradient methods for reinforcement learning with function
  approximation.
\newblock In \emph{Advances in Neural Information Processing Systems},
  volume~12, 1999.
\newblock URL
  \url{https://proceedings.neurips.cc/paper/1999/hash/464d828b85b0bed98e80ade0a5c43b0f-Abstract.html}.

\bibitem[Tao et~al.(2025)Tao, Xiang, Shukla, Qin, Hinrichsen, Yuan, Bao, Lin,
  Liu, Chan, Gao, Li, Mu, Xiao, Gurha, Rajesh, Choi, Chen, Huang, Calandra,
  Chen, Luo, and Su]{tao2025maniskill3}
Stone Tao, Fanbo Xiang, Arth Shukla, Yuzhe Qin, Xander Hinrichsen, Xiaodi Yuan,
  Chen Bao, Xinsong Lin, Yulin Liu, Tse-Kai Chan, Yuan Gao, Xuanlin Li,
  Tongzhou Mu, Nan Xiao, Arnav Gurha, Viswesh~Nagaswamy Rajesh, Yong~Woo Choi,
  Yen-Ru Chen, Zhiao Huang, Roberto Calandra, Rui Chen, Shan Luo, and Hao Su.
\newblock Demonstrating {GPU} parallelized robot simulation and rendering for
  generalizable embodied {AI} with {ManiSkill3}.
\newblock In \emph{Proceedings of Robotics: Science and Systems}, Los Angeles,
  CA, USA, June 2025.
\newblock \doi{10.15607/RSS.2025.XXI.021}.
\newblock URL \url{https://www.roboticsproceedings.org/rss21/p021.html}.

\bibitem[Tassa et~al.(2018)Tassa, Doron, Muldal, Erez, Li, de~Las~Casas,
  Budden, Abdolmaleki, Merel, Lefrancq, Lillicrap, and
  Riedmiller]{tassa2018deepmindcontrolsuite}
Yuval Tassa, Yotam Doron, Alistair Muldal, Tom Erez, Yazhe Li, Diego
  de~Las~Casas, David Budden, Abbas Abdolmaleki, Josh Merel, Andrew Lefrancq,
  Timothy Lillicrap, and Martin Riedmiller.
\newblock {DeepMind Control Suite}, 2018.
\newblock URL \url{https://arxiv.org/abs/1801.00690}.

\bibitem[van Erven \& Harremo{\"e}s(2014)van Erven and
  Harremo{\"e}s]{vanerven2014renyi}
Tim van Erven and Peter Harremo{\"e}s.
\newblock {R\'enyi} divergence and {Kullback--Leibler} divergence.
\newblock \emph{IEEE Transactions on Information Theory}, 60\penalty0
  (7):\penalty0 3797--3820, 2014.
\newblock \doi{10.1109/TIT.2014.2320500}.
\newblock URL \url{https://arxiv.org/abs/1206.2459}.

\bibitem[Voelcker et~al.(2024)Voelcker, Kastner, Gilitschenski, and
  Farahmand]{voelcker2024when}
Claas~A. Voelcker, Tyler Kastner, Igor Gilitschenski, and Amir-massoud
  Farahmand.
\newblock When does self-prediction help? understanding auxiliary tasks in
  reinforcement learning.
\newblock \emph{Reinforcement Learning Journal}, 4:\penalty0 1567--1597, 2024.
\newblock URL \url{https://rlj.cs.umass.edu/2024/papers/Paper197.html}.

\bibitem[Voelcker et~al.(2026)Voelcker, Brunnbauer, Hussing, Nauman, Abbeel,
  Eaton, Grosu, Farahmand, and Gilitschenski]{voelcker2026relative}
Claas~A. Voelcker, Axel Brunnbauer, Marcel Hussing, Michal Nauman, Pieter
  Abbeel, Eric Eaton, Radu Grosu, Amir-massoud Farahmand, and Igor
  Gilitschenski.
\newblock Relative entropy pathwise policy optimization.
\newblock In \emph{The Fourteenth International Conference on Learning
  Representations}, 2026.
\newblock URL
  \url{https://proceedings.iclr.cc/paper_files/paper/2026/hash/3c2fd72a28fb98facf98546727320249-Abstract-Conference.html}.

\bibitem[von Klitzing et~al.(2026)von Klitzing, Blessing, Schopmans,
  Friederich, and Neumann]{ICLR2026_8f38c9d6}
Christopher von Klitzing, Denis Blessing, Henrik Schopmans, Pascal Friederich,
  and Gerhard Neumann.
\newblock Learning {Boltzmann} generators via constrained mass transport.
\newblock In C.~Vondrick, B.~Hariharan, C.~Raffel, L.~Pinto, D.~Yang, and
  A.~Faust (eds.), \emph{International Conference on Learning Representations},
  volume 2026, pp.\  88771--88810, 2026.
\newblock URL
  \url{https://proceedings.iclr.cc/paper_files/paper/2026/file/8f38c9d65d89ce60a630b2619b383e55-Paper-Conference.pdf}.

\bibitem[Watson \& Peters(2023)Watson and Peters]{watson2023inferring}
Joe Watson and Jan Peters.
\newblock Inferring smooth control: {Monte Carlo} posterior policy iteration
  with {Gaussian} processes.
\newblock In \emph{Proceedings of the 6th Conference on Robot Learning}, volume
  205 of \emph{Proceedings of Machine Learning Research}, pp.\  67--79. PMLR,
  2023.
\newblock URL \url{https://proceedings.mlr.press/v205/watson23a.html}.

\bibitem[Zakka et~al.(2025)Zakka, Tabanpour, Liao, Haiderbhai, Holt, Luo,
  Allshire, Frey, Sreenath, Kahrs, Sferrazza, Tassa, and
  Abbeel]{zakka2025mujocoplayground}
Kevin Zakka, Baruch Tabanpour, Qiayuan Liao, Mustafa Haiderbhai, Samuel Holt,
  Jing~Yuan Luo, Arthur Allshire, Erik Frey, Koushil Sreenath, Lueder~A. Kahrs,
  Carmelo Sferrazza, Yuval Tassa, and Pieter Abbeel.
\newblock {MuJoCo} playground, 2025.
\newblock URL \url{https://arxiv.org/abs/2502.08844}.

\bibitem[Zhang et~al.(2023)Zhang, Pan, Li, Liu, Chen, Liu, and
  Wang]{zhang2023gaussiankl}
Yufeng Zhang, Jialu Pan, Kenli Li, Wanwei Liu, Zhenbang Chen, Xinwang Liu, and
  Ji~Wang.
\newblock On the properties of {Kullback--Leibler} divergence between
  multivariate {Gaussian} distributions.
\newblock In A.~Oh, T.~Naumann, A.~Globerson, K.~Saenko, M.~Hardt, and
  S.~Levine (eds.), \emph{Advances in Neural Information Processing Systems},
  volume~36, pp.\  58152--58165. Curran Associates, Inc., 2023.
\newblock \doi{10.52202/075280-2535}.

\bibitem[Ziebart(2010)]{ziebart2010modeling}
Brian~D. Ziebart.
\newblock \emph{Modeling Purposeful Adaptive Behavior with the Principle of
  Maximum Causal Entropy}.
\newblock PhD thesis, Carnegie Mellon University, December 2010.
\newblock URL
  \url{https://www.cs.cmu.edu/~bziebart/publications/thesis-bziebart.pdf}.

\end{thebibliography}

\clearpage
\appendix
\addtocontents{toc}{\protect\setcounter{tocdepth}{3}}
\begingroup
\renewcommand{\contentsname}{Appendix contents}
\hypersetup{hidelinks,linktoc=all}
\pdfbookmark[1]{Appendix contents}{appendix.contents}
\tableofcontents
\endgroup
\clearpage
\section{Extended paper sections}
\label{app:extended_paper_sections}
\subsection{Motivating experiment details}
\label{app:bimodal_protocol}
\label{app:bimodal_sampling}

We expand on the bimodal experiment in Section~\ref{sec:motivating_experiment}.
The cost is $c(a)=(a^2-1)^2-0.1a$, and the critic is exact,
$Q=-c$. We use Gaussian policies with fixed standard deviation
$\sigma=0.4$, temperature $\tau=0.1$, and FERPO multiplier
$\lambda=0.05$. The reference means are $\mu_0=-1,-0.6,0$ (A--C).
Each loss curve keeps the reference policy and its target fixed
while varying the candidate mean $\nu$.

In Figure~\ref{fig:bimodal_actor_losses}, each direction uses a newly
constructed target, frozen for differentiation. Filled green and open
red arrows indicate geometric direction toward and away from the global
cost minimum, not immediate cost decrease or eventual convergence.
Dotted lines mark the best fixed-width Gaussian mean; black dots mark
the references and diamonds visible FERPO loss minima. All candidates
in the local loss windows satisfy REPPO's KL bound $0.1$.

\textbf{Exact actor losses.}
We compute REPPO's expected cost analytically and FERPO's target
mean by numerical quadrature. REPPO uses its cost-and-entropy loss
while $k(\nu)=\KL(\pi_{\mu_0}\|\pi_\nu)<0.1$, and the loss
$0.1k(\nu)$ otherwise. Figure~\ref{fig:bimodal_full_losses} shows
this full loss at B, together with forward and reverse KL to the
same FERPO target. We scale both KL losses by $T=\tau+\lambda=0.15$
to express them in cost units. At B, the target mean is
$m_*\simeq-0.466$, so forward fitting moves right, whereas REPPO
and reverse KL initially move left.

\textbf{Estimating the target from samples.}
At B, sampling from the rollout policy recovers FERPO's rightward
direction in only $26.4\%$ of batches with $K=1{,}024$.
To study sampling coverage, we draw actions from
$r_\kappa=\mathcal N(\mu_0,\kappa^2\sigma^2)$ with
$\kappa\in\{1,2,3\}$, keeping the rollout policy and target fixed.
The proposal-corrected weights and estimated target mean are
\begin{equation}
 u_i=\frac{e^{-c(a_i)/T}\pi_{\mu_0}(a_i)^{\lambda/T}}{r_\kappa(a_i)},
 \qquad w_i=\frac{u_i}{\sum_j u_j},
 \qquad \widehat m=\sum_iw_i a_i.
 \label{eq:bimodal_proposal_weights}
\end{equation}
We use 5,000 batches for each $K\in\{64,1{,}024\}$, with paired
standard-normal draws across proposal widths. An update points right
when $\widehat m>\mu_0$. Table~\ref{tab:bimodal_snis} reports the
resulting direction frequencies and target-mean estimates.
For each batch, we use the SNIS weights to estimate FERPO's fitting
loss as a function of the candidate policy mean.
Figure~\ref{fig:bimodal_snis} shows the median and pointwise
5th--95th percentiles of these estimated loss curves across 5,000
batches, together with the exact population loss.
Wider proposals sample the better right well more often and recover
the rightward direction in $98.9\%$ and $99.9\%$ of batches at
$K=1{,}024$.

\begin{figure}[htbp]
 \centering
 \includegraphics[width=\textwidth]{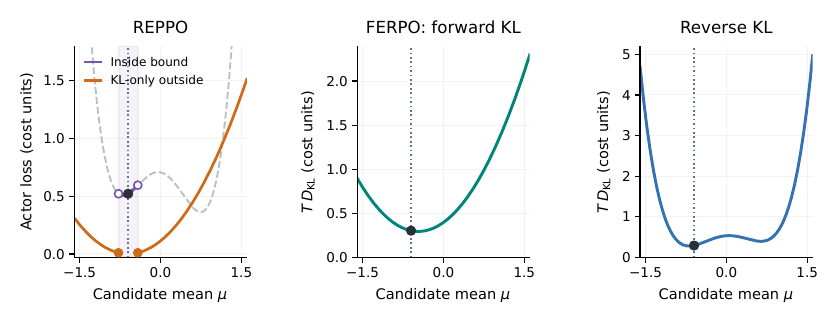}
 \caption{\textbf{Frozen-reference losses at B.} Left: REPPO's
 active expected-cost/entropy branch (purple) and KL-only branch
 (orange); the pale dashed curve extends the inactive cost branch.
 Open and filled endpoints show the loss switch. Middle: forward KL
 from the fixed FERPO target to the candidate Gaussian. Right: reverse
 KL to that same target. Both KL curves are scaled by $T=0.15$;
 this positive scaling preserves directions. Vertical dotted lines
 and black dots mark the reference $\mu_0=-0.6$.}
 \label{fig:bimodal_full_losses}
\end{figure}

\begin{figure}[htbp]
  \centering
  \includegraphics[width=\textwidth]{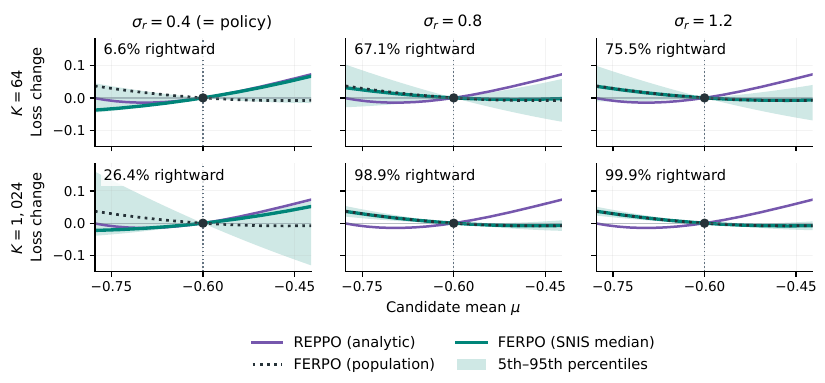}
  \caption{\textbf{Wider proposals recover the population fitting loss.}
  Losses at B with the same reference policy and target throughout. Teal: SNIS
  median and pointwise 5th--95th percentiles across 5,000 batches;
  dotted black: population FERPO; purple: analytic REPPO.
  Percentages give the frequency of rightward SNIS directions.
  Shading shows batch variation, not confidence intervals for the mean.
  Losses are zeroed at the reference, and forward losses are scaled by $T$.}
  \label{fig:bimodal_snis}
\end{figure}

\begin{table}[htbp]
 \centering
 \small
 \begin{tabular}{rrrcr}
\hline
Proposal $\sigma_r$ & $K$ & Median $\widehat m$ & 5th--95th percentiles & Rightward \\
\hline
0.4 & 64 & $-0.916$ & $[-0.963, -0.416]$ & 6.6\% \\
0.8 & 64 & $-0.498$ & $[-0.863, -0.066]$ & 67.1\% \\
1.2 & 64 & $-0.472$ & $[-0.753, -0.091]$ & 75.5\% \\
0.4 & 1,024 & $-0.824$ & $[-0.923, 0.275]$ & 26.4\% \\
0.8 & 1,024 & $-0.468$ & $[-0.565, -0.369]$ & 98.9\% \\
1.2 & 1,024 & $-0.466$ & $[-0.546, -0.381]$ & 99.9\% \\
\hline
\end{tabular}

 \caption{\textbf{SNIS target-mean estimates at B.} The population
 target mean is $-0.465943\ldots$ and rightward movement means
 $\widehat m>-0.6$. Intervals are the empirical 5th--95th percentiles
 across 5,000 batches, not confidence intervals. Policy standard
 deviation remains $0.4$ in every row.}
 \label{tab:bimodal_snis}
\end{table}

\subsection{Computation and memory on the A100}
\label{app:actor_cost}

\subsubsection{Update-cost protocol}
\label{app:actor_cost_protocol}

The update-cost measurements use one NVIDIA A100-SXM4-40GB with AMD
EPYC 7H12 host CPUs, PyTorch 2.9.1+cu128, FP32 parameters, matmul
precision \texttt{medium}, and one PyTorch CPU thread, without AMP
or compilation. We use the official REPPO PyTorch implementation
at commit \texttt{05a5776} and the published large-data MuJoCo
Playground configuration described in Section~\ref{sec:computational_cost}.\footnote{\href{https://github.com/cvoelcker/reppo/tree/05a5776c8efd2f56024d038c982b7d4c500f8221}{Pinned REPPO source};
the \href{https://github.com/cvoelcker/reppo/blob/05a5776c8efd2f56024d038c982b7d4c500f8221/src/torchrl/reppo.py\#L299-L311}{PyTorch}
and \href{https://github.com/cvoelcker/reppo/blob/05a5776c8efd2f56024d038c982b7d4c500f8221/src/jaxrl/reppo.py\#L505-L524}{JAX}
implementations both use 16 samples for the KL estimate.}
Each rollout has 131,072 states, with eight epochs of 64 minibatches
of 2048. The shared baseline has 17 observation and six action
dimensions, a three-layer width-512 actor, and two-layer width-512
critic encoder, Q head, and auxiliary predictor.
FERPO evaluates all $2{,}048K$ candidate state-action pairs in each
minibatch together.

FERPO completes critic training before its actor phase; REPPO
retains its interleaved update schedule. Three warmup minibatches
precede each trial. We use seeds 1--3
and randomize method order within each repetition. CUDA events
measure actor time; synchronized wall time measures the complete
update. Peak allocated PyTorch memory covers the complete update,
including rollout tensors and optimizer state, but excludes
simulator memory and allocator reservation.

\subsubsection{Memory cost}
\label{app:actor_memory}

REPPO backpropagates through Q for one action per state; its 16 KL
samples require policy evaluations only. FERPO instead evaluates
many actions together through Q. Even without retaining a backward
graph, these larger batches require larger temporary tensors.
FERPO's peak allocated memory therefore increases with $K$
(Figure~\ref{fig:actor_cost_main}).

Cached FERPO additionally stores $K$ actions, reference log
probabilities, and Q values for each of $S$ rollout states. With
action dimension $d$ and FP32 storage, this requires
\begin{equation}
  M_{\mathrm{cache}}=4SK(d+2)\ \text{bytes}.
  \label{eq:actor_cache_memory}
\end{equation}
For the measured rollout size at $K=16$, the cache occupies
64 MiB for $d=6$ and 528 MiB for $d=64$.

\subsubsection{Memory across rollout and update phases}
\label{app:actor_memory_cycles}

We separately profile CheetahRun with the same network and rollout
configuration at $K=16$. The matched comparison uses seeds 2 and 3
on one A100 node, with a fresh process per method and seed, two warmup
iterations, and three measured iterations. The simulator and a
32-environment evaluation instance remain resident; initialization and
evaluation are excluded from the measurement windows. NVML samples
total process memory every 5 ms, including runtime and allocator
reservations; these sampled peaks are lower bounds. PyTorch records
allocated and reserved peaks separately for each phase.
We exclude seed 1 from this comparison because uncached FERPO ran
on a second A100 with a different driver.

\begin{table}[htbp]
  \centering
  \small
  \setlength{\tabcolsep}{5pt}
  \begin{tabular}{lrrrrr}
    \hline
    & \multicolumn{2}{c}{Process peak}
    & \multicolumn{2}{c}{PyTorch allocated peak} & PyTorch reserved \\
    Method & Rollout & Update & Rollout & Update & Rollout \\
    \hline
    REPPO & 1,438 & 1,408 & 601 & 422 & 614 \\
    FERPO & 1,760 & 1,728 & 600 & 658 & 936 \\
    Cached FERPO & 1,806 & 1,774 & 600 & 723 & 982 \\
    \hline
  \end{tabular}
  \caption{\textbf{Memory across rollout and update phases (MiB).}
  Peaks are identical across the two matched seeds. Reserved memory
  includes active tensors.}
  \label{tab:actor_memory_cycles}
\end{table}

All three methods reach their process peak during rollouts
(Table~\ref{tab:actor_memory_cycles}). Active PyTorch rollout allocations
are nearly equal, while larger pools remain reserved after FERPO's
updates. These process footprints therefore reflect allocator retention
as well as live tensors, rather than minimum GPU capacity requirements.

\clearpage
\subsection{Learning ablations and benchmarks}
\label{app:learning_results}

\subsubsection{Ablation details}
\label{app:learning_ablation_details}

The candidate-count ablation uses per-seed G1JoystickFlatTerrain
evaluation logs for $K=2,4,8,16$. All four settings use actor/critic and temperature
learning rates of $0.0003$, entropy target factor $-0.5$, and normalized
ESS target $0.85$. The replacements correct the historical $K=2$
entropy target of $-1.0$ and $K=16$ temperature learning rate of $0.003$;
those superseded observations are excluded from the plots.
We report arithmetic means with one-SEM error bars and bands.
The temperature comparison reuses the corrected $K=16$ control.
These ablation cohorts remain separate from the G1/T1 benchmark
cohort below, which uses normalized ESS target $0.65$.

The candidate-count plot in Figure~\ref{fig:learning_k} (left) uses the
final common evaluation at 50,069,504 environment steps, rather
than selecting a best checkpoint or averaging a trailing window.
There are eight seeds at this endpoint for $K=2,16$, ten for $K=4,8$,
and eight for the additional $\tau=0$ marker at $K=16$.
The no-entropy learning curve in Figure~\ref{fig:learning_temperature}
uses the same eight completed seeds at every evaluation.
The corrected adaptive-temperature control uses seeds 1--8. Both temperature
curves and all candidate-count endpoints show mean $\pm$ one SEM.
All eight no-entropy runs reach 50,069,504 steps, with final mean
$30.32\pm0.27$. They disable entropy regularization and its temperature
update. The entropy target changed from $-1.0$ in the historical runs
to $-0.5$ in the new runs, but is inactive when $\tau=0$.
For all G1 ablations, including cached-Q and critic-width comparisons,
we compute $\mathrm{SEM}=s/\sqrt{n}$, where $s$ is the sample standard
deviation (with $n-1$ in its variance denominator) and $n$ is the
number of contributing seeds. These are standard errors of the mean,
not 95\% confidence intervals. Benchmark curves use IQM and bootstrap
intervals as specified below.

\begin{figure}[htbp]
  \centering
  \begin{minipage}[c]{0.47\textwidth}
    \includegraphics[width=\linewidth]{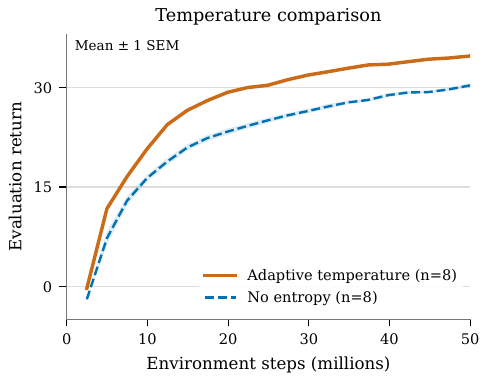}
  \end{minipage}\hfill
  \begin{minipage}[c]{0.49\textwidth}
    \caption{\textbf{Temperature learning curves.}
    Adaptive temperature versus no entropy regularization at $K=16$.
    Both curves show seed means with one-SEM bands.
    Both curves use eight completed seeds at every evaluation;
    the adaptive control uses corrected temperature learning rate $0.0003$.
    Final returns are $34.81\pm0.24$ with adaptive temperature and
    $30.32\pm0.27$ without entropy regularization.
    The final evaluations
    also appear in Figure~\ref{fig:learning_k} (left).}
    \label{fig:learning_temperature}
  \end{minipage}
\end{figure}
\clearpage
\begingroup
\subsubsection{{Critic-width protocol}}
\label{app:critic_width}

{The width comparison uses the uploaded G1 preset
on G1JoystickFlatTerrain with learning rate $3\times10^{-4}$, target
entropy factor $-0.5$, eight critic epochs, HL-Gauss loss, and auxiliary
weight one. Both methods use the full rollout. FERPO uses $K=8$;
REPPO retains its pathwise actor update and KL bound 0.1.
We vary hidden width in the critic encoder, $Q$ head, and auxiliary
predictor together; depth and the 512-wide actor are fixed.
Each endpoint is the final deterministic evaluation at 50,069,504
rollout interactions. Seeds 1--2 contribute at each width except
REPPO width 1,024, for which only seed 1 is available.
These are dedicated critic-width runs, separate from the published
REPPO benchmark results. We have not measured action-gradient
approximation error in these runs.}

At 50.07M interactions, reducing critic width from 124 to 62 lowers
FERPO's mean return from 23.04 to 6.29 and REPPO's from 22.12 to 4.19.
Both methods therefore rely on sufficient critic capacity.

\begin{figure}[htbp]
  \centering
  \includegraphics[width=0.85\textwidth]{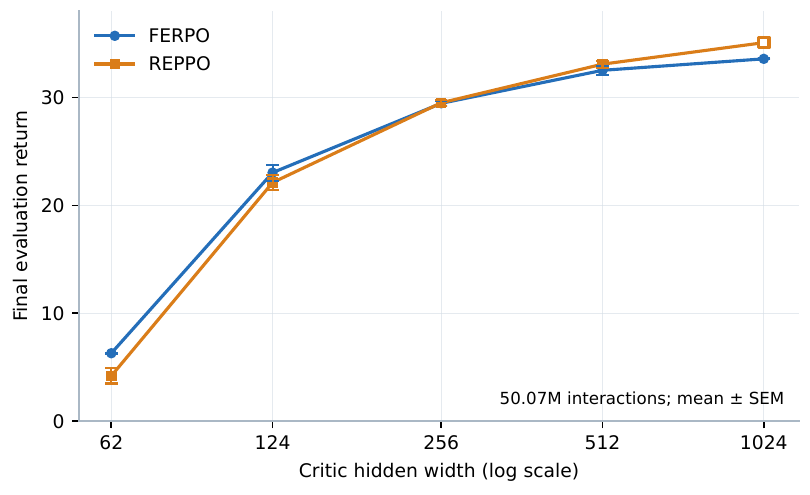}
  \caption{{\textbf{Final return versus critic width on G1.}
  Points show seed means and error bars show one SEM.
  The open square marks the single-seed REPPO reference at width
  1,024, for which no error bar is reported. Both methods lose
  substantial return at width 62; FERPO reaches 6.29 and REPPO 4.19.
  These endpoints do not identify the source of critic approximation error.}}
  \label{fig:critic_width_ablation}
\end{figure}
\endgroup

\clearpage
\subsubsection{Cached-Q timing protocol and elapsed-time curves}
\label{app:cached_q_timing}

We compare interleaved critic/actor updates with a staged variant
that first fits the critic, then freezes it and caches $K=8$ actions,
their $Q$ values, and rollout log probabilities. Actor training
recomputes SNIS weights as the multipliers change. Both methods use
eight full-rollout critic epochs, 1,024-wide critics, learning rate
$3\times10^{-4}$, and target entropy factor $-0.5$. The five-seed
rollout curves are unsmoothed. Final mean returns at 50,069,504 steps
are $34.03\pm0.21$ uncached and $33.58\pm0.40$ cached (mean $\pm$ SEM).

The timing annotation in Figure~\ref{fig:cached_q_ablation}
(middle) and elapsed-time curves in Figure~\ref{fig:cached_q_time}
use seeds 3--5 for each method. These six runs shared the
same Vega allocation and node, with one NVIDIA A100-SXM4-40GB and
a separate CPU set per run. Seeds 1--2 remain in the five-seed
rollout comparison in Figure~\ref{fig:cached_q_ablation} (middle),
but are excluded from timing because
their cached and uncached runs used different clusters.
The per-run clock starts before environment construction and includes
simulation, learning, evaluation, and checkpointing; scheduler waiting
time is excluded. Mean completion times at 50,069,504 rollout steps
are 55.7 minutes cached and 68.4 minutes uncached. The time reduction
is $1-T_{\mathrm{cached}}/T_{\mathrm{uncached}}=18.6\%$; the
throughput increase is
$T_{\mathrm{uncached}}/T_{\mathrm{cached}}-1=22.9\%$.
These are measured runtimes for this allocation.

We use recorded evaluation timestamps for the time curves. Within
each method, we linearly interpolate each seed's return to common
elapsed times, then compute the mean and SEM over all three
seeds. The bold mean curve ends when the first run in that method
finishes; faint individual curves continue to their own final
evaluation. We do not extrapolate or carry completed returns forward.
The arrow separately marks the mean completion times from the run
summaries, including the final checkpoint. The five-seed curves in
Figure~\ref{fig:cached_q_ablation} (middle) use recorded rollout-step coordinates
directly, without interpolation.

\begin{figure}[htbp]
  \centering
  \begin{minipage}[c]{0.47\textwidth}
    \includegraphics[width=\linewidth]{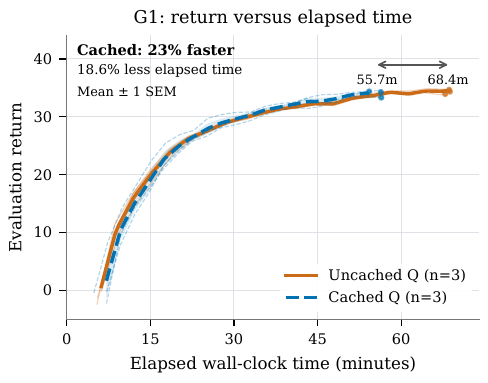}
  \end{minipage}\hfill
  \begin{minipage}[c]{0.49\textwidth}
    \caption{\textbf{Cached versus uncached Q against elapsed time.}
    G1 deterministic evaluation return for seeds 3--5 per method on the
    same Vega A100 node. Both use $K=8$ and eight critic epochs on the
    full rollout. Bold curves and one-SEM bands use linear interpolation
    at common times; faint curves show individual runs. The arrow marks
    mean completion times at 50,069,504 interactions: 55.7 versus
    68.4 minutes. At 50 minutes, interpolated mean returns are 33.55
    cached and 32.86 uncached. Figure~\ref{fig:cached_q_ablation} (middle)
    shows all five seeds against environment interactions.}
    \label{fig:cached_q_time}
  \end{minipage}
\end{figure}

\clearpage
\subsubsection{State-value critic with candidate simulation}
\label{app:state_value_ablation}

{We compare standard uncached FERPO with its cached
$V$ extension on AcrobotSwingupSparse.}
{For a fixed fitted critic, the ideal candidate score is
$Q_V(s,a)=r(s,a)+\gamma\E_{s'\sim P(\cdot\mid s,a)}[V_\psi(s')]$.
We approximate the transition expectation with one simulated successor
per candidate, giving
$\widehat Q_V(s,a)=r(s,a)+\gamma V_\psi(s')$.
This incorporates the observed immediate reward directly while using
a one-sample Monte Carlo estimate of the future value.}
{Scores are computed after fitting $V_\psi$ and held fixed throughout
actor training. The bootstrap is zero at true terminals; time limits
retain the value of the final observation.}

{Both methods use $K=16$, 512-wide critics, eight
critic epochs, HL-Gauss loss, auxiliary weight one, learning rate
$10^{-3}$, and target entropy factor $-1$. Across 12 matched seed
identifiers, cached $V$ reaches $324.79\pm75.28$ at 59.90 million
rollout steps, compared with $99.46\pm110.75$ for uncached $Q$
(mean $\pm$ one SD; Figure~\ref{fig:value_acrobot_ablation}, right).
The deterministic evaluation curves are unsmoothed. Each rollout
step requires $K$ additional candidate transitions, so cached $V$
uses $17$ times as many simulator transitions at this rollout budget.}

{\textbf{Efficient candidate simulation.}
We combine Warp graph execution with parallel simulation of all
$K$ candidates across environments. Candidates are sampled and
simulated during rollout, branching from each state without altering
the rollout trajectory. Their rewards and next observations remain
on the GPU until the critic has been fitted, then serve to construct
the score cache. This avoids repeated transfers of simulator states
between CPU and GPU. Actor epochs reuse the candidates and scores
while updating their SNIS weights. Since observation normalization
can change during rollout, we retain the sampling densities and
correct these weights for any difference from the frozen rollout
reference.}

{\textbf{Timing protocol.}
{We benchmark both methods on the same NVIDIA A100-SXM4-40GB with
Warp graphs enabled.} Both use 1,024 environments, 128 rollout steps
per iteration, actor and critic widths of 512, eight actor and critic
epochs, 64 minibatches, and one PyTorch CPU thread. We run three
repetitions in separate processes, randomizing method order and
averaging five iterations after two warmup iterations. GPU
synchronization at iteration boundaries ensures that timings include
all simulation and learning work; startup, compilation warmup,
evaluation, external logging, and checkpointing are excluded.
Cached $V$ takes 8.508 seconds per iteration versus 7.846 seconds
for uncached $Q$, an 8.4\% overhead.}

\subsubsection{{Additional planned ablation}}
\label{sec:planned_ablations}

{The following additional G1 comparison remains
planned. It will use matched seeds, interaction budgets, and evaluation
protocols, keeping other settings fixed.}

\textbf{Forward KL instead of ESS regularization.}
We will replace the log-ESS constraint with
$\E_{s\sim d_{\mathrm{roll}}}[\KL(\pi_{\mathrm{roll}}(\cdot\mid s)\|\pi_\theta(\cdot\mid s))]\leq0.1$,
matching REPPO's bound. The action expectation in this forward KL is
under the old rollout policy. We will adapt $\lambda$ using the KL
constraint gap and compare return, achieved KL, and normalized ESS
against the ESS-controlled baseline.

\clearpage
\begingroup
\subsubsection{{ESS versus policy-KL control}}
\label{app:ess_kl_ablation}

{We compare log-ESS control (Eq.~\ref{eq:log_ess_controller})
at $e_{\mathrm{target}}\in\{0.40,0.50,0.65,0.85\}$ with control of the
rollout-state mean $\KL(\pi_{\mathrm{roll}}\|\pi_{\mathrm{curr}})$ at
$\epsilon\in\{0.03,0.10,0.20,0.30\}$ nats. Both adapt $\lambda$
with the same forward-KL actor objective. ESS controls SNIS weight
concentration; policy KL measures policy displacement.}

{We use uncached FERPO ($K=16$) on HopperHop and
AcrobotSwingupSparse with 12 seeds per setting. All runs use width-512
critics, eight epochs, 64 minibatches, HL-Gauss loss, auxiliary weight
one, actor/critic and temperature learning rates $10^{-3}$, and entropy
target factor $-1$. Rollouts comprise 1,024 environments and 128 steps.
Final returns average the first episode from each of 128 deterministic
evaluation environments, with frozen normalizers and fixed initial
states within each run.}

{Both controllers support learning at suitable settings,
with task-dependent preferred regularization strengths and no consistent
return advantage for either controller (Figure~\ref{fig:ess_kl_end_return}).}

\begin{figure}[htbp]
  \centering
  \includegraphics[width=\textwidth]{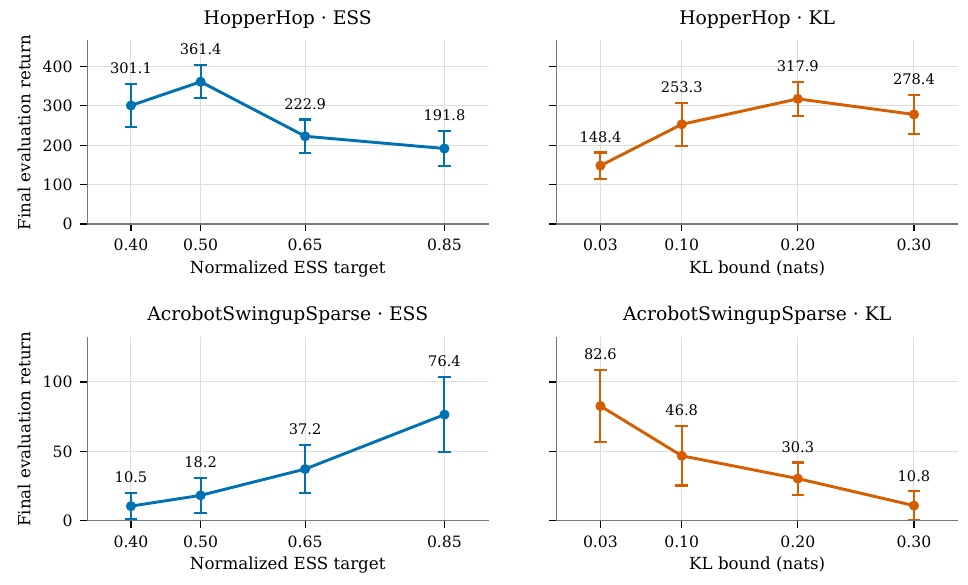}
  \caption{{\textbf{Final return versus controller target.}
  HopperHop (top) and AcrobotSwingupSparse (bottom), with ESS targets
  (left) and policy-KL bounds (right). Points show mean final return
  at 50.07M interactions with one SEM over 12 seeds.}}
  \label{fig:ess_kl_end_return}
\end{figure}

\clearpage
{Direct ESS control generally tracks its target, whereas
fixed policy-KL bounds produce ESS that varies across tasks and training
(Figure~\ref{fig:ess_kl_tracking}). At comparable mean achieved ESS,
direct control shows less temporal variation, although the aggressive
ESS target $0.40$ remains variable on Acrobot. These observations motivate
ESS as a direct, interpretable control of sampling-weight concentration.}

\begin{figure}[htbp]
  \centering
  \includegraphics[width=\textwidth]{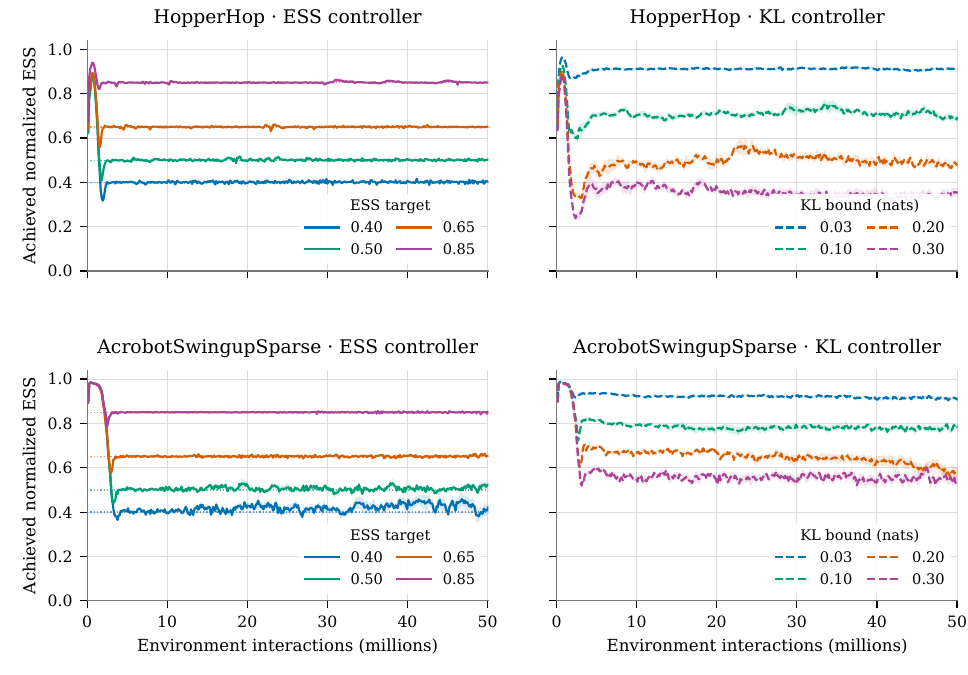}
  \caption{{\textbf{Achieved ESS under both controllers.}
  Per-iteration geometric mean of normalized SNIS ESS under direct ESS
  control (left) and policy-KL control (right). Curves show unsmoothed
  means with one-SEM bands across 12 seeds. Dotted lines mark ESS targets;
  right-panel legends give policy-KL bounds in nats.}}
  \label{fig:ess_kl_tracking}
\end{figure}
\endgroup

\clearpage
\subsubsection{Data and comparison protocol}
\label{app:learning_protocol}

The benchmark comprises 22 DMC, four G1/T1, and eight ManiSkill tasks.
FERPO uses training-log return evaluations for 19 DMC tasks and
success rates for ManiSkill. The three Walker tasks use separate
policy-evaluation CSVs, whose task metadata and reconstructed means
and counts match the original benchmark summaries. Walker has
ten evaluated seeds per task beyond 50M steps, replacing the older
partial training-log evaluations. For Walker, we display the available
positive checkpoint evaluations from 3,014,656 through 50,987,008
interactions; seeds 11--12 have no exported evaluation curves and are
not counted. The other 19 DMC tasks retain their own recorded evaluation
schedules. These existing evaluations use deterministic actions and
report 1,000-step episode lengths, but sampling settings differ:
the older training evaluator uses 1,024 environments, the newer one
uses 128, and the Walker checkpoint evaluator uses 128 episodes.
Initial evaluation states are fixed within a run rather than refreshed
at every checkpoint. Thus these are existing evaluation exports, not
a reevaluation under one identical sampling protocol.
{The G1/T1 benchmark uses the completed $K=16$, normalized
ESS target $0.65$ runs, with seeds 1--8 for each task. Actor/critic and temperature
learning rates are $0.0003$, and the entropy target factor is $-0.5$.
We use all 21 recorded deterministic training evaluations through
50,069,504 interactions, averaging the first episode from each of 128
parallel environments.}
The ManiSkill source label ``FEPPO architecture'' denotes FERPO.
We retain all available evaluated seeds within their observed ranges,
so counts may vary across checkpoints; training-only runs contribute
no observations.
Each DMC panel labels the number of contributing seeds for every method;
a range indicates changing availability along the displayed curve.
We retain the common 50M budget for numerical method comparisons and
do not extrapolate baselines to the last Walker checkpoint.

\textbf{Baseline sources and configuration selection.}
All baseline results are drawn from the experiments accompanying the
REPPO paper \citep{voelcker2026relative}.
REPPO uses the official per-task trial CSVs at commit
\texttt{05a5776}.\footnote{\href{https://github.com/cvoelcker/reppo/tree/05a5776c8efd2f56024d038c982b7d4c500f8221/results}{Official REPPO result CSVs at the pinned commit}.}
The local alternative REPPO configurations and baselines embedded
in FERPO's summary files are excluded. Following Figure~15 of
\citet{voelcker2026relative}, we distinguish PPO and PPO (Brax)
on DMC. We use the available CSV exports at their recorded step
coordinates; these are not verified as the exact trials behind that
published figure. We neither rescale their training horizons nor
combine the two PPO families. For the remaining methods,
we compare configurations separately by method family and suite
(DMC, G1/T1, or ManiSkill). A candidate must cover every task available for its
method in that group and span the 50-million-step selection budget.
We linearly interpolate each task's seed-mean return to that budget.
For DMC, we min--max normalize each task's score across eligible
configurations of the same method, assign 0.5 when all candidates
tie, and average over tasks. For ManiSkill, we average success
rates directly. The highest-scoring configuration is then fixed
across all tasks in its group, with alphabetical tie-breaking.
This is retrospective selection on the displayed runs, rather than
selection on an independent validation set.

\begin{table}[htbp]
  \centering
  \small
  \begin{tabular}{lll}
    \hline
    Group & Method & Selected configuration or source \\
    \hline
    DMC & REPPO & Official published trials \\
    DMC & PPO & \texttt{ppo\_4} \\
    DMC & PPO (Brax) & \texttt{ppo\_brax} \\
    DMC & FastTD3 & \texttt{fasttd3\_small} \\
    G1/T1 & REPPO & Official published trials \\
    ManiSkill & REPPO & Official published trials \\
    ManiSkill & PPO & \texttt{maniskill\_ppo\_short} \\
    \hline
  \end{tabular}
  \caption{Baseline sources, fixed across tasks within each group.
  The archive contains no PPO or FastTD3 G1/T1 curves, and
  no FastTD3 ManiSkill curves.}
  \label{tab:learning_baseline_selection}
\end{table}

\textbf{Statistics and display.}
All benchmark methods show IQM over the available independent seeds.
We compute IQM by integrating the middle 50\% of empirical
probability mass, including fractional observations at quartile
boundaries. FERPO runs are aligned on the union of their recorded
evaluation steps, using linear interpolation only between observed
evaluations of the same run. We do not extrapolate or fill in
missing run endpoints. At each checkpoint, we resample the available
seeds with replacement 10,000 times and report the 2.5th and 97.5th
percentiles, using fixed random seeds. Checkpoints with the same
available seeds share the same trajectory resampling; a changed
seed cohort is resampled separately. At least two observed seeds
are required. These are pointwise 95\% intervals conditional on
the observed seeds and selected configuration; they do not correct
for missing runs or configuration search. Baselines have complete
trial grids and use whole-trajectory resampling. FERPO uses exactly
the same actor and critic architectures as REPPO for each benchmark
suite.

REPPO evaluations labeled step zero are placed at 2,490,368
environment steps; nonzero REPPO coordinates and all other
methods' coordinates remain as recorded. Figures display up to
50 million steps without smoothing or extrapolation, except that the
DMC task panels extend to the recorded Walker checkpoint at 50,987,008.
FERPO's
ManiSkill curves end at their last recorded evaluation,
49,971,200 steps. Bootstrap success-rate intervals lie within
$[0,1]$ without clipping.

\textbf{Aggregation across environments and seeds.}
Figure~\ref{fig:learning_aggregate} pools scores within each suite.
For $T$ environments and $n_t$ available seeds in environment $t$,
each seed receives weight $1/(T n_t)$ before IQM trimming.
Scores retain their original units without cross-task normalization.
We use 21 shared checkpoints, requiring at least two seeds per
environment and method. The common ranges are 3.01--50M steps
for DMC, 2.49--50M for G1/T1, and 2.49--49.97M for ManiSkill.
The 10,000-resample seed bootstrap is stratified by environment,
keeping the environment set and its equal weights fixed.
In particular, each of the three Walker tasks contributes $1/22$
of the DMC probability mass before trimming, split among its ten
evaluated seeds: $1/220$ per seed. A task with twenty evaluated
seeds still contributes $1/22$, with $1/440$ per seed. We do not
duplicate Walker seeds or discard seeds from other tasks to equalize
counts. At 50M, FERPO contributes 372 task--seed observations across
the 22 tasks (7--20 per task). The aggregate uses the same positive
FERPO evaluations as the task panels; its first checkpoint is
3,014,656 steps, and it ends at the common 50M budget without
extrapolating baselines to Walker's last checkpoint.

\clearpage
\subsubsection{Pairwise comparisons of final benchmark performance}
\label{app:endpoint_comparisons}

The endpoint comparison uses the same tasks, available seeds, configuration
choices, and weighting as Figure~\ref{fig:learning_aggregate}. All methods
within a suite are compared at its final common budget: 50,000,000 steps
for DMC and G1/T1, and 49,971,200 steps for ManiSkill. Each seed contributes
its linearly interpolated evaluation score only when that budget lies
within its observed range. Individual IQMs and their marginal 95\%
intervals in Section~\ref{sec:learning_benchmarks} are the existing
10,000-resample Figure~\ref{fig:learning_aggregate} estimates.

For the pairwise analysis, we independently resample seeds with replacement
within each fixed environment and method, preserving each environment's
seed count and total weight. Numerical seed labels do not establish
paired runs across methods, so we do not use a paired-method bootstrap.
For each of 100,000 resamples, we compute the suite IQM for each method
and subtract the two IQMs of interest. We retain the environment set
rather than resampling environments. The procedure uses deterministic
method-specific random seeds and reproduces the plotted endpoint IQMs
before computing differences.

There are fourteen comparisons: ten on DMC, one on G1/T1, and three on
ManiSkill. We use quantiles $0.025/14$ and $1-0.025/14$ of each bootstrap
difference distribution. These are approximately 99.643\% per-comparison intervals,
corresponding to a Bonferroni correction that targets at least 95\%
coverage for the family of fourteen comparisons, subject to the bootstrap
approximation. A difference interval entirely above zero supports higher
IQM for method A; an interval entirely below zero supports method B.
An interval spanning zero is inconclusive and does not establish equal
performance. The intervals condition on the observed seeds and selected
configurations; they do not account for configuration search or
informative missingness.

\begin{table}[htbp]
  \centering
  \small
  \setlength{\tabcolsep}{4pt}
  \begin{tabular}{llrr}
\hline
Suite & Method A $-$ method B & IQM difference & Adjusted interval \\
\hline
MuJoCo DMC & FERPO $-$ REPPO & $6.67$ & $[-6.89,20.65]$ \\
MuJoCo DMC & FERPO $-$ FastTD3 & $7.10$ & $[-11.82,30.90]$ \\
MuJoCo DMC & FERPO $-$ PPO & $587.89$ & $[570.19,604.70]$ \\
MuJoCo DMC & FERPO $-$ PPO (Brax) & $354.86$ & $[307.06,398.14]$ \\
MuJoCo DMC & REPPO $-$ FastTD3 & $0.44$ & $[-18.30,23.87]$ \\
MuJoCo DMC & REPPO $-$ PPO & $581.22$ & $[563.44,597.76]$ \\
MuJoCo DMC & REPPO $-$ PPO (Brax) & $348.19$ & $[300.83,391.24]$ \\
MuJoCo DMC & FastTD3 $-$ PPO & $580.78$ & $[554.61,602.24]$ \\
MuJoCo DMC & FastTD3 $-$ PPO (Brax) & $347.76$ & $[297.07,393.03]$ \\
MuJoCo DMC & PPO $-$ PPO (Brax) & $-233.02$ & $[-281.80,-188.12]$ \\
{MuJoCo G1/T1} & {FERPO $-$ REPPO} & {$-1.41$} & {$[-1.94,-0.94]$} \\
ManiSkill (pp) & FERPO $-$ REPPO & $0.15$ & $[-3.75,4.31]$ \\
ManiSkill (pp) & FERPO $-$ PPO & $11.07$ & $[0.67,27.83]$ \\
ManiSkill (pp) & REPPO $-$ PPO & $10.91$ & $[0.40,27.78]$ \\
\hline
\end{tabular}

  \caption{\textbf{Pairwise differences in final suite IQM.}
  Intervals use 100,000 seed bootstrap resamples and a Bonferroni
  correction across all fourteen comparisons (95\% family-wise target).
  MuJoCo differences are in return units; ManiSkill differences are
  in percentage points (pp). Positive values favor method A.
  {The G1/T1 comparison uses eight FERPO seeds per task.}}
  \label{tab:endpoint_comparisons}
\end{table}

\clearpage
\subsubsection{DMC tasks}

\begin{figure}[htbp]
  \centering
  \includegraphics[width=\textwidth]{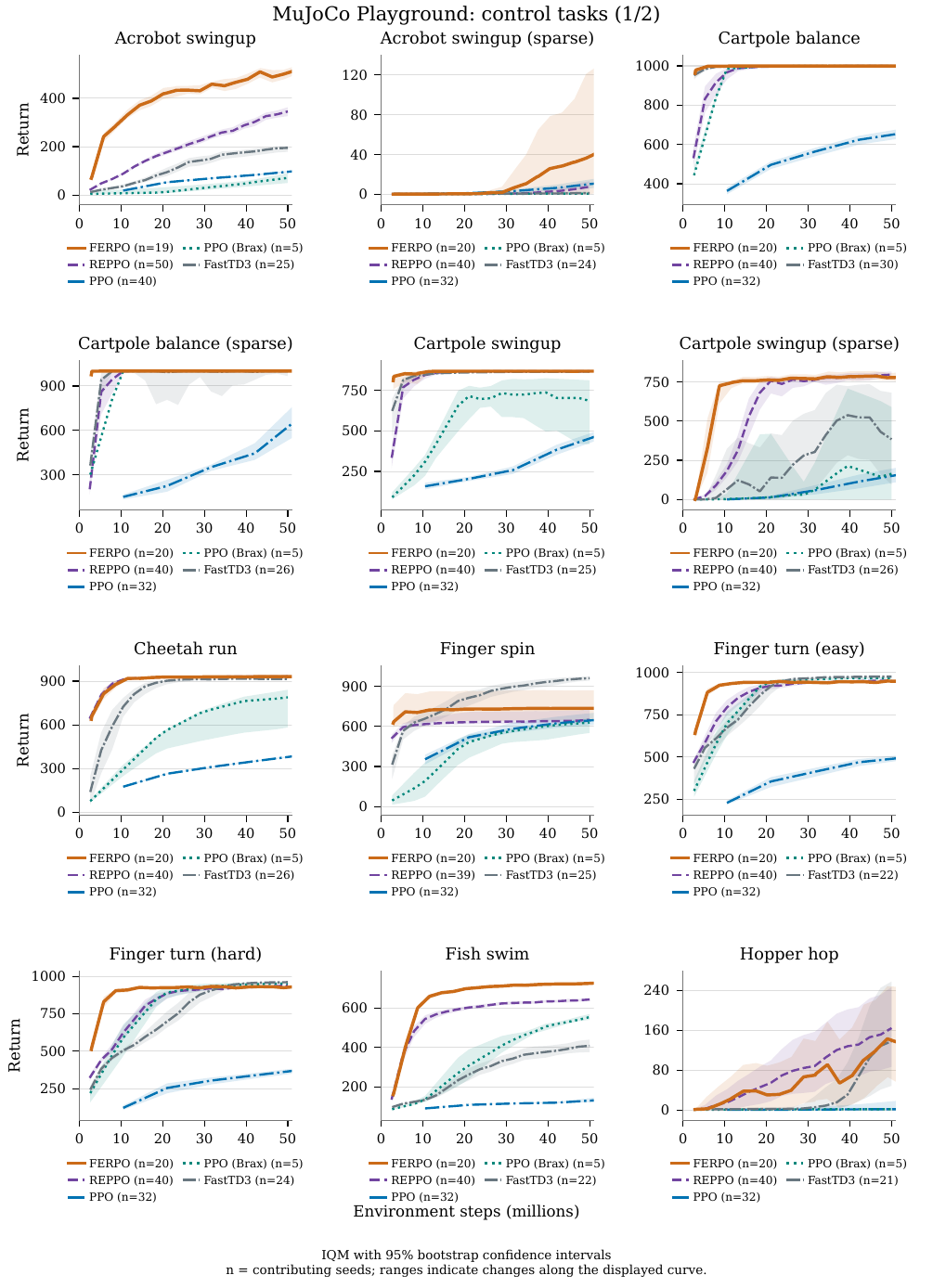}
  \caption{\textbf{DMC learning curves (1/2).} The first 12 of
  22 tasks. All methods show IQM with pointwise 95\% bootstrap
  confidence intervals. Legends report contributing seeds $n$;
  ranges indicate changes across the displayed evaluations.
  Curves retain their source evaluation schedules.
  Baseline configurations are fixed across all DMC tasks as
  specified in Table~\ref{tab:learning_baseline_selection}.}
  \label{fig:learning_control_1}
\end{figure}
\clearpage
\begin{figure}[p]
  \centering
  \includegraphics[width=\textwidth]{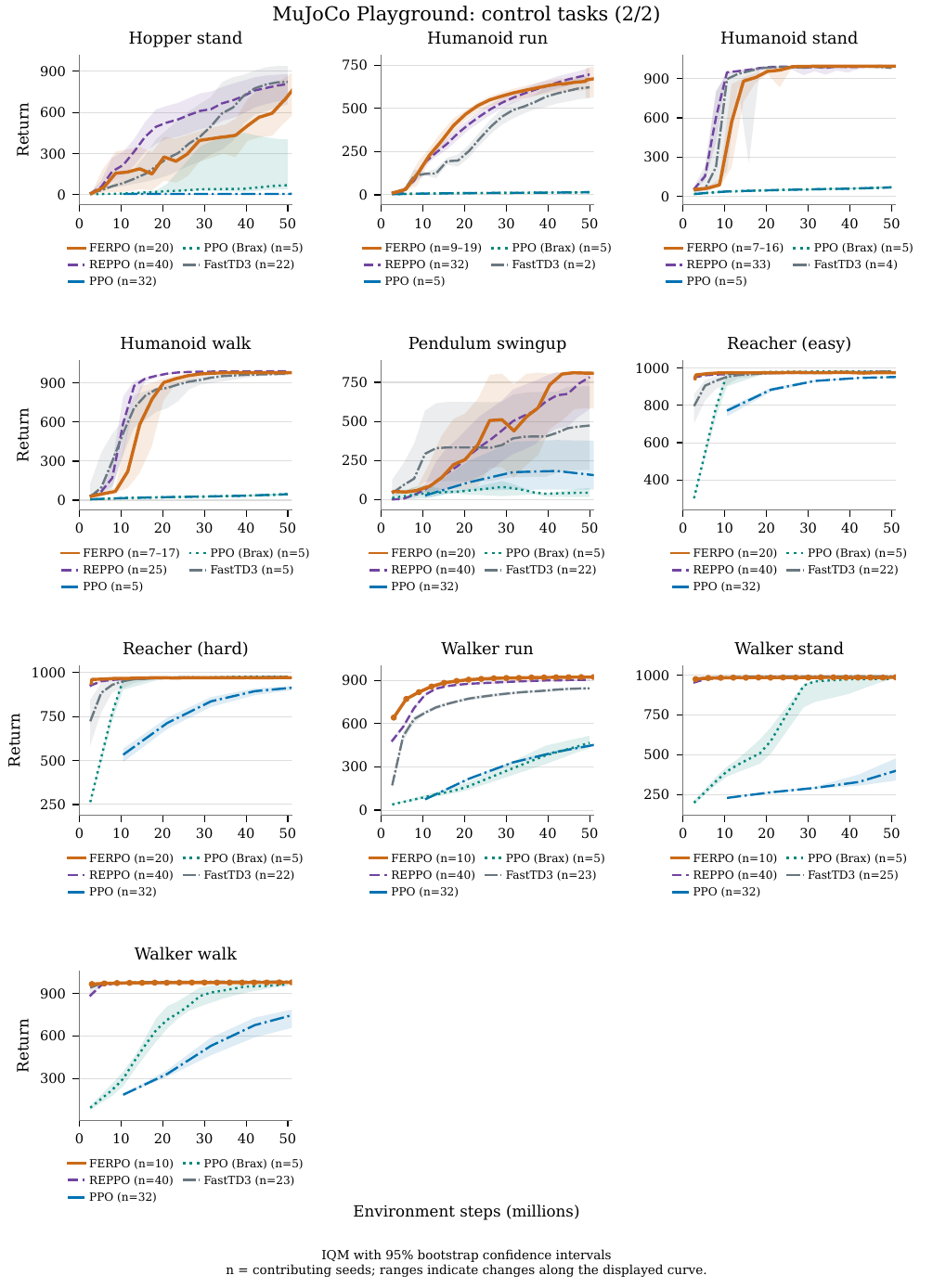}
  \caption{\textbf{DMC learning curves (2/2).} The remaining
  ten tasks, under the same protocol as
  Figure~\ref{fig:learning_control_1}. All methods show IQM with
  pointwise 95\% bootstrap confidence intervals. Walker markers denote
  the available checkpoint evaluations from 3.01M to 50.99M steps,
  with ten seeds per task. Legends report contributing seed counts.}
  \label{fig:learning_control_2}
\end{figure}

\clearpage
\subsubsection{G1/T1 locomotion and ManiSkill manipulation}

{At 50M interactions, FERPO's IQM returns are 35.56 and
30.38 on G1 flat and rough terrain, compared with REPPO's 37.60 and
32.89. On T1 flat and rough terrain, FERPO reaches 37.05 and 30.68,
compared with REPPO's 36.48 and 30.73. Thus FERPO has a slightly higher
point estimate on T1 flat and a similar endpoint on T1 rough, while
REPPO retains a larger lead on both G1 tasks.}

\begin{figure}[htbp]
  \centering
  \includegraphics[width=0.85\textwidth]{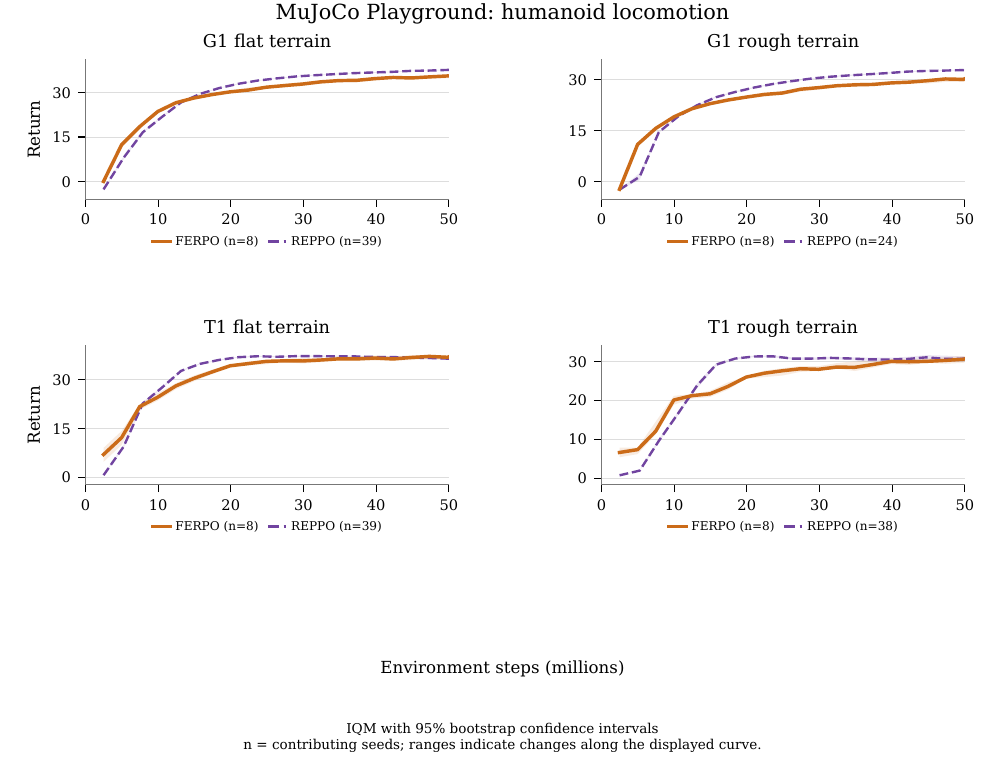}
  \caption{{\textbf{All four G1/T1 locomotion tasks.} FERPO and
  official REPPO trials show IQM with pointwise 95\% bootstrap
  confidence intervals. FERPO uses $K=16$, normalized ESS target $0.65$,
  and eight completed seeds per task; legends report each method's seed
  count. Results use the recorded evaluation protocols
  described in Appendix~\ref{app:learning_protocol}.}}
  \label{fig:learning_locomotion}
\end{figure}

\begin{figure}[htbp]
  \centering
  \includegraphics[width=\textwidth]{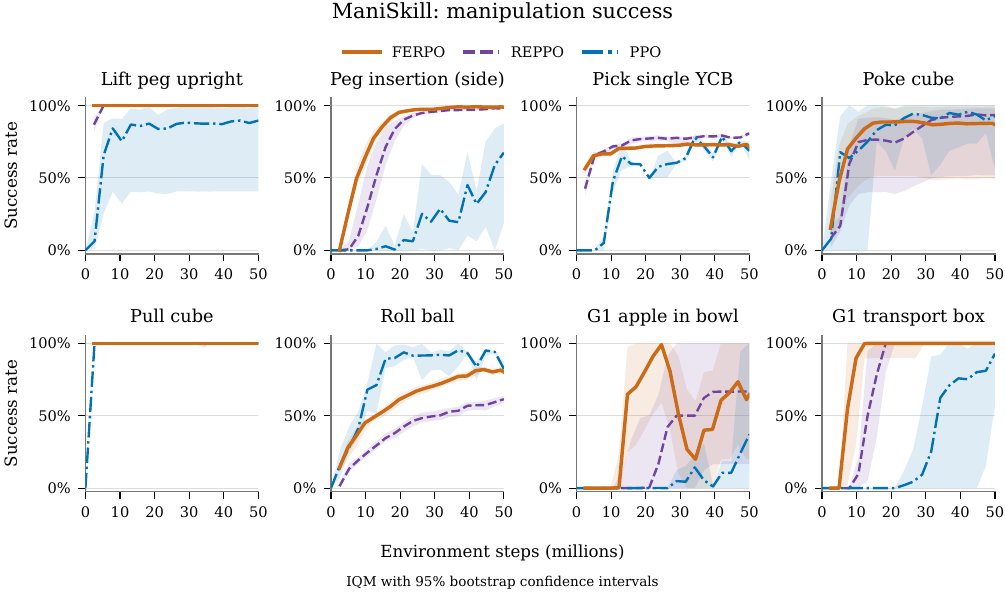}
  \caption{\textbf{All eight ManiSkill tasks.} All methods show
  success-rate IQM with pointwise 95\% bootstrap confidence
  intervals. FERPO uses 20 seeds per task except G1 transport box
  (ten). PPO uses one configuration across all eight tasks;
  REPPO uses the official published trial CSVs.}
  \label{fig:learning_maniskill}
\end{figure}

\clearpage
\subsection{FERPO pseudocode}
\label{app:ferpo_algorithm}
\begingroup

Algorithm~\ref{alg:ferpo} shows the interleaved critic and actor updates
used in the main learning setup. Critic targets are computed once per
rollout; candidate actions and their Q values are refreshed for each
actor update. We write $\langle\cdot\rangle_{\mathcal B}$ for a
minibatch average and $\operatorname{sg}$ for stop-gradient.
Bootstrap states use the final observation before any environment
reset. The frozen parameters
$\bar\phi$ are a snapshot for target construction.

\begin{algorithm}[H]
  \caption{FERPO with interleaved critic and actor updates}
  \label{alg:ferpo}
  \small
  \raggedright
  \begin{algorithmic}[1]
    \REQUIRE Environment, rollout length $T$, discount $\gamma$, trace parameter
      $\lambda_{\mathrm{TD}}$, auxiliary weight $c_{\mathrm{aux}}$,
      candidate count $K\geq2$, epochs $E$, minibatch size,
      entropy target $\mathcal H_{\mathrm{target}}$, ESS target
      $e_{\mathrm{target}}\in(1/K,1)$, critic settings from
      Section~\ref{sec:critic_learning}.
    \STATE Initialize policy $\pi_\theta$, critic $Q_\phi$, log multipliers
      $\eta,\xi$, and their optimizers.
    \WHILE{the environment interaction budget remains}
      \STATE Freeze $\pi_{\mathrm{roll}}\leftarrow\pi_\theta$ and
        $\bar\phi\leftarrow\phi$, $\tau_{\mathrm{roll}}\leftarrow e^\eta$.
      \STATE Collect rollout transitions $\mathcal D$ using
        $a_t\sim\pi_{\mathrm{roll}}(\cdot\mid s_t)$,
        including terminal and truncation flags.
      \STATE For each transition $t$, set $m_t\leftarrow0$ at a true terminal
        and $m_t\leftarrow1$ otherwise, including truncations.
      \STATE Set $c_t\leftarrow1$ if $t<T-1$ and transition $t$ is neither
        terminal nor truncated; otherwise set $c_t\leftarrow0$.
      \STATE At each nonterminal successor, use $a_t^+\leftarrow a_{t+1}$
        if $c_t=1$; otherwise draw
        $a_t^+\sim\pi_{\mathrm{roll}}(\cdot\mid s_{t+1})$.
      \STATE Store $v_t^+\leftarrow Q_{\bar\phi}(s_{t+1},a_t^+)$ and
        $b_t^+\leftarrow\log\pi_{\mathrm{roll}}(a_t^+\mid s_{t+1})$
        without gradients; set both to zero at true terminals.
      \STATE Set $G_T\leftarrow0$ at the rollout end; $c_{T-1}=0$
        ensures bootstrapping uses $v_{T-1}^+$.
      \FOR{$t=T-1,\ldots,0$ in each environment}
        \STATE $G_t\leftarrow r_t+\gamma m_t
          [(1-\lambda_{\mathrm{TD}}c_t)v_t^+
          +\lambda_{\mathrm{TD}}c_tG_{t+1}
          -\tau_{\mathrm{roll}}b_t^+]$.
      \ENDFOR
      \STATE Store detached $G_t$, next-feature targets from the frozen
        encoder $\bar\phi$, and reward targets $r_t$ in $\mathcal D$.
      \FOR{epoch $=1,\ldots,E$}
        \FOR{each shuffled minibatch $\mathcal B\subset\mathcal D$}
          \STATE $\mathcal L_Q\leftarrow
            \bigl\langle\mathbf1_{\{\text{not truncated}\}}
            [\mathcal L_{\mathrm{HL\text{-}Gauss}}(\phi;s_t,a_t,G_t)
            +c_{\mathrm{aux}}\mathcal L_{\mathrm{aux}}(\phi;t)]
            \bigr\rangle_{\mathcal B}$.
          \STATE Take an optimizer step to minimize $\mathcal L_Q$
            with respect to $\phi$, holding the critic targets fixed.
          \STATE $\tau\leftarrow e^\eta$, $\lambda\leftarrow e^\xi$.
          \STATE For each $s\in\mathcal B$, draw
            $a_i(s)\sim\pi_{\mathrm{roll}}(\cdot\mid s)$,
            $i=1,\ldots,K$, without gradients.
          \STATE Evaluate $v_i(s)\leftarrow Q_\phi(s,a_i(s))$ and
            $b_i(s)\leftarrow\log\pi_{\mathrm{roll}}(a_i(s)\mid s)$
            without gradients.
          \STATE $\ell_i(s)\leftarrow[v_i(s)-\tau b_i(s)]/(\tau+\lambda)$.
          \STATE $w_i(s)\leftarrow\operatorname{sg}
            \bigl[\operatorname{softmax}_i(\ell(s))\bigr]$
            \COMMENT{normalize over the $K$ actions at each state}
          \STATE $\widehat{\mathcal L}_{\mathrm F}\leftarrow
            -\bigl\langle\sum_{i=1}^K w_i(s)
            \log\pi_\theta(a_i(s)\mid s)\bigr\rangle_{\mathcal B}$.
          \STATE Draw fresh $\widetilde a_s\sim\pi_\theta(\cdot\mid s)$
            and estimate $\widehat{\mathcal H}\leftarrow
            -\langle\log\pi_\theta(\widetilde a_s\mid s)\rangle_{\mathcal B}$.
          \STATE $\widehat e(s)\leftarrow
            \bigl(K\sum_{i=1}^K w_i(s)^2\bigr)^{-1}$.
          \STATE $\mathcal L_\tau\leftarrow\tau\,
            \operatorname{sg}[\widehat{\mathcal H}-\mathcal H_{\mathrm{target}}]$.
          \STATE $\mathcal L_\lambda\leftarrow\lambda\,
            \operatorname{sg}[\langle\log\widehat e(s)\rangle_{\mathcal B}
            -\log e_{\mathrm{target}}]$.
          \STATE Take an optimizer step to minimize
            $\widehat{\mathcal L}_{\mathrm F}+\mathcal L_\tau+
            \mathcal L_\lambda$ with respect to $(\theta,\eta,\xi)$,
            holding $\phi$ and $\pi_{\mathrm{roll}}$ fixed.
        \ENDFOR
      \ENDFOR
      \STATE Discard $\mathcal D$ before collecting the next rollout.
    \ENDWHILE
    \RETURN $\pi_\theta$.
  \end{algorithmic}
\end{algorithm}

$\mathcal L_{\mathrm{HL\text{-}Gauss}}$ matches the critic's categorical
prediction to a Gaussian-smoothed histogram of the scalar target $G_t$;
$Q_\phi(s,a)$ is the mean of this prediction.
$\mathcal L_{\mathrm{aux}}$ predicts the next critic representation and
reward, with fixed targets and zero loss at episode ends
(Section~\ref{sec:critic_learning}). The target recursion adds the next
action's entropy contribution once, outside the TD mixture. Actor
updates differentiate only policy log probabilities and multiplier
losses; their gradients do not change the critic or rollout policy.

{\textbf{Policy parameterization and action bounds.}
For bounded-action environments, we use a tanh-squashed Gaussian policy
$\pi_\theta$, following REPPO \citep{voelcker2026relative} and SAC
\citep{haarnoja2018soft}:
$z\sim\mathcal N(\mu_\theta(s),\operatorname{diag}(\sigma_\theta^2(s)))$,
$a=\tanh(z)$, where a neural network parameterizes $\mu_\theta$ and
$\sigma_\theta$. Actions are rescaled to the environment's action bounds
where needed. Policy log densities include the corresponding
change-of-variables correction.}

For cached FERPO, complete critic training first, freeze $Q_\phi$, and
sample and store $a_i(s)$, $v_i(s)$, and $b_i(s)$ once for every rollout
state. Reuse these tensors throughout actor training, recomputing
$\ell_i(s)$ and $w_i(s)$ as $\tau$ and $\lambda$ change; release the
cache before the next rollout. Appendix~\ref{app:actor_cost} reports
the runtime and memory costs of this schedule.
\endgroup

\subsection{State-dependent normalization in policy fitting}
\label{app:state_normalization}
\begingroup

AWR's practical regression objective omits the per-state normalizer
\citep[Eqs.~8--10]{peng2019advantage}; AWAC explicitly discusses the
same omission \citep[Appendix~A-B]{nair2020awac}.
To isolate its effect, fix a state distribution $d$, the rollout policy
$\pi_{\mathrm{roll}}(\cdot\mid s)$, advantage estimates $A(s,a)$, and
$\lambda>0$. We consider the exponential weights before
clipping and assume the expectations and gradients below exist.
Define
\begin{equation}
  \begin{aligned}
    u(s,a)&:=\exp(A(s,a)/\lambda),
    &Z(s)&:=\E_{a\sim\pi_{\mathrm{roll}}(\cdot\mid s)}[u(s,a)],\\
    q_A(a\mid s)&:=\frac{\pi_{\mathrm{roll}}(a\mid s)u(s,a)}{Z(s)},
    &0&<Z(s)<\infty.
  \end{aligned}
  \label{eq:advantage_normalized_target}
\end{equation}
Let $\ell_\theta(s):=-\E_{a\sim q_A(\cdot\mid s)}
[\log\pi_\theta(a\mid s)]$. The normalized forward-KL fit minimizes
$L_{\mathrm{norm}}(\theta):=\E_{s\sim d}[\ell_\theta(s)]$,
up to a constant independent of $\theta$. Omitting $Z(s)$ instead gives
\begin{equation}
  \begin{aligned}
    L_{\mathrm{omit}}(\theta)
    &:=-\E_{s\sim d}\E_{a\sim\pi_{\mathrm{roll}}(\cdot\mid s)}
      [u(s,a)\log\pi_\theta(a\mid s)]\\
    &=\E_{s\sim d}[Z(s)\ell_\theta(s)].
  \end{aligned}
  \label{eq:normalizer_state_reweighting}
\end{equation}
Thus, holding the target weights fixed during actor fitting and writing
$g_\theta(s):=\nabla_\theta\ell_\theta(s)$,
\begin{equation}
  \nabla_\theta L_{\mathrm{norm}}=\E_{s\sim d}[g_\theta(s)],
  \qquad
  \nabla_\theta L_{\mathrm{omit}}=\E_{s\sim d}[Z(s)g_\theta(s)].
  \label{eq:normalizer_gradient_reweighting}
\end{equation}
The omission therefore changes the relative influence of states even
with unlimited data. A constant $Z(s)$ only rescales the gradient.
With independent, unrestricted fitting at each state, both objectives
recover $q_A$. With a shared neural policy, varying $Z(s)$ can change
the overall gradient direction and, when the target cannot be fitted
exactly, the optimal approximation. The size of this effect depends on
the variation in $Z(s)$ and the alignment of gradients across states.

An advantage baseline does not generally remove this reweighting:
for $A(s,a)=Q(s,a)-V(s)$,
$Z(s)=e^{-V(s)/\lambda}\E_{a\sim\pi_{\mathrm{roll}}(\cdot\mid s)}[e^{Q(s,a)/\lambda}]$,
which need not equal one. The baseline cancels from $q_A$ after
normalization, but affects state weights when $Z(s)$ is omitted.
AWR's additional weight clipping changes the action weights as well
\citep[Section~3.3]{peng2019advantage}.

FERPO estimates normalized weights separately within each state using
Eq.~\ref{eq:sampled_forward_kl}. SNIS has finite-sample bias but is
consistent under the usual integrability and sampling assumptions
\citep[Section~9.2]{owen2013monte}; see also
Appendix~\ref{app:projection_importance_sampling}.
This establishes a difference in the fitted objective, rather than a
general performance ordering.
\endgroup

\section{Proofs}
\label{app:proofs}

\subsection{Derivation of the KL surrogate}
\label{app:kl_surrogate}

\begingroup
\textbf{State distribution and expectation in Eq.~\ref{eq:maxent_objective}.}
\begingroup
For a fixed policy $\pi$, let $\rho_t^\pi$ denote the state density
at time $t$ in trajectories generated by $s_0\sim\rho_0$,
$a_t\sim\pi(\cdot\mid s_t)$, and
$s_{t+1}\sim P(\cdot\mid s_t,a_t)$. Thus $\rho_0^\pi=\rho_0$. As in
Section~\ref{sec:problem_description}, define the normalized discounted
state-visitation distribution
\begin{equation*}
  d^\pi(s):=(1-\gamma)\sum_{t=0}^{\infty}\gamma^t\rho_t^\pi(s),
  \qquad d_{\mathrm{roll}}:=d^{\pi_{\mathrm{roll}}}.
\end{equation*}
Whereas $\rho_t^\pi$ describes the states at a single time step,
$d^\pi$ averages these distributions over time with normalized
discount weights.
\endgroup
For any function $f(s,a)$ with finite absolute discounted expectation,
the definition of $d^\pi$ gives
\begin{equation*}
  \frac{1}{1-\gamma}
      \E_{\substack{s\sim d^\pi\\a\sim\pi(\cdot\mid s)}}[f(s,a)]
    =\sum_{t=0}^{\infty}\gamma^t
      \E_{\substack{s\sim\rho_t^\pi\\a\sim\pi(\cdot\mid s)}}[f(s,a)] =\E_{\rho_0,P,\pi}\!\left[
      \sum_{t=0}^{\infty}\gamma^t f(s_t,a_t)\right].
\end{equation*}
Taking $f(s,a)=r(s,a)-\tau\log\pi(a\mid s)$ shows that
Eq.~\ref{eq:maxent_objective} is exactly the discounted
maximum-entropy return.

To describe one step of state evolution under the
policy, define the transition kernel
\begin{equation*}
  P_\pi(s'\mid s):=\int\pi(a\mid s)P(s'\mid s,a)\,\mathrm da.
\end{equation*}
\begingroup
The state distribution evolves by averaging the next-state
distribution over the current states:
\begin{equation*}
  \rho_{t+1}^\pi(s')
  =\int\rho_t^\pi(s)P_\pi(s'\mid s)\,\mathrm ds
  =:(\rho_t^\pi P_\pi)(s').
\end{equation*}
Thus $\rho_{t+1}^\pi=\rho_t^\pi P_\pi$ is shorthand for this integral;
it describes one step under the environment dynamics.
\endgroup
Using this recursion and shifting the index in the discounted sum yields
$d^\pi=(1-\gamma)\rho_0+\gamma d^\pi P_\pi$.

\textbf{Gradient of the KL surrogate.}
\endgroup
Applying the policy-gradient theorem \citep{sutton1999policy} to the
entropy-regularized objective gives
\begin{equation}
  \nabla_\theta J_\tau(\pi_\theta)
  =\frac{1}{1-\gamma}
  \E_{\substack{s\sim d^{\pi_\theta}\\a\sim\pi_\theta(\cdot\mid s)}}
  \!\left[
    \nabla_\theta\log\pi_\theta(a\mid s)
    \bigl(Q_\tau^{\pi_\theta}(s,a)-\tau\log\pi_\theta(a\mid s)\bigr)
  \right].
  \label{eq:soft_policy_gradient}
\end{equation}
The additional derivative of the entropy term has zero expectation by
$\E_{a\sim\pi_\theta}[\nabla_\theta\log\pi_\theta(a\mid s)]=0$.

For the frozen rollout critic, define the normalized Boltzmann target
\begin{equation}
  p_{Q,\tau}(a\mid s)
  :=\frac{\exp(Q_{\mathrm{roll}}(s,a)/\tau)}{Z_Q(s)},
  \qquad
  Z_Q(s):=\int_{\mathcal A}
    \exp(Q_{\mathrm{roll}}(s,a)/\tau)\,\mathrm da.
  \label{eq:boltzmann_target}
\end{equation}
Expanding the reverse KL gives the per-state identity
\begin{equation}
  \tau\KL\!\left(\pi_\theta\middle\|p_{Q,\tau}\right)
  =-\E_{a\sim\pi_\theta}[Q_{\mathrm{roll}}(s,a)]
   -\tau\mathcal H(\pi_\theta(\cdot\mid s))
   +\tau\log Z_Q(s).
  \label{eq:reverse_kl_identity}
\end{equation}
Differentiating this identity while holding the rollout critic and
state distribution fixed and comparing with
Eq.~\ref{eq:soft_policy_gradient} yields the gradient-matching condition
stated in Section~\ref{sec:problem_description}.
Thus, minimizing Eq.~\ref{eq:maxent_surrogate} matches the local
return-ascent direction at the rollout parameters. The equivalence is
local because the frozen critic and state distribution do not track
subsequent policy changes.

\subsection{Derivation of the optimal regularized target}
\label{app:optimal_target}

Fix a state $s$ and write $Q(a):=Q_{\mathrm{roll}}(s,a)$ and
$\mu(a):=\pi_{\mathrm{roll}}(a\mid s)$. Assume $\mu>0$ almost everywhere on
$\mathcal A$ and that the integrals below are finite. For $\tau>0$ and
$\lambda\geq0$, the per-state contribution to
Eq.~\ref{eq:constrained_lagrangian} is
\begin{equation}
  \begin{aligned}
    \mathcal L_s(q;\tau,\lambda)
    &:=-\E_{a\sim q}[Q(a)]
      +\tau(\mathcal H_{\mathrm{target}}-\mathcal H(q))
      +\lambda(\KL(q\|\mu)-\delta_{\mathrm{KL}})\\
    &=\int_{\mathcal A}q(a)
      \bigl[(\tau+\lambda)\log q(a)-Q(a)-\lambda\log\mu(a)\bigr]
      \,\mathrm da
      +\tau\mathcal H_{\mathrm{target}}-\lambda\delta_{\mathrm{KL}}.
  \end{aligned}
  \label{eq:target_objective_expansion}
\end{equation}
Introduce a multiplier $\nu$ for the normalization constraint:
$\widetilde{\mathcal L}_s(q,\nu;\tau,\lambda)
:=\mathcal L_s(q;\tau,\lambda)
+\nu(\int_{\mathcal A}q(a)\,\mathrm da-1)$.
The stationarity condition is
\begin{equation}
  \frac{\delta\widetilde{\mathcal L}_s}{\delta q(a)}
  =(\tau+\lambda)(\log q(a)+1)-Q(a)-\lambda\log\mu(a)+\nu=0.
  \label{eq:target_stationarity}
\end{equation}
Solving for $q$ and normalizing gives
\begin{equation}
  \begin{aligned}
    q^*(a)
    &=\frac{1}{Z_*(s)}
      \exp\!\left(\frac{Q(a)}{\tau+\lambda}\right)
      \mu(a)^{\lambda/(\tau+\lambda)},\\
    Z_*(s)
    &:=\int_{\mathcal A}
      \exp\!\left(\frac{Q(a)}{\tau+\lambda}\right)
      \mu(a)^{\lambda/(\tau+\lambda)}\,\mathrm da,
  \end{aligned}
  \label{eq:normalized_optimal_target}
\end{equation}
which is Eq.~\ref{eq:optimal_target}. Substituting $\log q^*$ into
Eq.~\ref{eq:target_objective_expansion} yields the exact identity
\begin{equation}
  \mathcal L_s(q;\tau,\lambda)
  =(\tau+\lambda)\KL(q\|q^*)
    -(\tau+\lambda)\log Z_*(s)
    +\tau\mathcal H_{\mathrm{target}}-\lambda\delta_{\mathrm{KL}}.
  \label{eq:target_kl_decomposition}
\end{equation}
Since KL is nonnegative and vanishes only when its arguments agree
almost everywhere, $q^*$ is the unique global minimizer. Averaging
this identity over $d_{\mathrm{roll}}$ also proves the reverse-KL
projection form of Eq.~\ref{eq:constrained_lagrangian} used in
Section~\ref{sec:target_and_projection}.

For the forward-KL fit, the target-to-proposal density ratio is
\begin{equation}
  \frac{q^*(a)}{\mu(a)}
  =\frac{1}{Z_*(s)}
    \exp\!\left(\frac{Q(a)-\tau\log\mu(a)}{\tau+\lambda}\right).
  \label{eq:target_importance_ratio}
\end{equation}
Self-normalizing these ratios across sampled actions cancels $Z_*(s)$
and gives Eq.~\ref{eq:sampled_forward_kl}, with the exact critic
replaced by $Q_\phi$ in practice.

\subsection{Log ESS and divergence control}
\label{app:log_ess_proof}

\begin{proof}[Target-divergence bound in Proposition~\ref{prop:log_ess_chi_squared} and ESS identities]
Fix a state $s$, and write $q(a)=q_\lambda^*(a\mid s)$ for the
target density and $\mu(a)=\pi_{\mathrm{roll}}(a\mid s)$ for the
rollout-policy density. Define the importance ratio
$r(a):=q(a)/\mu(a)$ wherever $\mu(a)>0$. As assumed in the
proposition, $q\ll\mu$ and $\E_\mu[r^2]<\infty$. The population
normalized ESS is $e_\infty:=1/\E_\mu[r^2]$.
Because $q$ and $\mu$ are probability densities and $q\ll\mu$,
$\E_\mu[r]=\int q(a)\,\mathrm da=1$. Therefore,
\begin{equation*}
  \chi^2(q\|\mu)=\E_\mu[(r-1)^2]=\E_\mu[r^2]-1,
  \qquad e_\infty=\frac{1}{1+\chi^2(q\|\mu)}.
\end{equation*}
Jensen's inequality for the concave logarithm gives
\begin{equation*}
  \KL(q\|\mu)=\E_q[\log r]
  \leq\log\E_q[r]
  =\log\E_\mu[r^2]
  =D_2(q\|\mu)
  =-\log e_\infty.
\end{equation*}
Since $-\log x$ is decreasing, $e_\infty\geq e_{\mathrm{target}}$
implies $\KL(q\|\mu)\leq-\log e_{\mathrm{target}}$.
For independent samples $a_1,\ldots,a_K\sim\mu$, define the
normalized weights $w_i=r(a_i)/\sum_jr(a_j)$ and empirical normalized
ESS $\widehat e:=1/(K\sum_iw_i^2)$. Substituting the weights gives
\begin{equation*}
  \widehat e
  =\frac{\bigl(K^{-1}\sum_{i=1}^K r(a_i)\bigr)^2}
         {K^{-1}\sum_{i=1}^K r(a_i)^2}
  \xrightarrow[K\to\infty]{\mathrm{a.s.}}
  \frac{(\E_\mu[r])^2}{\E_\mu[r^2]}
  =e_\infty.
\end{equation*}
The convergence follows from the strong law of large numbers and the
finite second-moment assumption. If an initial batch has zero total
weight, the normalized estimate is undefined; the sample average of
$r$ converges to one, so the denominator for weight normalization is
positive for all sufficiently large $K$ almost surely.
\end{proof}

\textbf{Finite-batch interpretation.}
There is an exact identity for any finite batch of normalized
weights $w=(w_1,\ldots,w_K)$. With $u_K(i)=1/K$ on sample indices,
\begin{equation*}
  \chi^2(w\|u_K)=K\sum_iw_i^2-1,
  \qquad
  -\log\widehat e=\log\!\left(1+\chi^2(w\|u_K)\right)
                 =D_2(w\|u_K).
\end{equation*}
This finite-batch equality concerns probability vectors on sampled
indices. It does not identify $-\log\widehat e$ with the divergence
between the full action distributions $q$ and $\mu$.

\begingroup
\subsubsection{{From target ESS to policy total variation}}
\label{app:ess_actor_tv}

A decrease in the population forward-KL fitting loss suffices to
transfer the target's divergence control to the actor in total
variation. We use the half-$L^1$ convention in
Section~\ref{sec:problem_description} and natural logarithms throughout.

We restate the actor-control assumption of
Proposition~\ref{prop:log_ess_chi_squared}.
Fix a state and a target $q^*$, suppressing state conditioning.
Assume $q^*\ll\pi_{\mathrm{roll}}$ and
$\E_{\pi_{\mathrm{roll}}}[(q^*/\pi_{\mathrm{roll}})^2]<\infty$,
and let
$e_\infty:=1/\E_{\pi_{\mathrm{roll}}}[(q^*/\pi_{\mathrm{roll}})^2]$.
Suppose the actor update does not increase the population forward-KL
fitting objective:
\begin{equation}
  \KL(q^*\|\pi_{\mathrm{curr}})
  \leq\KL(q^*\|\pi_{\mathrm{roll}}).
  \label{eq:ess_actor_fit_improvement}
\end{equation}
Then Eq.~\ref{eq:ess_actor_tv_bound} holds:
\begin{equation*}
  \begin{aligned}
    \frac12D_{\mathrm{TV}}(\pi_{\mathrm{curr}},\pi_{\mathrm{roll}})^2
    &\leq\KL(q^*\|\pi_{\mathrm{roll}})\\
    &\leq D_2(q^*\|\pi_{\mathrm{roll}})=-\log e_\infty.
  \end{aligned}
\end{equation*}
No exact minimization or Gaussian policy assumption is required.

\begin{proof}[Proof of the actor-TV bound in Proposition~\ref{prop:log_ess_chi_squared}]
The triangle inequality, obtained by adding and subtracting $q^*$
inside the absolute value in the TV definition, gives
\begin{equation*}
  D_{\mathrm{TV}}(\pi_{\mathrm{curr}},\pi_{\mathrm{roll}})
  \leq D_{\mathrm{TV}}(\pi_{\mathrm{curr}},q^*)
      +D_{\mathrm{TV}}(q^*,\pi_{\mathrm{roll}}).
\end{equation*}
TV is symmetric, so Pinsker's inequality
\citep[Theorem~31]{vanerven2014renyi} can bound both terms using
KL divergences with $q^*$ as their first argument. Applying
Eq.~\ref{eq:ess_actor_fit_improvement} then yields
\begin{equation*}
  \begin{aligned}
    D_{\mathrm{TV}}(\pi_{\mathrm{curr}},\pi_{\mathrm{roll}})
    &\leq\sqrt{\tfrac12\KL(q^*\|\pi_{\mathrm{curr}})}
         +\sqrt{\tfrac12\KL(q^*\|\pi_{\mathrm{roll}})}\\
    &\leq 2\sqrt{\tfrac12\KL(q^*\|\pi_{\mathrm{roll}})}
     =\sqrt{2\KL(q^*\|\pi_{\mathrm{roll}})}.
  \end{aligned}
\end{equation*}
Squaring both sides and dividing by two gives
$\tfrac12D_{\mathrm{TV}}(\pi_{\mathrm{curr}},\pi_{\mathrm{roll}})^2
\leq\KL(q^*\|\pi_{\mathrm{roll}})$.
The target-divergence result proved above supplies the remaining
inequality and identity,
$\KL(q^*\|\pi_{\mathrm{roll}})
\leq D_2(q^*\|\pi_{\mathrm{roll}})=-\log e_\infty$.
\end{proof}

\textbf{Meaning of fitting improvement.}
With $q^*$ fixed, the difference of the fitting objectives is
\begin{equation*}
  \KL(q^*\|\pi_{\mathrm{curr}})
   -\KL(q^*\|\pi_{\mathrm{roll}})
  =-\E_{q^*}\!\left[
       \log\frac{\pi_{\mathrm{curr}}}{\pi_{\mathrm{roll}}}\right].
\end{equation*}
Thus Eq.~\ref{eq:ess_actor_fit_improvement} requires only that the
actor increase or preserve the target-weighted population log
likelihood. This is the objective approximated by the SNIS actor
loss; it is distinct from an assumption of higher environment return.

\textbf{Averaging over rollout states.}
The result also holds on average without per-state fitting
improvement. All distributions below are conditioned on $s$.
Squaring the first Pinsker bound above and using
$(\sqrt{x}+\sqrt{y})^2\leq2(x+y)$ gives the pointwise inequality
\begin{equation*}
  D_{\mathrm{TV}}(\pi_{\mathrm{curr}},\pi_{\mathrm{roll}})^2
  \leq\KL(q^*\|\pi_{\mathrm{curr}})
      +\KL(q^*\|\pi_{\mathrm{roll}}).
\end{equation*}
Assume the relevant expectations are finite and that
\begin{equation*}
  \E_{s\sim d_{\mathrm{roll}}}[\KL(q^*\|\pi_{\mathrm{curr}})]
  \leq\E_{s\sim d_{\mathrm{roll}}}[\KL(q^*\|\pi_{\mathrm{roll}})].
\end{equation*}
Averaging the pointwise inequality and applying the target-divergence
bound in Eq.~\ref{eq:ess_actor_tv_bound} yields
\begin{equation*}
  \begin{aligned}
    \frac12\E_{s\sim d_{\mathrm{roll}}}
      [D_{\mathrm{TV}}(\pi_{\mathrm{curr}},\pi_{\mathrm{roll}})^2]
    &\leq\E_{s\sim d_{\mathrm{roll}}}
      [\KL(q^*\|\pi_{\mathrm{roll}})]\\
    &\leq-\E_{s\sim d_{\mathrm{roll}}}[\log e_\infty(s)].
  \end{aligned}
\end{equation*}
If the geometric mean of population ESS meets the controller target,
$\E_{s\sim d_{\mathrm{roll}}}[\log e_\infty(s)]
\geq\log e_{\mathrm{target}}$, the right-hand side is at most
$-\log e_{\mathrm{target}}$.
This controls average squared TV, rather than the largest deviation
across states. Decreases in finite-sample SNIS loss and empirical ESS
control do not by themselves guarantee these population conditions.
Nor does the TV bound imply an upper bound on the fitted actor's KL;
it cannot replace the KL penalty in
Eq.~\ref{eq:trust_region_bound} without further assumptions.
\endgroup

\subsection{Convergence under exact fitting}
\label{app:exact_fit_proof}

The regularized target in Eq.~\ref{eq:optimal_target} interpolates
between the reference policy and the Boltzmann distribution. Repeated
exact updates reach the Boltzmann distribution when the reference is
replaced by the fitted policy after every update. The following result
quantifies how $\lambda$ controls this convergence for a fixed critic.

\begin{proposition}[Convergence of exact FERPO updates for a fixed critic]
\label{prop:exact_fit_convergence}
Fix a state $s$, a critic $Q(s,\cdot)$, and a temperature $\tau>0$.
Suppress state conditioning and let
$p(a):=\exp(Q(s,a)/\tau)/Z(s)$ be normalizable.
Assume $\KL(p\|\pi_0)<\infty$ and fix $0<\lambda<\infty$.
At each iteration, compute the unrestricted target and fit it exactly:
\begin{equation*}
  q_n^*:=\operatorname*{arg\,min}_q
    \left\{\tau\KL(q\|p)+\lambda\KL(q\|\pi_n)\right\},
  \qquad \pi_{n+1}=q_n^*.
\end{equation*}
The policy after $n\geq1$ exact updates is
\begin{equation}
  \pi_n(a)
  =\frac{1}{Z_n(s)}
    \pi_0(a)^{\left(\frac{\lambda}{\tau+\lambda}\right)^n}
    \exp\!\left[
      \frac{1-\left(\frac{\lambda}{\tau+\lambda}\right)^n}{\tau}
      Q(s,a)
    \right],
  \label{eq:exact_fit_closed_form}
\end{equation}
where $Z_n(s)$ normalizes the density. Since
$\left(\lambda/(\tau+\lambda)\right)^n\to0$ as $n\to\infty$,
the initial policy's exponent vanishes and the coefficient of $Q(s,a)$
approaches $1/\tau$. Consequently,
\begin{equation*}
  \lim_{n\to\infty}\pi_n(a)
  =\frac{\exp(Q(s,a)/\tau)}
         {\int_{\mathcal A}\exp(Q(s,b)/\tau)\,\mathrm db}
  =p(a)
  \quad\text{almost everywhere},
\end{equation*}
and convergence also holds in total variation. The limit is the unique
maximizer of $\E_{a\sim q}[Q(s,a)-\tau\log q(a)]$, independently of
the fixed finite value of $\lambda$. For $\lambda=0$, the target is
reached in one exact update.
\end{proposition}

\textbf{The role of $\lambda$.}
Equation~\ref{eq:exact_fit_closed_form} separates the influence of the
initial policy from the Boltzmann target. Larger $\lambda$ makes
$\lambda/(\tau+\lambda)$ closer to one, so the initial policy's
influence decays more slowly. Smaller $\lambda$ gives faster progress
toward the same target. Thus, $\lambda$ controls the rate of approach,
while $\tau$ determines the limiting Boltzmann distribution.

The proposition assumes exact forward-KL fits and a refreshed reference
policy at every iteration. Repeated fitting to a single frozen target
only recovers that target. It establishes convergence to the Boltzmann
distribution of the fixed critic; convergence of full RL training,
where the critic is re-estimated, requires a separate policy-iteration
argument.

\begin{proof}[Proof of Proposition~\ref{prop:exact_fit_convergence}]
Write $\beta:=\lambda/(\tau+\lambda)\in(0,1)$.
Equation~\ref{eq:optimal_target} gives
$\pi_{n+1}=p^{1-\beta}\pi_n^\beta/C_n$, where
$C_n:=\int p^{1-\beta}\pi_n^\beta\,\mathrm da$.
Induction gives
$\pi_n\propto p^{1-\beta^n}\pi_0^{\beta^n}$.
Substituting $p\propto\exp(Q/\tau)$ and
$\beta=\lambda/(\tau+\lambda)$ yields
Eq.~\ref{eq:exact_fit_closed_form}.
The finite-KL assumption ensures that $\pi_0$ covers the support of $p$;
this property is preserved by the update. As $n\to\infty$,
$p^{1-\beta^n}\pi_0^{\beta^n}\to p$ almost everywhere.
The bound $p^{1-t}\pi_0^t\leq p+\pi_0$ for $0<t<1$ permits dominated
convergence, so the normalizing integral of this geometric mixture
tends to one. This proves the stated almost-everywhere limit.

To quantify convergence, H\"older's inequality and normalization imply
$0<C_n\leq1$. Therefore,
\begin{equation}
  \KL(p\|\pi_{n+1})
  =\beta\KL(p\|\pi_n)+\log C_n
  \leq\beta\KL(p\|\pi_n).
  \label{eq:exact_fit_kl_identity}
\end{equation}
Iterating gives the KL contraction bound
\begin{equation}
  \KL(p\|\pi_{n+1})\leq\beta\KL(p\|\pi_n),
  \qquad
  \KL(p\|\pi_n)\leq\beta^n\KL(p\|\pi_0).
  \label{eq:exact_fit_kl_rate}
\end{equation}
Pinsker's inequality then implies convergence in total variation.
Optimality of $p$ follows
from Eq.~\ref{eq:reverse_kl_identity}. For $\lambda=0$, the update
minimizes $\KL(q\|p)$ and reaches $p$ immediately.
\end{proof}

\textbf{Finite-step accuracy and varying multipliers.}
If $D_0:=\KL(p\|\pi_0)>\varepsilon>0$, then
$n\geq\lceil\log(D_0/\varepsilon)/\log(1+\tau/\lambda)\rceil$
exact updates suffice for $\KL(p\|\pi_n)\leq\varepsilon$.
For varying coefficients $\lambda_k>0$ at fixed $Q$ and $\tau$,
$\beta^n$ is replaced by
$\prod_{k=0}^{n-1}\lambda_k/(\tau+\lambda_k)$; convergence follows if
this product tends to zero, as it does for uniformly bounded
$\lambda_k$.

\subsection{Update directions and importance sampling}
\label{app:projection_geometry}

For a Gaussian policy with fixed covariance, both REPPO and FERPO have
mean gradients that can be expressed as averages of cost gradients.
REPPO averages under the current policy; FERPO averages under its
cost-reweighted target, with a factor $1/\lambda$. We derive these
identities and explain how the different distributions produce the
opposing directions in Figure~\ref{fig:bimodal_actor_losses}.
We then show why recovering the target direction from samples
requires adequate proposal coverage.

\subsubsection{Cost gradients under the policy and the target}
\label{app:actor_direction_comparison}

Fix a state and a cost $c(a)=-Q_\phi(s,a)$ on $\mathbb R^d$.
Let $\pi_m=\mathcal N(m,\Sigma)$ be the reference and initial actor,
with fixed covariance $\Sigma\succ0$. The candidate actor is
$\pi_n=\mathcal N(n,\Sigma)$. Fix $\tau,\lambda>0$ and write
$T=\tau+\lambda$. FERPO constructs
\begin{equation}
  q_m^*(a)=\frac{1}{Z_m}e^{-c(a)/T}\pi_m(a)^{\lambda/T},
  \label{eq:direction_target}
\end{equation}
and holds $q_m^*$ fixed while fitting $\pi_n$.
Let $m^*(m):=\E_{a\sim q_m^*}[a]$ denote the target mean.
Assume $c$ is continuously differentiable and bounded below, and that
$c$ and $\nabla c$ grow at most polynomially. These conditions ensure
finite target moments and justify the differentiation below.

\begin{proposition}[Initial mean-descent directions]
\label{prop:actor_direction_comparison}
At $n=m$, REPPO's KL regularization is inactive. Its mean loss is
$L_{\mathrm R}(n)=\E_{\pi_n}[c]-\tau H(\pi_n)$, while FERPO fits
$L_{\mathrm F}(n;m)=\KL(q_m^*\|\pi_n)$. Their negative gradients are
\begin{equation}
  \begin{aligned}
    d_{\mathrm R}(m)
      &:=-\left.\nabla_n L_{\mathrm R}(n)\right|_{n=m}
        =-\E_{a\sim\pi_m}[\nabla c(a)],\\
    d_{\mathrm F}(m)
      &:=-\left.\nabla_n L_{\mathrm F}(n;m)\right|_{n=m}
        =\Sigma^{-1}\bigl(m^*(m)-m\bigr)
        =-\frac{1}{\lambda}\E_{a\sim q_m^*}[\nabla c(a)].
  \end{aligned}
  \label{eq:actor_direction_comparison}
\end{equation}
The exact forward-KL mean fit sets $n=m^*(m)$.
Reverse KL to the same frozen target has the same initial descent
\emph{direction} as REPPO, with its negative gradient equal to
$d_{\mathrm R}(m)/T$.
\end{proposition}

\begin{proof}
Using the reparameterization $a=n+\Sigma^{1/2}z$, where
$z\sim\mathcal N(0,I)$, we obtain
\begin{equation*}
  \nabla_n\E_{a\sim\pi_n}[c(a)]
  =\E_{a\sim\pi_n}[\nabla_a c(a)].
\end{equation*}
Since the covariance is fixed, the entropy does not depend on $n$.
This gives REPPO's mean-descent direction.
For FERPO, expansion of the Gaussian log density gives
\begin{equation}
  L_{\mathrm F}(n;m)
   =\frac12\bigl(n-m^*(m)\bigr)^\top\Sigma^{-1}
                  \bigl(n-m^*(m)\bigr)+\mathrm{const}(m).
  \label{eq:gaussian_frozen_fit}
\end{equation}
Differentiating with respect to $n$ gives the FERPO
gradient, and the quadratic is minimized at $n=m^*(m)$.

To express FERPO's gradient in terms of cost gradients, differentiate
the target log density with respect to the action:
\begin{equation}
  \nabla_a\log q_m^*(a)
    =-\frac{1}{T}\bigl(\nabla c(a)+\lambda\Sigma^{-1}(a-m)\bigr).
  \label{eq:target_action_score}
\end{equation}
The Gaussian factor and the lower bound on $c$ make the target density
vanish at infinity; the polynomial growth assumptions ensure
integrability. Integrating its derivative over $\mathbb R^d$ therefore
gives $\E_{q_m^*}[\nabla_a\log q_m^*]=\int\nabla_a q_m^*(a)\,\mathrm da=0$.
Taking expectations in Eq.~\ref{eq:target_action_score} yields
\begin{equation}
  \E_{q_m^*}[\nabla c(a)]
    =\lambda\Sigma^{-1}\bigl(m-m^*(m)\bigr),
  \label{eq:target_cost_gradient_identity}
\end{equation}
which proves the final equality for $d_{\mathrm F}$.

\textbf{Connection to REPPO.}
We now consider fitting the same frozen target using reverse KL,
$\KL(\pi_n\|q_m^*)$. FERPO instead minimizes the forward KL,
$\KL(q_m^*\|\pi_n)$. The following decomposition shows that reverse-KL
fitting gives the same initial descent direction as REPPO, up to a
positive scaling factor. Substituting the target density in
Eq.~\ref{eq:direction_target} yields
\begin{equation}
  T\KL(\pi_n\|q_m^*)
    =\E_{\pi_n}[c]-\tau H(\pi_n)
       +\lambda\KL(\pi_n\|\pi_m)+T\log Z_m.
  \label{eq:reverse_projection_cost}
\end{equation}
At $n=m$, the KL penalty contributes no gradient, and $Z_m$ is
independent of the candidate mean $n$. Differentiating the right-hand
side therefore gives REPPO's loss gradient. Consequently, minimizing
$\KL(\pi_n\|q_m^*)$ gives the initial descent direction
$d_{\mathrm R}(m)/T$: the same direction as REPPO, scaled by $1/T$.
\end{proof}

\textbf{Shared small-variance limit.}
Assume additionally that $c$ is three-times continuously differentiable
and its derivatives through order three grow at most polynomially.
Fix the mean $m$ and positive constants $\tau$ and $\lambda$.
Let $\Sigma=\sigma^2I$ with $\sigma\to0$. The Taylor expansion underlying the delta method
\citep[Appendix~C]{khan2023bayesian} applies to both averages.
Under $\pi_m$, $\E[a-m]=0$ and $\E\|a-m\|^2=O(\sigma^2)$.
The target is proportional to $e^{-c(a)/T}$ times the density of
$\mathcal N(m,(T/\lambda)\sigma^2I)$. The lower bound on $c$ and
the growth assumptions imply
$\E_{q_m^*}\|a-m\|^2=O(\sigma^2)$ and
$\E_{q_m^*}[\nabla c(a)]=O(1)$.
Equation~\ref{eq:target_cost_gradient_identity} then gives
$\E_{q_m^*}[a-m]=O(\sigma^2)$.
Expanding $\nabla c(a)$ around $m$ therefore yields, for the initial
loss gradients $g_{\mathrm R}:=-d_{\mathrm R}(m)$ and
$g_{\mathrm F}:=-d_{\mathrm F}(m)$,
\begin{equation*}
  g_{\mathrm R}=\nabla c(m)+O(\sigma^2),
  \qquad
  \lambda g_{\mathrm F}=\nabla c(m)+O(\sigma^2).
\end{equation*}
Thus both methods approach the same descent direction wherever
$\nabla c(m)\ne0$; in one dimension, their signs agree for sufficiently
small $\sigma$. This shared limit does not imply agreement at finite
width, where the target can be broad or multimodal, or uniformly near
stationary points. Figure~\ref{fig:bimodal_actor_losses} uses the exact
expectations at finite width.

At B in Figure~\ref{fig:bimodal_actor_losses}, $m=-0.6$ and
$\sigma=0.4$. The expected cost derivative is $0.284$, while
$m^*(m)\simeq-0.466$, giving
\begin{equation*}
  d_{\mathrm R}=-0.284,
  \qquad d_{\mathrm F}\simeq\frac{-0.466+0.6}{0.4^2}=0.838.
\end{equation*}
Thus FERPO initially moves toward the better right well while REPPO
moves toward the worse left well. This is the potential benefit
illustrated by the experiment: fitting the reweighted mean can give
an update toward a better region even when the expected-cost gradient
points elsewhere. The comparison concerns directions, not matched
step sizes or eventual optimization success. It applies at the
reference policy with fixed covariance.

\subsubsection{Recovering the target-mean direction from samples}
\label{app:projection_importance_sampling}

The update-direction comparison in
Appendix~\ref{app:actor_direction_comparison} uses the exact target mean. With SNIS,
FERPO instead uses $\widehat m^*=\sum_iw_i a_i$ and the estimated
negative gradient is $\widehat d_{\mathrm F}=\Sigma^{-1}(\widehat m^*-m)$.
The integration-by-parts identity holds for the population target;
it need not hold for a finite weighted sample. In particular, SNIS
is generally biased at finite $K$.
For a proposal $r$ that is positive wherever $q_m^*$ is positive,
independent samples $a_i\sim r$ receive weights
\begin{equation}
  u_i=\frac{e^{-c(a_i)/T}\pi_m(a_i)^{\lambda/T}}{r(a_i)},
  \qquad w_i=\frac{u_i}{\sum_j u_j}.
  \label{eq:general_proposal_weights}
\end{equation}
Changing $r$ changes the estimate's accuracy while preserving the
same target, provided the denominator uses the actual proposal.

\begin{proposition}[Accuracy of the sampled target mean]
\label{prop:projection_snis}
For fixed $q_m^*$ and $r$, $\widehat m^*\to m^*(m)$ almost surely as
$K\to\infty$. If the matrix
\begin{equation}
  \Omega_r:=\int\frac{q_m^*(a)^2}{r(a)}
      (a-m^*(m))(a-m^*(m))^\top\,\mathrm da
  \label{eq:projection_snis_covariance}
\end{equation}
has finite trace, then
$\sqrt K\bigl(\widehat m^*-m^*(m)\bigr)
\Longrightarrow\mathcal N(0,\Omega_r)$.
Here, $\Longrightarrow$ denotes convergence in distribution.
For large $K$, the estimated target mean is approximately Gaussian
around the exact mean, with covariance $\Omega_r/K$. Thus, the typical
estimation error decreases as $1/\sqrt K$, while its magnitude depends
on how well the proposal covers the target distribution.
\end{proposition}

\begin{proof}
This is the standard SNIS argument \citep[Section~9.2]{owen2013monte}.
Let $W=q_m^*/r$. Then $\E_rW=1$, $\E_r[Wa]=m^*(m)$, and
$\widehat m^*=(\sum_iW(a_i)a_i)/(\sum_iW(a_i))$.
The strong law of large numbers proves consistency, since the target
has a finite first moment. For the limiting distribution, write
\[
  \sqrt K(\widehat m^*-m^*(m))
  =\frac{K^{-1/2}\sum_i W(a_i)(a_i-m^*(m))}
         {K^{-1}\sum_i W(a_i)}.
\]
The numerator summands have mean zero and covariance $\Omega_r$.
The multivariate central limit theorem and convergence of the
denominator to one give the result.
\end{proof}

\textbf{Why a wider proposal can help.}
The covariance in Eq.~\ref{eq:projection_snis_covariance} is large
when the target places mass far from its mean in a region that the
proposal rarely samples. Such a region contributes to the exact mean
but is often absent from a finite batch. Increasing proposal density
there can reduce this source of error. This explains the improvement
from wider proposals in Figure~\ref{fig:bimodal_snis}; increasing
width is useful only insofar as it improves coverage of influential
target regions.

Although SNIS is consistent, finite-sample error can place the estimated
target mean on the wrong side of the current policy mean, reversing
the update direction.

\subsection{From surrogate improvement to return improvement}
\label{app:trust_region_bound}

\begingroup
\newcommand{\boundsubstep}[3]{\par\noindent\phantomsection\label{#3}\textbf{Substep~#1: #2.}\space\ignorespaces}
{
\noindent\textbf{Statement to be proved.}
We restate Eq.~\ref{eq:trust_region_bound} from the main text.
For a fixed temperature $\tau$, an exact rollout critic, and bounded
soft advantage, define the actual return change, the surrogate loss
change, and the predicted return gain as
\begin{equation*}
  \begin{aligned}
    \Delta J_\tau
      &:=J_\tau(\pi_\theta)-J_\tau(\pi_{\mathrm{roll}}),\\
    \Delta\mathcal L_{\mathrm{ME}}
      &:=\mathcal L_{\mathrm{ME}}(\theta;\theta_{\mathrm{roll}})
        -\mathcal L_{\mathrm{ME}}(\theta_{\mathrm{roll}};\theta_{\mathrm{roll}}),\\
    \Delta J_{\mathrm{surrogate}}&:=-\Delta\mathcal L_{\mathrm{ME}}.
  \end{aligned}
\end{equation*}
Both $\Delta J_\tau$ and $\Delta\mathcal L_{\mathrm{ME}}$ are
defined as candidate minus rollout. Since $J_\tau$ is maximized but
$\mathcal L_{\mathrm{ME}}$ is minimized, a negative loss change gives
a positive predicted return gain $\Delta J_{\mathrm{surrogate}}$.
{Under the assumptions detailed below, we prove the
bound with a TV advantage penalty and a KL entropy correction:}
\begingroup
\begin{equation*}
  \begin{aligned}
    \Delta J_\tau
    &\geq\Delta J_{\mathrm{surrogate}}
      -\frac{4\gamma\epsilon_A}{(1-\gamma)^2}
       (D_{\mathrm{TV}}^{\max})^2
      -\frac{\tau}{1-\gamma}D_{\mathrm{KL}}^{\max},\\
    D_{\mathrm{TV}}^{\max}
    &:=\sup_s D_{\mathrm{TV}}\!\left(
      \pi_\theta(\cdot\mid s),\pi_{\mathrm{roll}}(\cdot\mid s)
    \right),\\
    D_{\mathrm{KL}}^{\max}
    &:=\sup_s\KL\!\left(
      \pi_\theta(\cdot\mid s)\middle\|\pi_{\mathrm{roll}}(\cdot\mid s)
    \right).
  \end{aligned}
  \tag{\ref{eq:trust_region_bound}}
\end{equation*}
\endgroup
Here $\epsilon_A=\sup_{s,a}|A_{\mathrm{roll}}(s,a)|$ bounds the soft
advantage, defined below. If $\Delta J_{\mathrm{surrogate}}$ exceeds
{the sum of the two penalties}, then
$J_\tau(\pi_\theta)>J_\tau(\pi_{\mathrm{roll}})$.
}

\textbf{Proof roadmap.}
The surrogate uses rollout states, whereas the
return depends on states visited by the candidate policy. To bound
the resulting prediction error, we proceed as follows.
{
\begin{enumerate}
  \item \textbf{Express the surrogate gain.}
    Expand the KL losses in
    \hyperref[app:bound_kl_expansion]{Substep~1.2}, then average the
    resulting advantage and entropy correction over rollout states in
    \hyperref[app:bound_rollout_average]{Substep~1.4}.
  \item \textbf{Express the return gain in the same form.}
    Telescope the rollout value terms in
    \hyperref[app:bound_telescoping]{Substep~2.2}. In
    \hyperref[app:bound_candidate_average]{Substep~2.4}, show that the
    two gains differ only in their state distributions.
  \item \textbf{Bound the advantage error.}
    Bound the expected advantage in
    \hyperref[app:bound_advantage_size]{Substep~3.1} and the change in
    visited states in
    \hyperref[app:bound_policy_change]{Substep~3.5}. Combine these
    bounds into a {TV penalty} in
    \hyperref[app:bound_advantage_combine]{Substep~3.6}.
  \item \textbf{Complete the return-improvement bound.}
    Bound the entropy correction in
    \hyperref[app:bound_entropy_error]{Substep~4.1}, combine the two
    penalties in \hyperref[app:bound_combine]{Substep~4.2}, and state
    the sufficient condition for improvement in
    \hyperref[app:bound_interpretation]{Substep~4.3}.
\end{enumerate}
}

We adapt the entropy-regularized performance-difference
identity of \citet[Lemma~B.1]{cayci2024convergence}
(\hyperref[app:bound_candidate_average]{Substep~2.4})
and combine it with the TRPO
state-distribution mismatch argument \citep{schulman2015trust}
(\hyperref[app:bound_advantage_combine]{Substep~3.6}).

\textbf{Notation and assumptions.}
Keep $\tau$ and $\pi_{\mathrm{roll}}$ fixed throughout the update, and
let $\pi_\theta$ be the candidate policy. We use the main text's
$\Delta J_\tau$ and $\Delta J_{\mathrm{surrogate}}$, with a positive
sign denoting improvement in both cases. The exact rollout critic
defines
\begin{equation*}
  A_{\mathrm{roll}}(s,a)
  =Q_{\mathrm{roll}}(s,a)-\tau\log\pi_{\mathrm{roll}}(a\mid s)
   -V_\tau^{\pi_{\mathrm{roll}}}(s),
  \qquad
  \epsilon_A:=\sup_{s,a}|A_{\mathrm{roll}}(s,a)|<\infty.
\end{equation*}
Its expectation under the rollout policy is zero at every state.
Define
\begin{equation}
  \begin{aligned}
    \bar A_\theta(s)
      &:=\E_{a\sim\pi_\theta(\cdot\mid s)}[A_{\mathrm{roll}}(s,a)],\\
    k(s)&:=\KL\!\left(\pi_\theta(\cdot\mid s)
                         \middle\|\pi_{\mathrm{roll}}(\cdot\mid s)\right),
    \qquad
    D:=\sup_s k(s)
      =D_{\mathrm{KL}}^{\max}(\pi_\theta\|\pi_{\mathrm{roll}})<\infty.
  \end{aligned}
  \label{eq:bound_notation}
\end{equation}
Here $\bar A_\theta(s)$ averages the rollout advantage over the
candidate's actions. The divergence $k(s)$ measures change from the
rollout policy; the KL in Eq.~\ref{eq:maxent_surrogate} instead measures
distance to the Boltzmann target. We assume the relevant discounted
expectations are finite and
$\gamma^t\E_{\rho_0,P,\pi_\theta}
[|V_\tau^{\pi_{\mathrm{roll}}}(s_t)|]\to0$.
The latter condition, satisfied for bounded soft
values, removes the terminal term in
\hyperref[app:bound_telescoping]{Substep~2.2}.

\subsubsection*{Step 1: express the surrogate improvement}
We show that the surrogate gain is the rollout-state
average of $\bar A_\theta(s)-\tau k(s)$, scaled by $1/(1-\gamma)$.
This form will allow a direct comparison with the return gain.

\boundsubstep{1.1}{Write the difference of the two losses}{app:bound_loss_difference}
First, negate the candidate-minus-rollout loss
change to obtain the predicted return gain, then expand it as a
difference of KL divergences to the same Boltzmann target.
Both surrogate losses use the same frozen critic $Q_{\mathrm{roll}}$,
Boltzmann target $p_{Q,\tau}$, and state distribution
$d_{\mathrm{roll}}$. Only the policy inside the KL changes. Starting
from the definition of improvement in the main text and substituting
Eq.~\ref{eq:maxent_surrogate},
\begin{equation*}
  \begin{aligned}
    \Delta J_{\mathrm{surrogate}}
    &=-\Delta\mathcal L_{\mathrm{ME}}
     =\mathcal L_{\mathrm{ME}}(\theta_{\mathrm{roll}};
        \theta_{\mathrm{roll}})
      -\mathcal L_{\mathrm{ME}}(\theta;\theta_{\mathrm{roll}})\\
    &=\frac{\tau}{1-\gamma}\E_{s\sim d_{\mathrm{roll}}}\!\left[
      \KL(\pi_{\mathrm{roll}}\|p_{Q,\tau})
      -\KL(\pi_\theta\|p_{Q,\tau})\right].
  \end{aligned}
\end{equation*}
In the following per-state calculations, all action distributions and
action expectations are conditional on the same state $s$.

\boundsubstep{1.2}{Expand the KL losses and cancel the normalizer}{app:bound_kl_expansion}
We turn the KL difference into a difference of
entropy-augmented critic values. The target's normalizing constant
cancels because it is identical in both losses.
From Eq.~\ref{eq:boltzmann_target},
$\log p_{Q,\tau}(a\mid s)=Q_{\mathrm{roll}}(s,a)/\tau-\log Z_Q(s)$.
For either $\pi=\pi_{\mathrm{roll}}$ or $\pi=\pi_\theta$, this gives
\begin{equation*}
  \begin{aligned}
    \tau\KL(\pi\|p_{Q,\tau})
    &=\E_{a\sim\pi}\!\left[
       \tau\log\pi(a\mid s)-\tau\log p_{Q,\tau}(a\mid s)\right]\\
    &=\E_{a\sim\pi}\!\left[
       \tau\log\pi(a\mid s)-Q_{\mathrm{roll}}(s,a)\right]
       +\tau\log Z_Q(s).
  \end{aligned}
\end{equation*}
The last term is outside the action expectation because it depends
only on $s$. Subtracting the candidate's KL from the rollout policy's
KL now makes the cancellation explicit:
\begin{equation*}
  \begin{aligned}
    &\tau\bigl[\KL(\pi_{\mathrm{roll}}\|p_{Q,\tau})
                  -\KL(\pi_\theta\|p_{Q,\tau})\bigr]\\
    &\quad=\E_{a\sim\pi_{\mathrm{roll}}}\!\left[
       \tau\log\pi_{\mathrm{roll}}(a\mid s)-Q_{\mathrm{roll}}(s,a)\right]
       +\tau\log Z_Q(s)\\
    &\qquad\phantom{=}
       -\E_{a\sim\pi_\theta}\!\left[
       \tau\log\pi_\theta(a\mid s)-Q_{\mathrm{roll}}(s,a)\right]
       -\tau\log Z_Q(s)\\
    &\quad=\E_{a\sim\pi_\theta}\!\left[
       Q_{\mathrm{roll}}(s,a)-\tau\log\pi_\theta(a\mid s)\right]
       -V_\tau^{\pi_{\mathrm{roll}}}(s).
  \end{aligned}
\end{equation*}
Here we used the soft-value definition from
Section~\ref{sec:problem_description},
\begin{equation*}
  V_\tau^{\pi_{\mathrm{roll}}}(s)
  =\E_{a\sim\pi_{\mathrm{roll}}}\!\left[
     Q_{\mathrm{roll}}(s,a)-\tau\log\pi_{\mathrm{roll}}(a\mid s)\right].
\end{equation*}
Thus the rollout expectation in the subtraction equals
$-V_\tau^{\pi_{\mathrm{roll}}}(s)$.

\boundsubstep{1.3}{Identify the advantage and entropy correction}{app:bound_advantage_identity}
We now show that the remaining candidate expectation
equals $\bar A_\theta(s)-\tau k(s)$.
Add and subtract
$\tau\log\pi_{\mathrm{roll}}(a\mid s)$ for each action:
\begin{equation*}
  \begin{aligned}
    &Q_{\mathrm{roll}}(s,a)-\tau\log\pi_\theta(a\mid s)
       -V_\tau^{\pi_{\mathrm{roll}}}(s)\\
    &\quad=\underbrace{Q_{\mathrm{roll}}(s,a)
       -\tau\log\pi_{\mathrm{roll}}(a\mid s)
       -V_\tau^{\pi_{\mathrm{roll}}}(s)}_{A_{\mathrm{roll}}(s,a)}
       -\tau\bigl[\log\pi_\theta(a\mid s)
                  -\log\pi_{\mathrm{roll}}(a\mid s)\bigr]\\
    &\quad=A_{\mathrm{roll}}(s,a)
       -\tau\log\frac{\pi_\theta(a\mid s)}
                         {\pi_{\mathrm{roll}}(a\mid s)}.
  \end{aligned}
\end{equation*}
Taking the action expectation under $\pi_\theta$ yields
\begin{equation*}
  \begin{aligned}
    &\E_{a\sim\pi_\theta}\!\left[
       Q_{\mathrm{roll}}(s,a)-\tau\log\pi_\theta(a\mid s)\right]
       -V_\tau^{\pi_{\mathrm{roll}}}(s)\\
    &\quad=\underbrace{\E_{a\sim\pi_\theta}
        [A_{\mathrm{roll}}(s,a)]}_{\bar A_\theta(s)}
      -\tau\underbrace{\E_{a\sim\pi_\theta}\!\left[
        \log\frac{\pi_\theta(a\mid s)}
                  {\pi_{\mathrm{roll}}(a\mid s)}\right]}_{
        \KL(\pi_\theta\|\pi_{\mathrm{roll}})=k(s)}.
  \end{aligned}
\end{equation*}
The action distribution in this expectation determines the direction
of the KL: its first argument is the candidate policy.
\boundsubstep{1.4}{Average over rollout states}{app:bound_rollout_average}
We complete the surrogate identity by inserting the
per-state result from
\hyperref[app:bound_advantage_identity]{Substep~1.3} into the loss
difference from \hyperref[app:bound_loss_difference]{Substep~1.1}:
\begin{equation}
  \begin{aligned}
    \Delta J_{\mathrm{surrogate}}
    &=\frac{1}{1-\gamma}\E_{s\sim d_{\mathrm{roll}}}\!\left[
       \E_{a\sim\pi_\theta}[Q_{\mathrm{roll}}(s,a)
              -\tau\log\pi_\theta(a\mid s)]
       -V_\tau^{\pi_{\mathrm{roll}}}(s)\right]\\
    &=\frac{1}{1-\gamma}\E_{s\sim d_{\mathrm{roll}}}
       [\bar A_\theta(s)-\tau k(s)].
  \end{aligned}
  \label{eq:bound_surrogate_gain}
\end{equation}
Thus, the surrogate gain is the candidate's expected rollout
advantage, minus the correction needed to evaluate entropy under the
candidate policy.

\subsubsection*{Step 2: express the actual return improvement}
We show that $\Delta J_\tau$ averages the same
quantity $\bar A_\theta(s)-\tau k(s)$ as the surrogate, but over the
candidate's state distribution $d^{\pi_\theta}$. We derive this
identity directly from the return definition.

\boundsubstep{2.1}{Subtract the returns using a common initial distribution}{app:bound_return_difference}
We first write the return difference as an
expectation over candidate trajectories. Restating
Eq.~\ref{eq:maxent_objective},
\begin{equation*}
  \begin{aligned}
    J_\tau(\pi_\theta)
    &=\frac{1}{1-\gamma}
      \E_{\substack{s\sim d^{\pi_\theta}\\a\sim\pi_\theta(\cdot\mid s)}}
      [r(s,a)-\tau\log\pi_\theta(a\mid s)]\\
    &=\E_{\rho_0,P,\pi_\theta}\!\left[
      \sum_{t=0}^{\infty}\gamma^t
      \bigl(r(s_t,a_t)-\tau\log\pi_\theta(a_t\mid s_t)\bigr)
      \right].
  \end{aligned}
\end{equation*}
By definition, the rollout policy's soft value is its expected
entropy-augmented return from a given state. Hence
$J_\tau(\pi_{\mathrm{roll}})
=\E_{s_0\sim\rho_0}[V_\tau^{\pi_{\mathrm{roll}}}(s_0)]$.
Both policies start from the same $\rho_0$, so this subtraction can
be placed inside the expectation over candidate trajectories:
\begin{equation*}
  \Delta J_\tau
  =\E_{\rho_0,P,\pi_\theta}\!\left[
    \sum_{t=0}^{\infty}\gamma^t
      \bigl(r(s_t,a_t)-\tau\log\pi_\theta(a_t\mid s_t)\bigr)
    -V_\tau^{\pi_{\mathrm{roll}}}(s_0)\right].
\end{equation*}

\boundsubstep{2.2}{Telescope the rollout values}{app:bound_telescoping}
We replace the subtracted initial value by a sum of
differences between successive values. This puts every contribution
to $\Delta J_\tau$ inside the same discounted sum.
For any trajectory and finite horizon $N$, the intermediate value
terms cancel:
\begin{equation*}
  \begin{aligned}
  &\sum_{t=0}^{N}\gamma^t
       \bigl[\gamma V_\tau^{\pi_{\mathrm{roll}}}(s_{t+1})
                    -V_\tau^{\pi_{\mathrm{roll}}}(s_t)\bigr] =\sum_{t=1}^{N+1}\gamma^t
             V_\tau^{\pi_{\mathrm{roll}}}(s_t)
           -\sum_{t=0}^{N}\gamma^t
             V_\tau^{\pi_{\mathrm{roll}}}(s_t)\\
  &\quad=-V_\tau^{\pi_{\mathrm{roll}}}(s_0)
           +\gamma^{N+1}V_\tau^{\pi_{\mathrm{roll}}}(s_{N+1}).
  \end{aligned}
\end{equation*}
This cancellation also holds along candidate trajectories: it only
uses successive evaluations of the same rollout value function.
Taking expectations and letting $N\to\infty$, the terminal term
vanishes by the assumption stated above. Substituting this sum for
the initial-value subtraction therefore gives
\begin{equation*}
  \Delta J_\tau
  =\E_{\rho_0,P,\pi_\theta}\!\left[
    \sum_{t=0}^{\infty}\gamma^t
    \Bigl(r(s_t,a_t)-\tau\log\pi_\theta(a_t\mid s_t)
       +\gamma V_\tau^{\pi_{\mathrm{roll}}}(s_{t+1})
       -V_\tau^{\pi_{\mathrm{roll}}}(s_t)\Bigr)\right].
\end{equation*}
Each summand is the reward plus entropy contribution at the current
step, plus the discounted next-state value, minus the current-state
value.

\boundsubstep{2.3}{Evaluate the conditional one-step contribution}{app:bound_conditional_return}
We show that each summand, averaged over the
candidate's action and the next state, equals
$\bar A_\theta(s)-\tau k(s)$.
For fixed $s$ and $a$, the rollout Bellman equation is
\begin{equation*}
  Q_{\mathrm{roll}}(s,a)
  =r(s,a)+\gamma\E_{s'\sim P(\cdot\mid s,a)}
       [V_\tau^{\pi_{\mathrm{roll}}}(s')].
\end{equation*}
Thus, averaging over the next state replaces
$r(s,a)+\gamma V_\tau^{\pi_{\mathrm{roll}}}(s')$
by $Q_{\mathrm{roll}}(s,a)$. Averaging next over the candidate's
action and using
\hyperref[app:bound_advantage_identity]{Substep~1.3} gives
\begin{equation*}
  \begin{aligned}
  &\E_{\substack{a\sim\pi_\theta(\cdot\mid s)\\s'\sim P(\cdot\mid s,a)}}
      \!\left[r(s,a)-\tau\log\pi_\theta(a\mid s)
           +\gamma V_\tau^{\pi_{\mathrm{roll}}}(s')
           -V_\tau^{\pi_{\mathrm{roll}}}(s)\right]\\
  &\quad=\E_{a\sim\pi_\theta(\cdot\mid s)}\!\left[
           Q_{\mathrm{roll}}(s,a)-\tau\log\pi_\theta(a\mid s)
           -V_\tau^{\pi_{\mathrm{roll}}}(s)\right]\\
  &\quad=\E_{a\sim\pi_\theta(\cdot\mid s)}\!\left[
           A_{\mathrm{roll}}(s,a)
           -\tau\log\frac{\pi_\theta(a\mid s)}
                             {\pi_{\mathrm{roll}}(a\mid s)}\right] =\bar A_\theta(s)-\tau k(s).
  \end{aligned}
\end{equation*}
These nested averages are the law of total expectation applied to
the candidate trajectories. The value functions remain those of the
rollout policy throughout.

\boundsubstep{2.4}{Average over candidate states and compare the gains}{app:bound_candidate_average}
We now obtain the return-improvement identity and
compare it with Eq.~\ref{eq:bound_surrogate_gain}.
At time $t$, these states follow $\rho_t^{\pi_\theta}$. Using
$d^{\pi_\theta}=(1-\gamma)\sum_{t\geq0}\gamma^t\rho_t^{\pi_\theta}$,
we obtain
\begin{equation}
  \Delta J_\tau
    =\sum_{t=0}^{\infty}\gamma^t
      \E_{s\sim\rho_t^{\pi_\theta}}
        [\bar A_\theta(s)-\tau k(s)] =\frac{1}{1-\gamma}\E_{s\sim d^{\pi_\theta}}
        [\bar A_\theta(s)-\tau k(s)].
  \label{eq:bound_actual_gain}
\end{equation}
For a direct comparison, restate the actual return improvement and
the surrogate improvement from Eq.~\ref{eq:bound_surrogate_gain}:
\begin{equation*}
  \Delta J_\tau
    =\frac{1}{1-\gamma}\E_{s\sim d^{\pi_\theta}}
      [\bar A_\theta(s)-\tau k(s)], \qquad \Delta J_{\mathrm{surrogate}}
    =\frac{1}{1-\gamma}\E_{s\sim d_{\mathrm{roll}}}
      [\bar A_\theta(s)-\tau k(s)].
\end{equation*}
With an exact rollout critic and fixed temperature, the two
improvements differ only in how states are weighted. Both use the
candidate policy's actions inside $\bar A_\theta(s)$ and $k(s)$.
The surrogate evaluates the candidate policy on states visited by
the rollout policy, whereas the actual return reflects states visited
by the candidate itself. If these state distributions coincide, the
two improvements are exactly equal. The remaining proof bounds the
discrepancy caused by changing the visited states.

\subsubsection*{Step 3: bound the error in the advantage term}
We bound the error from averaging $\bar A_\theta$
over rollout states instead of candidate states. The argument has two parts:
a small policy change limits both the expected rollout advantage
and the change in visited states. Their product gives the advantage
penalty in Eq.~\ref{eq:trust_region_bound}.

\boundsubstep{3.1}{Bound the expected rollout advantage}{app:bound_advantage_size}
We show that
$|\bar A_\theta(s)|\leq 2\epsilon_A\alpha$, where $\alpha$ is the
largest policy change in total variation. Recall $\epsilon_A$ and
define $\alpha$ by
\begin{equation*}
  \begin{aligned}
    \epsilon_A&:=\sup_{s,a}|A_{\mathrm{roll}}(s,a)|,\\
    \alpha&:=\sup_s D_{\mathrm{TV}}\!\left(
      \pi_\theta(\cdot\mid s),\pi_{\mathrm{roll}}(\cdot\mid s)\right)
      =\frac12\sup_s\int_{\mathcal A}
        |\pi_\theta(a\mid s)-\pi_{\mathrm{roll}}(a\mid s)|\,\mathrm da.
  \end{aligned}
\end{equation*}
Here $\epsilon_A$ bounds the advantage magnitude, while
$\alpha\in[0,1]$ measures the largest policy change across states in
total variation (TV). By the soft-value definition,
\begin{equation*}
  \int_{\mathcal A}\pi_{\mathrm{roll}}(a\mid s)
      A_{\mathrm{roll}}(s,a)\,\mathrm da
  =\E_{a\sim\pi_{\mathrm{roll}}}\!\left[
      Q_{\mathrm{roll}}(s,a)-\tau\log\pi_{\mathrm{roll}}(a\mid s)\right]
      -V_\tau^{\pi_{\mathrm{roll}}}(s)=0.
\end{equation*}
Subtracting this zero from
$\bar A_\theta(s)=\int_{\mathcal A}\pi_\theta(a\mid s)
A_{\mathrm{roll}}(s,a)\,\mathrm da$ introduces the difference between
the two policies. The triangle inequality and the definitions above give
\begin{equation}
  \begin{aligned}
    |\bar A_\theta(s)|
    &=\left|\int_{\mathcal A}
      [\pi_\theta(a\mid s)-\pi_{\mathrm{roll}}(a\mid s)]
      A_{\mathrm{roll}}(s,a)\,\mathrm da\right|\\
    &\leq\int_{\mathcal A}
      |\pi_\theta(a\mid s)-\pi_{\mathrm{roll}}(a\mid s)|
      |A_{\mathrm{roll}}(s,a)|\,\mathrm da\\
    &\leq\epsilon_A\int_{\mathcal A}
      |\pi_\theta(a\mid s)-\pi_{\mathrm{roll}}(a\mid s)|\,\mathrm da
      \leq 2\epsilon_A\alpha.
  \end{aligned}
  \label{eq:bound_expected_advantage}
\end{equation}
\boundsubstep{3.2}{Derive the balance equation for discounted visits}{app:bound_state_balance}
To compare state distributions, we first derive
$d^\pi=(1-\gamma)\rho_0+\gamma d^\pi P_\pi$ from the definition of
discounted visitation. This identity separates the initial visit
from all later visits. We begin by defining the one-step transition
operator $P_\pi$.
Starting at state $s$, first draw an action from $\pi(\cdot\mid s)$,
then draw the next state from $P(\cdot\mid s,a)$. Averaging over the
action gives the next-state distribution under the policy:
\begin{equation*}
  P_\pi(s'\mid s)
  :=\int_{\mathcal A}\pi(a\mid s)P(s'\mid s,a)\,\mathrm da.
\end{equation*}
This conditional distribution is called the policy's transition
kernel. If the current state has distribution $\rho$, the next-state
distribution is written $\rho P_\pi$, meaning
\begin{equation*}
  (\rho P_\pi)(s')
  :=\int\rho(s)P_\pi(s'\mid s)\,\mathrm ds.
\end{equation*}
Thus multiplication by $P_\pi$ means advancing a state distribution
by one time step. In particular, $\rho_{t+1}^\pi=\rho_t^\pi P_\pi$.

Next, separate the initial visit from the later visits.
The definition of $d^\pi$ includes every time step, starting at
$t=0$. Separating out its first term gives an equivalent expression:
\begin{equation*}
  d^\pi
    :=(1-\gamma)\sum_{t=0}^{\infty}\gamma^t\rho_t^\pi =(1-\gamma)\rho_0
      +(1-\gamma)\sum_{t=1}^{\infty}\gamma^t\rho_t^\pi.
\end{equation*}
The $t=0$ term is $(1-\gamma)\gamma^0\rho_0^\pi
=(1-\gamma)\rho_0$. Both expressions describe the same
distribution; the second writes its initial-state contribution
separately.

Now rewrite the remaining sum. Set $u=t-1$ and use
$\rho_{u+1}^\pi=\rho_u^\pi P_\pi$:
\begin{equation*}
  \begin{aligned}
  (1-\gamma)\sum_{t=1}^{\infty}\gamma^t\rho_t^\pi
    &=\gamma(1-\gamma)\sum_{u=0}^{\infty}
        \gamma^u\rho_{u+1}^\pi\\
  &=\gamma(1-\gamma)\sum_{u=0}^{\infty}
        \gamma^u(\rho_u^\pi P_\pi) =\gamma\underbrace{\left[(1-\gamma)\sum_{u=0}^{\infty}
        \gamma^u\rho_u^\pi\right]}_{d^\pi}P_\pi
     =\gamma d^\pi P_\pi.
  \end{aligned}
\end{equation*}
The same transition operator $P_\pi$ acts linearly on every term, so
we can first form the weighted sum in brackets. This bracket is
exactly the original definition of $d^\pi$, which explains why it
reappears. Applying $P_\pi$ advances every visit by one step, and the
factor $\gamma$ adds the corresponding discount: the first term of
$\gamma d^\pi P_\pi$ is $(1-\gamma)\gamma\rho_1^\pi$.
Thus it collects visits from time $1$ onward. Adding the initial
visit gives
\begin{equation*}
  d^\pi=(1-\gamma)\rho_0+\gamma d^\pi P_\pi.
\end{equation*}

\boundsubstep{3.3}{Separate propagation from policy change}{app:bound_state_difference}
We express the difference in visited states as two
contributions: propagation of an existing state difference, and the
effect of choosing different actions. Apply the balance equation
from \hyperref[app:bound_state_balance]{Substep~3.2} to each policy.
Both policies use the same $\rho_0$, so their initial-state terms
cancel. We then add and subtract
$\gamma d_{\mathrm{roll}}P_{\pi_\theta}$:
\begin{equation*}
  \begin{aligned}
  d^{\pi_\theta}-d_{\mathrm{roll}}
    &=\gamma\bigl[d^{\pi_\theta}P_{\pi_\theta}
                 -d_{\mathrm{roll}}P_{\pi_{\mathrm{roll}}}\bigr] =\gamma\bigl[d^{\pi_\theta}P_{\pi_\theta}
                 -d_{\mathrm{roll}}P_{\pi_\theta}
                 +d_{\mathrm{roll}}P_{\pi_\theta}
                 -d_{\mathrm{roll}}P_{\pi_{\mathrm{roll}}}\bigr]\\
  &=\gamma(d^{\pi_\theta}-d_{\mathrm{roll}})P_{\pi_\theta}
      +\gamma d_{\mathrm{roll}}
         (P_{\pi_\theta}-P_{\pi_{\mathrm{roll}}}).
  \end{aligned}
\end{equation*}
The first term propagates the difference in current-state
distributions through the same candidate transition kernel. The
second starts from the same rollout state distribution and measures
the effect of changing the policy's action probabilities. We bound
these two contributions separately below.

\boundsubstep{3.4}{Bound propagation through a common transition kernel}{app:bound_state_propagation}
We show that one step under the same policy cannot
increase the $L^1$ distance between state distributions. This bounds
the first contribution from
\hyperref[app:bound_state_difference]{Substep~3.3}.
Write $\delta d:=d^{\pi_\theta}-d_{\mathrm{roll}}$ for the signed
difference of the two state distributions, with
$\|\delta d\|_1:=\int|\delta d(s)|\,\mathrm ds$.
Taking norms in the preceding identity and applying the triangle
inequality gives
\begin{equation*}
  \|\delta d\|_1
  \leq\gamma\|\delta dP_{\pi_\theta}\|_1
      +\gamma\|d_{\mathrm{roll}}
        (P_{\pi_\theta}-P_{\pi_{\mathrm{roll}}})\|_1.
\end{equation*}
For the first term, use the nonnegativity and normalization of the
transition probabilities:
\begin{equation*}
  \|\delta dP_{\pi_\theta}\|_1
    =\int\left|\int\delta d(s)P_{\pi_\theta}(s'\mid s)
                       \,\mathrm ds\right|\,\mathrm ds' \leq\int|\delta d(s)|
       \underbrace{\left[\int P_{\pi_\theta}(s'\mid s)
                              \,\mathrm ds'\right]}_{1}
       \,\mathrm ds
     =\|\delta d\|_1.
\end{equation*}
Thus applying the same transition kernel cannot increase the
$L^1$ distance between state distributions.

\boundsubstep{3.5}{Bound policy change and the resulting state shift}{app:bound_policy_change}
We bound the second contribution by $2\alpha$.
Combining it with \hyperref[app:bound_state_propagation]{Substep~3.4}
then gives $\|d^{\pi_\theta}-d_{\mathrm{roll}}\|_1
\leq 2\gamma\alpha/(1-\gamma)$.
First fix a state $s$ and expand the two
policy-averaged transition kernels:
\begin{equation*}
  \begin{aligned}
  &\int|P_{\pi_\theta}(s'\mid s)
           -P_{\pi_{\mathrm{roll}}}(s'\mid s)|\,\mathrm ds' =\int\left|\int_{\mathcal A}
       [\pi_\theta(a\mid s)-\pi_{\mathrm{roll}}(a\mid s)]
       P(s'\mid s,a)\,\mathrm da\right|\,\mathrm ds'\\
  &\quad\leq\int_{\mathcal A}
       |\pi_\theta(a\mid s)-\pi_{\mathrm{roll}}(a\mid s)|
       \underbrace{\left[\int P(s'\mid s,a)\,\mathrm ds'\right]}_{1}
       \,\mathrm da
     \leq 2\alpha.
  \end{aligned}
\end{equation*}
The last inequality is the definition of $\alpha$: at every state,
the $L^1$ distance between the action distributions is at most
$2\alpha$. Averaging over the normalized rollout state distribution
preserves this bound:
\begin{equation*}
  \begin{aligned}
    \|d_{\mathrm{roll}}
      (P_{\pi_\theta}-P_{\pi_{\mathrm{roll}}})\|_1
    &\leq\int d_{\mathrm{roll}}(s)
      \left[\int|P_{\pi_\theta}(s'\mid s)
                   -P_{\pi_{\mathrm{roll}}}(s'\mid s)|
                   \,\mathrm ds'\right]\,\mathrm ds\\
    &\leq 2\alpha\int d_{\mathrm{roll}}(s)\,\mathrm ds
     =2\alpha.
  \end{aligned}
\end{equation*}
Substituting these two bounds into the triangle inequality gives
$\|\delta d\|_1\leq\gamma\|\delta d\|_1+2\gamma\alpha$.
Move the term $\gamma\|\delta d\|_1$ to the left and divide by
$1-\gamma>0$. This gives
\begin{equation}
  (1-\gamma)\|d^{\pi_\theta}-d_{\mathrm{roll}}\|_1
    \leq 2\gamma\alpha, \qquad \|d^{\pi_\theta}-d_{\mathrm{roll}}\|_1
    \leq\frac{2\gamma\alpha}{1-\gamma}.
  \label{eq:bound_state_shift}
\end{equation}
\boundsubstep{3.6}{{Combine the advantage bounds in TV}}{app:bound_advantage_combine}
We now bound the advantage-expectation error by
multiplying $\sup_s|\bar A_\theta(s)|$ by the state-distribution
distance. Both bounds are proportional to $\alpha$, so their product is
{proportional to $\alpha^2$.}
Combining Eqs.~\ref{eq:bound_expected_advantage} and
\ref{eq:bound_state_shift} {gives}
\begingroup
\begin{equation}
  \begin{aligned}
    \frac{1}{1-\gamma}
    \left|\E_{d^{\pi_\theta}}[\bar A_\theta]
           -\E_{d_{\mathrm{roll}}}[\bar A_\theta]\right|
    &\leq\frac{\|d^{\pi_\theta}-d_{\mathrm{roll}}\|_1
                  \sup_s|\bar A_\theta(s)|}{1-\gamma}\\
    &\leq\frac{4\gamma\epsilon_A}{(1-\gamma)^2}\alpha^2.
  \end{aligned}
  \label{eq:bound_advantage_error}
\end{equation}
\endgroup
This step bounds the contribution of the advantage term to the
discrepancy between $\Delta J_\tau$ and
$\Delta J_{\mathrm{surrogate}}$. Small policy changes limit both
the change in visited states and the expected rollout advantage,
producing the quadratic total-variation bound above. Step~4 accounts
for the remaining entropy correction.

\subsubsection*{Step 4: complete the return-improvement bound}
We bound the remaining entropy correction and add
it to the advantage error. This proves Eq.~\ref{eq:trust_region_bound}
and gives a sufficient condition for $J_\tau$ to improve.

\boundsubstep{4.1}{Bound the entropy correction}{app:bound_entropy_error}
Subtracting the two gain identities from
\hyperref[app:bound_candidate_average]{Substep~2.4}, the contribution
from $-\tau k(s)$ is
$\frac{\tau}{1-\gamma}
(\E_{d_{\mathrm{roll}}}[k]-\E_{d^{\pi_\theta}}[k])$.
We show that this contribution is at least $-\tau D/(1-\gamma)$.
Since $0\leq k(s)\leq D$, both state averages lie in $[0,D]$. Hence
{
\begin{equation*}
  \frac{\tau}{1-\gamma}
  \left(\E_{s\sim d_{\mathrm{roll}}}[k(s)]
        -\E_{s\sim d^{\pi_\theta}}[k(s)]\right)
  \geq -\frac{\tau}{1-\gamma}D.
\end{equation*}
}

\boundsubstep{4.2}{Combine the advantage and entropy bounds}{app:bound_combine}
We obtain the claimed lower bound on $\Delta J_\tau$
by subtracting the two gain identities and applying
\hyperref[app:bound_advantage_combine]{Substep~3.6} to the advantage
term, then \hyperref[app:bound_entropy_error]{Substep~4.1} to the
entropy term:
\begingroup
\begin{equation}
  \begin{aligned}
  \Delta J_\tau
    \geq{}&\Delta J_{\mathrm{surrogate}}
      -\frac{4\gamma\epsilon_A}{(1-\gamma)^2}\alpha^2 +\frac{\tau}{1-\gamma}
      \left(\E_{s\sim d_{\mathrm{roll}}}[k(s)]
            -\E_{s\sim d^{\pi_\theta}}[k(s)]\right)\\
  \geq{}&\Delta J_{\mathrm{surrogate}}
      -\frac{4\gamma\epsilon_A}{(1-\gamma)^2}\alpha^2
      -\frac{\tau}{1-\gamma}D.
  \end{aligned}
  \label{eq:bound_final_derivation}
\end{equation}
\endgroup
This proves Eq.~\ref{eq:trust_region_bound}.
The first penalty bounds the error in
the advantage prediction; the second bounds the error in its entropy
correction. Both arise from evaluating the update on rollout states
instead of the states visited by the candidate.

\boundsubstep{4.3}{State the improvement condition and its practical meaning}{app:bound_interpretation}
The bound guarantees an increase in $J_\tau$ when
the surrogate gain exceeds the two error penalties. Explicitly,
\begingroup
\begin{equation*}
  \Delta J_{\mathrm{surrogate}}
  >\frac{4\gamma\epsilon_A}{(1-\gamma)^2}\alpha^2
      +\frac{\tau}{1-\gamma}D
  \quad\Longrightarrow\quad
  J_\tau(\pi_\theta)>J_\tau(\pi_{\mathrm{roll}}).
\end{equation*}
\endgroup
{To express a sufficient condition using only KL,
apply Pinsker's inequality, $2\alpha^2\leq D$:
\begin{equation*}
  \frac{4\gamma\epsilon_A}{(1-\gamma)^2}\alpha^2
      +\frac{\tau}{1-\gamma}D
  \leq\left(\frac{2\gamma\epsilon_A}{(1-\gamma)^2}
            +\frac{\tau}{1-\gamma}\right)D.
\end{equation*}
The surrogate gain exceeding this KL upper bound is therefore also
sufficient for return improvement.}
In words, the predicted improvement must exceed the bound on the
prediction error. A small KL alone does not imply improvement, and
failing this sufficient condition does not imply that the update is
harmful.

This motivates the KL term in Eq.~\ref{eq:constrained_lagrangian}:
minimize the surrogate while limiting policy change. Replacing the
worst-state KL constraint by a rollout-average constraint gives the
average penalty. With the critic and rollout distribution fixed, the
full Lagrangian satisfies
\begin{equation*}
  \begin{aligned}
    \mathcal L(\pi_\theta,\tau,\lambda)
    &=(1-\gamma)\mathcal L_{\mathrm{ME}}(\theta;\theta_{\mathrm{roll}})
      +\lambda\E_{s\sim d_{\mathrm{roll}}}[k(s)]\\
    &\quad+\tau\mathcal H_{\mathrm{target}}-\lambda\delta_{\mathrm{KL}}
      -\tau\E_{s\sim d_{\mathrm{roll}}}[\log Z_Q(s)].
  \end{aligned}
\end{equation*}
The second line is independent of the policy when the multipliers are
fixed. The average constraint is a practical relaxation. With this
relaxation and an estimated critic, minimizing $\mathcal L$ over the
policy does not by itself guarantee the sufficient condition above.

\endgroup

\begingroup
\subsection{Trust regions for Gaussian policies}
\label{app:gaussian_trust_regions}

For Gaussian policies, a bound on either KL direction also bounds the
other. A TV bound strictly below one bounds both KL directions.
The short proofs below give conservative constants sufficient to
control the entropy correction in Eq.~\ref{eq:trust_region_bound}.

\textbf{Setup and notation.}
Fix a state and suppress the conditioning on $s$. Assume positive-definite
covariances, meaning every variance is strictly positive for our diagonal
Gaussian policies, and use natural logarithms. KL and TV are unchanged by a
common invertible transformation: its Jacobian cancels in the density
ratio and the integral. We can therefore undo the common tanh
squashing and action rescaling, then subtract the rollout mean and
divide by the rollout standard deviation in each coordinate. In
these standardized coordinates,
\begin{equation*}
  \pi_{\mathrm{roll}}=\mathcal N(0,I),\qquad
  \pi_{\mathrm{curr}}
    =\mathcal N\!\left(m,\operatorname{diag}(\sigma_1^2,\ldots,\sigma_d^2)\right).
\end{equation*}
Here $m_i$ is the standardized mean change, and $\sigma_i^2$ is the
current variance relative to the rollout variance. Full covariance
matrices reduce to the same form by whitening and rotation.
The Gaussian KL formulas separate variance changes from mean changes:
\begin{equation}
  \begin{aligned}
    \KL(\pi_{\mathrm{roll}}\|\pi_{\mathrm{curr}})
      &=\frac12\sum_i
        \left[\log\sigma_i^2+\frac1{\sigma_i^2}-1
                    +\frac{m_i^2}{\sigma_i^2}\right],\\
    \KL(\pi_{\mathrm{curr}}\|\pi_{\mathrm{roll}})
      &=\frac12\sum_i
        \left[\sigma_i^2-1-\log\sigma_i^2+m_i^2\right].
  \end{aligned}
  \label{eq:gaussian_kl_coordinates}
\end{equation}
Each variance term and each mean term is nonnegative.

\subsubsection{Forward KL bounds reverse KL}
\label{app:gaussian_kl_comparison}

Suppose $\KL(\pi_{\mathrm{roll}}\|\pi_{\mathrm{curr}})\leq\delta$.
We prove the explicit bound
\begin{equation}
  \boxed{\KL(\pi_{\mathrm{curr}}\|\pi_{\mathrm{roll}})
    \leq \delta e^{1+2\delta}.}
  \label{eq:gaussian_reverse_budget}
\end{equation}
This bound tends to zero with $\delta$ and does not depend on the
action dimension. A sharper bound is given by
\citet[Theorem~1]{zhang2023gaussiankl}.

\textbf{Step 1: bound the variance ratios.}
Since every contribution to forward KL is nonnegative,
\begin{equation*}
  \log\sigma_i^2+\frac1{\sigma_i^2}-1\leq2\delta.
\end{equation*}
Rearranging and using $1/\sigma_i^2>0$ gives
\begin{equation*}
  \log\sigma_i^2
  \leq 1+2\delta-\frac1{\sigma_i^2}
  \leq 1+2\delta.
\end{equation*}
Exponentiating preserves the inequality, so
\begin{equation*}
  \sigma_i^2\leq e^{1+2\delta}=:M_\delta.
\end{equation*}
This is a conservative upper bound on the variance; the corresponding
standard deviation satisfies $\sigma_i\leq e^{1/2+\delta}$.

\textbf{Step 2: compare the KL terms.}
For any $x>0$,
\begin{equation}
  x-1-\log x
  \leq \max\{1,x\}\bigl(x^{-1}-1+\log x\bigr).
  \label{eq:gaussian_variance_comparison}
\end{equation}
To see this, both scalar expressions vanish at $x=1$, and
the derivative of the left expression is $x$ times the derivative
of $x^{-1}-1+\log x$. Integrating between $1$ and $x$ bounds this
factor by $\max\{1,x\}$.
Apply Eq.~\ref{eq:gaussian_variance_comparison} with $x=\sigma_i^2$.
Step 1 gives both $M_\delta\geq1$ and $\sigma_i^2\leq M_\delta$,
so $\max\{1,\sigma_i^2\}\leq M_\delta$. Thus the variance contribution
in each coordinate satisfies
\begin{equation*}
  \begin{aligned}
    \sigma_i^2-1-\log\sigma_i^2
    &\leq\max\{1,\sigma_i^2\}
      \left(\frac1{\sigma_i^2}-1+\log\sigma_i^2\right)\\
    &\leq M_\delta
      \left(\frac1{\sigma_i^2}-1+\log\sigma_i^2\right).
  \end{aligned}
\end{equation*}
The mean contribution satisfies the same comparison:
\begin{equation*}
  m_i^2=\sigma_i^2\frac{m_i^2}{\sigma_i^2}
  \leq M_\delta\frac{m_i^2}{\sigma_i^2}.
\end{equation*}
Add these two inequalities in each coordinate, then sum over $i$.
The same $M_\delta$ applies to every coordinate, so it factors out:
\begin{equation*}
  \begin{aligned}
    \KL(\pi_{\mathrm{curr}}\|\pi_{\mathrm{roll}})
    &=\frac12\sum_i
      \left[\sigma_i^2-1-\log\sigma_i^2+m_i^2\right]\\
    &\leq\frac12\sum_i\left[
      M_\delta\left(\frac1{\sigma_i^2}-1+\log\sigma_i^2\right)
      +M_\delta\frac{m_i^2}{\sigma_i^2}\right]\\
    &=M_\delta\,\frac12\sum_i
      \left[\log\sigma_i^2+\frac1{\sigma_i^2}-1
            +\frac{m_i^2}{\sigma_i^2}\right]\\
    &=M_\delta\KL(\pi_{\mathrm{roll}}\|\pi_{\mathrm{curr}})
      \leq M_\delta\delta.
  \end{aligned}
\end{equation*}
The last inequality uses the assumed forward-KL budget
$\KL(\pi_{\mathrm{roll}}\|\pi_{\mathrm{curr}})\leq\delta$,
proving Eq.~\ref{eq:gaussian_reverse_budget}.
Exchanging the policies proves the reverse implication. If their
covariances are equal, the two KL divergences are exactly equal.

\subsubsection{TV bounds both KL directions}
\label{app:gaussian_tv_comparison}

At a fixed state, Gaussians with positive-definite covariances satisfy
$\rho:=D_{\mathrm{TV}}(\pi_{\mathrm{roll}},\pi_{\mathrm{curr}})<1$.
Let $B_{\mathrm G}(\delta)$ denote the dimension-independent reverse-KL
bound given a forward-KL budget $\delta$, established by
\citet[Theorem~1]{zhang2023gaussiankl}. This bound is finite and tends
to zero with $\delta$. Writing $t_\rho=-2\log(1-\rho)$, we obtain the
following corollary:
\begin{equation}
  \boxed{
    \max\!\left\{
      \KL(\pi_{\mathrm{curr}}\|\pi_{\mathrm{roll}}),
      \KL(\pi_{\mathrm{roll}}\|\pi_{\mathrm{curr}})
    \right\}
    \leq\underbrace{2B_{\mathrm G}(t_\rho)+t_\rho}_{C_{\mathrm{TV}}(\rho)}.}
  \label{eq:gaussian_tv_to_kl}
\end{equation}

\textbf{Proof.}
Define the affinity and normalized geometric mean
\begin{equation*}
  A=\int\sqrt{\pi_{\mathrm{roll}}(a)\pi_{\mathrm{curr}}(a)}\,\mathrm da,
  \qquad
  \pi_{\mathrm{mid}}(a)
  =\frac{\sqrt{\pi_{\mathrm{roll}}(a)\pi_{\mathrm{curr}}(a)}}{A}.
\end{equation*}
The distribution $\pi_{\mathrm{mid}}$ is Gaussian because its log density
is quadratic. The TV--Hellinger inequality gives $A\geq1-\rho$
\citep[Section~2]{devroye2023gaussiantv}.
Substituting the definition of $\pi_{\mathrm{mid}}$ yields
\begin{equation*}
  \KL(\pi_{\mathrm{mid}}\|\pi_{\mathrm{roll}})
  +\KL(\pi_{\mathrm{mid}}\|\pi_{\mathrm{curr}})
  =-2\log A\leq t_\rho.
\end{equation*}
Both terms are nonnegative, so each is at most $t_\rho$. Zhang et al.'s
theorem therefore gives
$\KL(\pi_{\mathrm{curr}}\|\pi_{\mathrm{mid}})\leq B_{\mathrm G}(t_\rho)$.
Finally, the same density identity implies
\begin{equation*}
  \begin{aligned}
    \KL(\pi_{\mathrm{curr}}\|\pi_{\mathrm{roll}})
    &=2\KL(\pi_{\mathrm{curr}}\|\pi_{\mathrm{mid}})-2\log A\\
    &\leq2B_{\mathrm G}(t_\rho)+t_\rho.
  \end{aligned}
\end{equation*}
Exchanging the policies proves the other direction.\hfill$\square$

Thus $C_{\mathrm{TV}}(\rho)$ is finite for every $\rho<1$ and tends to
zero as $\rho\to0$. A common invertible tanh transformation preserves
both divergences. Applying this result uniformly across states requires
a common TV bound strictly below one.

\subsubsection{Implications for the return bound}
\label{app:gaussian_trust_region_implications}

If the forward-KL budget $\delta$ holds at every state, Pinsker and
Eq.~\ref{eq:gaussian_reverse_budget} give
$(D_{\mathrm{TV}}^{\max})^2\leq\delta/2$ and
$D_{\mathrm{KL}}^{\max}\leq\delta e^{1+2\delta}$.
A uniform TV budget $\rho<1$ instead gives
$D_{\mathrm{KL}}^{\max}\leq C_{\mathrm{TV}}(\rho)$.
Thus either constraint controls both penalties in
Eq.~\ref{eq:trust_region_bound}. Small temperature further reduces
the entropy penalty. The exact-critic and bounded-soft-advantage
assumptions still apply, and the surrogate gain must exceed the
penalties to guarantee improvement.

These are uniform, population-level statements. REPPO's sampled
forward-KL constraint \citep{voelcker2026relative} and FERPO's
state-average ESS objective do not establish worst-state control.
For FERPO, the fitting condition of Proposition~\ref{prop:ess_actor_tv}
and a population ESS floor $e_\infty(s)\geq e_0>e^{-1/2}$ at every
state would give
$D_{\mathrm{TV}}^{\max}\leq\sqrt{-2\log e_0}<1$; the TV result then
also bounds the entropy-related KL. These comparisons concern the
two Gaussian actors; the unrestricted target $q^*$ need not be Gaussian.
\endgroup

\end{document}